\newif\ifarxiv
\arxivtrue   

\ifarxiv
    \documentclass[10pt]{article}
    \usepackage{geometry}
    \usepackage{authblk}
    \RequirePackage{fancyhdr}
    \RequirePackage{eso-pic}
    \RequirePackage{forloop}
\else
    \documentclass{article}
    \usepackage[nonatbib,final]{neurips_2025}
\fi

\usepackage[style=authoryear, backend=biber, natbib=true, uniquename=false, maxbibnames=999]{biblatex}
\usepackage[utf8]{inputenc} 
\usepackage[T1]{fontenc}    
\usepackage{hyperref}       
\hypersetup{
  colorlinks=true,
  linkcolor=blue,
  filecolor=magenta,
  urlcolor=cyan,
  citecolor=blue
}
\usepackage{url}            
\usepackage{booktabs}       
\usepackage{amsfonts}       
\usepackage{nicefrac}       
\usepackage{microtype}      
\usepackage{xcolor}         
\usepackage{algorithm,algorithmic}
\usepackage{amsmath}
\usepackage{amssymb}
\usepackage{mathtools}
\usepackage{amsthm}
\usepackage{enumitem}
\usepackage{soul}
\usepackage{wrapfig}
\usepackage{rka-style}          
\usepackage{subfigure}

\usepackage[capitalize,noabbrev]{cleveref}
\crefname{lemma}{lemma}{lemmas}
\Crefname{lemma}{Lemma}{Lemmas}
\crefname{example}{example}{examples}
\Crefname{example}{Example}{Examples}
\crefname{fact}{fact}{facts}
\Crefname{fact}{Fact}{Facts}
\crefname{theorem}{theorem}{theorems}
\Crefname{theorem}{Theorem}{Theorems}
\crefname{assumption}{assumption}{assumptions}
\Crefname{assumption}{Assumption}{Assumptions}
\crefname{proposition}{proposition}{propositions}
\Crefname{proposition}{Proposition}{Propositions}
\usepackage{xcolor}

\definecolor{color_sgd}{HTML}{e37400}
\definecolor{color_dp}{HTML}{a50e0e}
\definecolor{color_dp4bsz}{HTML}{ee675c}
\definecolor{color_dp_4bsz}{HTML}{ee675c}
\definecolor{color_nesterov}{HTML}{0d652d}
\definecolor{color_nesterov2}{HTML}{34a853}
\definecolor{color_sf}{HTML}{174ea6}

\theoremstyle{plain}
\newtheorem{theorem}{Theorem}
\newtheorem{proposition}{Proposition}
\newtheorem{lemma}{Lemma}
\newtheorem{corollary}{Corollary}

\theoremstyle{definition}

\newtheorem{assumption}{Assumption}

\theoremstyle{remark}

\usepackage[disable,textsize=tiny]{todonotes}

\usepackage{color-edits}
\addauthor{sk}{blue}
\addauthor{ak}{green}
\addauthor{mz}{red}

\ifarxiv
\usepackage{parskip}
    \title{\bf Understanding Outer Optimizers in Local SGD: Learning Rates, Momentum, and Acceleration}
    \author{Ahmed Khaled$^{1, }$\thanks{Part of this work was done during an internship at Google DeepMind.}\ , Satyen Kale$^{2,}$\thanks{Currently at Apple.}\ , Arthur Douillard$^3$, Chi Jin$^1$, \\ Rob Fergus$^{4,5}$, and Manzil Zaheer$^3$}
\affil{$^1$Princeton University, \quad $^2$Google Research \quad $^3$Google DeepMind \\
$^4$New York University \quad $^5$FAIR at Meta 
}
\date{}
\else
    \title{Understanding Outer Optimizers in Local SGD: Learning Rates, Momentum, and Acceleration}
    
    \author{%
      Ahmed Khaled\thanks{Part of this work was done during an internship at Google DeepMind.} \\
      Princeton University \\
      \texttt{ahmed.khaled@princeton.edu}
      \And
      Satyen Kale\thanks{Currently at Apple.} \\
      Google Research \\  \texttt{satyen@satyenkale.com} \\
      \And
      Arthur Douillard \\
      Google DeepMind \\
      \texttt{douillard@google.com} \\
      \And
      Chi Jin \\
      Princeton University \\
      Princeton, NJ 08544 \\
      \texttt{chij@princeton.edu}
      \And
      Rob Fergus \\
      NYU, Meta \\
      \texttt{fergus@cs.nyu.edu} \\
      \And
      Manzil Zaheer \\
      Google DeepMind \\
      \texttt{manzilzaheer@gmail.com}
    }
\fi

\begin{document}

\maketitle

\begin{abstract}
  Modern machine learning often requires training with large batch size, distributed data, and massively parallel compute hardware (like mobile and other edge devices or distributed data centers). Communication becomes a major bottleneck in such settings but methods like Local Stochastic Gradient Descent (Local SGD) show great promise in reducing this additional communication overhead. Local SGD consists of three parts: a local optimization process, an aggregation mechanism, and an outer optimizer that uses the aggregated updates from the nodes to produce a new model. While there exists an extensive literature on understanding the impact of hyperparameters in the local optimization process, the choice of outer optimizer and its hyperparameters is less clear. We study the role of the outer optimizer in Local SGD, and prove new convergence guarantees for the algorithm. In particular, we show that tuning the outer learning rate allows us to (a) trade off between optimization error and stochastic gradient noise variance, and (b) make up for ill-tuning of the inner learning rate. Our theory suggests that the outer learning rate should sometimes be set to values greater than $1$. We extend our results to settings where we use momentum in the outer optimizer, and we show a similar role for the momentum-adjusted outer learning rate. We also study acceleration in the outer optimizer and show that it improves the convergence rate as a function of the number of communication rounds, improving upon the convergence rate of prior algorithms that apply acceleration locally. Finally, we also introduce a novel data-dependent analysis of Local SGD that yields further insights on outer learning rate tuning.  We conduct comprehensive experiments with standard language models and various outer optimizers to validate our theory.
\end{abstract}

\vspace{-2mm}
\section{Introduction}
Training very large scale machine learning models requires a lot of compute. This compute is often centrally controlled by a single entity and tightly connected in a data center. Gradients are constantly synchronized, hardware failures are controlled and mitigated, and things (mostly) run smoothly. Building this training infrastructure is expensive, however, and centralized control might not be desirable for all models. This has led to a surge of interest in decentralized collaborative training of large-scale models across different, potentially poorly connected clusters~\citep{douillard23_diloc,jaghouar24_opend,jaghouar2024intellect}. This has motivated the adoption of federated learning algorithms in training language models, chiefly for scalability and communication efficiency rather than data privacy. Efficient parallelization strategies also factored in the remarkable recent training of DeepSeek V3 and R1 on a tight budget~\citep{liu2024deepseek,guo2025deepseek}.

A foundational algorithm in distributed and federated optimization is Local
SGD~\citep{wang21_field_guide_to_feder_optim}. Many popular algorithms fit in
the FedOpt template \citep{reddi20_adapt_feder_optim} (Algorithm~\ref{alg:fed-opt}), including FedAdam~\citep{reddi20_adapt_feder_optim}, FedRR~\citep{mishchenko21_proxim_feder_random_reshuf,malinovsky22_feder_random_reshuf_with_compr_varian_reduc}, DiLoCo~\citep{douillard23_diloc,jaghouar24_opend} and many others. FedOpt solves the minimization problem $\min_{x \in \mathbb{R}^d} f(x)$ given access to $M$ different computational nodes and unbiased stochastic gradients of $f$. FedOpt consists of three main components: an inner update loop on every client, an aggregation of the client updates, and then an outer update step taken on the server.

\begin{algorithm}[h]
  \begin{algorithmic}[1]
  \STATE \textbf{Input.} Update rules $\mathrm{LocalUpdate}$ and $\mathrm{OuterUpdate}$. Initial point $x_0$.
\FOR{communication rounds $r=0, 1, \ldots, R-1$}
    \STATE Broadcast $x_r$ to each node $m$\;
    \FOR{each node $m$ in parallel}
        \STATE Set $y_{m, r, 0} = x_r$. \\
        \FOR{local steps $h= 0, 1, \ldots, H-1$}
          \STATE Set $y_{m, r, h+1} = \mathrm{LocalUpdate}(y_{m, r, h}, g_{m, r, h})$ for stochastic gradient $g_{m, r, h}$ at $y_{m, r, h}$.
        \ENDFOR
        \STATE Communicate $y_{m, r, H}$ to the server.
    \ENDFOR
    \STATE Compute the update or ``outer gradient'' $\hat{\Delta}_{r, H} = \frac{1}{M} \sum_{m=1}^M (y_{m, r, H} - x_r)$.
    \STATE Update $x_{r+1} = \text{OuterUpdate}(x_r, -\hat{\Delta}_{r, H})$.
\ENDFOR
\end{algorithmic}
\caption{The FedOpt Algorithmic Template}
\label{alg:fed-opt}
\end{algorithm}

When both the local and outer update rules correspond to gradient descent (i.e. $x_{\mathrm{new}} = x_{\mathrm{old}} - \beta \Delta$ for some stepsize $\beta$ and update vector $\Delta$), the corresponding algorithm is Generalized Local SGD. If we additionally take the outer stepsize to be $1$, we get Local SGD. Local SGD simply does $H$ steps of SGD on each node, and then averages the result after applying the updates. This is the most common form in which the algorithm is analyzed, as in e.g.~\citep{stich18_local_sgd_conver_fast_commun_littl,khaled19_tight_theor_local_sgd_ident_heter_data,woodworth20_is_local_sgd_better_than_minib_sgd,koloskova20_unified_theor_decen_sgd_with,glasgow21_sharp_bound_feder_averag_local,patel24_limit_poten_local_sgd_distr}.
In practice, different choices of outer optimizers perform better. For
example, DiLoCo/OpenDiLoCo use SGD with Nesterov Momentum as the outer
optimizer~\citep{douillard23_diloc}. This has motivated much analysis of different outer optimizers and their impact~\citep{reddi20_adapt_feder_optim,malinovsky22_server_side_steps_sampl_without,jhunjhunwala23_fedex,sun23_role_server_momen_feder_learn}. However, our  theoretical understanding of the fundamental Generalized Local SGD algorithm remains limited. In particular, it is not clear why the bilevel optimization structure of the algorithm is helpful from an optimization perspective, even in the i.i.d. setting where the data distribution is the same on all the nodes. Additionally and to the best of our knowledge, we have no explicit expressions for what the ideal learning rate pair $(\eta, \gamma)$ for the inner and outer updates, respectively, should be. Empirically, outer optimizers employing Nesterov acceleration have the best performance, yet to the best of our knowledge why or how it improves convergence is not known.

\textbf{Contributions.}  Our paper takes steps to address the above questions and makes the following contributions.
\vspace{-2mm}
\begin{itemize}[left=3mm]
    \item We conduct a novel, tighter analysis of Generalized Local SGD (Theorem~\ref{thm:gen-loc-sgd}) that shows the outer learning rate plays a dual role. It (a) interpolates between two extreme regimes: taking many effective steps at the cost of higher variance to taking fewer steps but at reduced variance
    and (b) increases the algorithmic robustness to hyperparameter tuning by making up for ill-tuned inner learning rates. The latter holds even in the absence of any stochastic gradient noise.
    \vspace{-0.5mm}
    \item We extend the above analysis to cover Generalized Local SGD where the outer optimizer also uses momentum (Theorem~\ref{thm:momentum}) and show that this gives additional leeway in tuning $\gamma$.
        \vspace{-0.5mm}
  \item We provide a convergence analysis for Local SGD with an accelerated outer optimizer and unaccelerated inner optimizer (Theorem~\ref{thm:accelerated}), showing that using Nesterov acceleration in the outer loop achieves better dependence on the number of communication rounds $R$ in the drift terms compared to standard Local SGD and improving upon the convergence rate of FedAc~\citep{yuan20_feder_accel_stoch_gradien_descen}.
    \vspace{-0.5mm}
    \item We also derive a data-dependent, high-probability guarantee for the convergence of Local SGD with GD as the outer optimizer (Theorem~\ref{thm:loc-sgd-guarantee}) that shows further benefits of tuning the outer stepsize in more nuanced settings.
    \item We additionally conduct an extensive empirical analysis for training large-scale language models with various outer optimizers (gradient descent, accelerated gradient descent, and Schedule-Free gradient descent).
\end{itemize}
\vspace{-1mm}
We now review related work, then proceed to our main results.

\section{Related Work}

There is a rich literature on algorithms for communication-efficient distributed
optimization for \emph{federated learning}~\citep{konecny16_feder_learn}, where
multiple clients collaborate on solving a machine learning problem ~\citep{wang21_field_guide_to_feder_optim}. Federated learning
algorithms are designed to reduce the effect of data
heterogeneity~\citep{karimireddy19_scaff,wang21_field_guide_to_feder_optim,murata21_bias_varian_reduc_local_sgd},
ensure the data stays private~\citep{wei2020federated}, deal with
intermittent or cyclic client availability~\citep{eichner2019semi}, among other
issues.

As models have grown larger in size over the past few years, going from a few
million parameters to billions \citep{brown20_languag_model_are_few_shot_learn}, the scale of training runs has also grown to
include many more devices divided across multiple computing clusters rather than
a single cluster~\citep{diskin2021distributedcollab,huang2022crosssilofederatedlearningchallenges,borzunov2023petals,douillard23_diloc}. Even within a single datacenter,
training runs now involve tens of thousands of GPUs~\citep{jiang2024megascale}.
This has motivated researchers to develop and use algorithms inspired by the
federated learning setting for large-scale training instead. Examples of such
algorithms include DiLoCo~\citep{douillard23_diloc}, its open cousin
OpenDiLoCo~\citep{jaghouar24_opend}, DiPaCo~\citep{douillard24_dipac}, and
others~\citep{liu24_async_local_sgd_train_languag_model,
  liang24_commun_effic_large_scale_distr,liu2024deepseek}. Federated learning methods thus have found use in
pretraining and fine-tuning language models~\citep{jaghouar24_opend,
  yang24_sa_fedlor}, and may prove particularly important for scaling even
larger models in the future~\citep{iacob2024worldwide,sani2024future,rush2024drjax}. We note
that the use of methods for federated learning even for i.i.d. distributed training is not
new, and is perhaps being ``re-discovered'' as training runs grow too large to
fit on single clusters. For example, \citet{lin18_dont_use_large_mini_batch}
argued that using Local SGD can be more efficient than traditional Minibatch SGD
in some settings. \Citet{ortiz21_trade_offs_local_sgd_at_scale} also conducted
experiments studying the trade-offs of using Local SGD in training image
classification models.

The most popular algorithm in the federated optimization literature is Local SGD or Federated Averaging~\citep{wang21_field_guide_to_feder_optim}. It is a generalization of minibatch SGD that, rather than communicating at every step of the optimization process, communicates only intermittently. Local SGD shows remarkable efficiency in many settings in practice, and therefore its convergence and generalization properties have been the
subject of intense theoretical investigation over the past few
years~\citep{stich18_local_sgd_conver_fast_commun_littl,khaled19_tight_theor_local_sgd_ident_heter_data,woodworth20_is_local_sgd_better_than_minib_sgd,woodworth20_minib_vs_local_sgd_heter_distr_learn,patel2023on,glasgow21_sharp_bound_feder_averag_local,gu23_why_when_does_local_sgd,patel24_limit_poten_local_sgd_distr}.
Many variants of Local SGD exist, including those that use random reshuffling instead of
i.i.d. sampling
locally~\citep{yun21_minib_vs_local_sgd_with_shuff,mishchenko21_proxim_feder_random_reshuf},
adaptive methods such as
Adam~\citep{reddi20_adapt_feder_optim,wang22_commun_effic_adapt_feder_learn},
and modifications to handle data
heterogeneity~\citep{karimireddy19_scaff,mitra21_achiev_linear_conver_feder_learn},
personalization~\citep{hanzely20_lower_bound_optim_algor_person_feder_learn}, or
additionally use gradient
compression~\citep{haddadpour20_feder_learn_with_compr,safaryan21_smoot_matric_beat_smoot_const}. Generalized Local SGD, where we use two stepsizes (as in Algorithm~\ref{alg:fed-opt}), is known to be important in managing the trade-off between converging quickly and converging to a mismatched point in heterogeneous distributed optimization~\citep{woodworth20_minib_vs_local_sgd_heter_distr_learn,charles20_outsiz_impor_learn_rates_local_updat_method,patel24_limit_poten_local_sgd_distr}. Our focus here is on the \emph{homogeneous} or i.i.d. data setting; Here, the most related works are \citep{karimireddy19_scaff,malinovsky22_server_side_steps_sampl_without,jhunjhunwala23_fedex,sun23_role_server_momen_feder_learn} and we discuss our work's relation to theirs in detail in the next section after reviewing some preliminaries.

\section{Theory}

In this section we conduct the study our main algorithm, Generalized Local SGD (Algorithm~\ref{alg:fed-opt} with $\mathrm{LocalUpdate}(y, g) = y - \eta g$ and $\mathrm{OuterUpdate}(x, \Delta) = x - \gamma \Delta$). We first review some preliminaries, then present our main results.

\subsection{Preliminaries}
We are solving the optimization problem
$\displaystyle \min_{x \in \mathbb{R}^d} f(x)$,
where we assume $f$ satisfies the following curvature and regularity condition.
\begin{assumption}\label{asm:f-cvx-smooth}
The function $f$ is differentiable, convex, has $L$-Lipschitz gradients, and has a minimizer $x_{\ast}$.
\end{assumption}
\vspace{-2mm}
We suppose that we can access a \emph{stochastic first-order oracle} that given a point $x$ returns a gradient $g (x)$ that satisfies the following assumption.
\begin{assumption}\label{asm:stoch-gradients}
Given a point $x \in \mathbb{R}^d$, the stochastic gradients $g(x) \in \mathbb{R}^d$ are (a) unbiased in expectation $\ec{g(x)} = \nabla f(x)$, and (b) has variance bounded as $\ecn{g(x) - \nabla f(x)} \leq \sigma^2$, where $\ec{\cdot}$ denotes the expectation operator.
\end{assumption}
Our setting is distributed, but with identically distributed data: there are $M$ different nodes, but they all sample stochastic gradients from the same data distribution in an i.i.d. (independent and identically distributed) manner.
We denote the inner product between two vectors $a$ and $b$ by $\ev{a, b}$ and by $\norm{\cdot}$ the corresponding Euclidean norm. For the purpose of theoretical analysis, can write Generalized Local SGD succinctly as
\begin{align}
  \vspace{-1mm}
  y_{m, r, 0} &= x_r, \qquad g_{m,r , h} = \text{Stochastic gradient of $y_{m, r, h}$} \nonumber  \\
  y_{m, r, h+1} &= y_{m, r, h} - \eta g_{m, r, h}, \text { for } m = 1, \ldots, M \text { in parallel and } h=0,1, \ldots, H-1 \text { in sequence. } \nonumber \\
  x_{r+1} &= x_r - \gamma \eta \sum_{h=0}^{H-1} \frac{1}{M} \sum_{m=1}^M g_{m, r, h}. \tag{GEN-LOC-SGD} \label{eq:alg-loc-sgd}
\end{align}
\vspace{-1mm}
To simplify our analysis, we follow~\citep{stich18_local_sgd_conver_fast_commun_littl} and define the virtual sequences
\begin{align}\label{eq:virtual-sequences}
\vspace{-1mm}
 \begin{split}
  y_{r, h} \eqdef \frac{1}{M} \sum_{m=1}^M y_{m, r, h}, \qquad\qquad g_{r, h} \eqdef \frac{1}{M} \sum_{m=1}^M g_{m, r, h} \\
  \overline{g}_{m, r, h} \eqdef \ec[r,h-1]{g_{m, r, h}} = \nabla f(y_{m, r, h}),
  \qquad\qquad \overline{g}_{r, h} \eqdef \ec[r,h-1]{g_{r, h}}.
 \end{split}
\end{align}

\vspace{-1mm}
\subsection{Main convergence result}

Recall that we consider Algorithm~\ref{alg:fed-opt} the particular case when $\mathrm{LocalUpdate}(y, g) = y - \eta g$ and $\mathrm{OuterUpdate}(x, \Delta) = x - \gamma \Delta$. 

\textbf{Existing results on the convergence of Gen. Local SGD.} When the outer stepsize $\gamma = 1$, the convergence of~\eqref{eq:alg-loc-sgd} is very well understood, with tightly matching upper and lower bounds~\citep{khaled19_tight_theor_local_sgd_ident_heter_data,woodworth20_is_local_sgd_better_than_minib_sgd,glasgow21_sharp_bound_feder_averag_local}. In particular, the best rate for the algorithm is
\begin{equation}
\label{eq:vanilla-loc-sgd-rate}
  \!\mathbb{E}\!\left[\!f\!\left(\! \frac{1}{RH}\! \sum_{r=0}^{R-1}\! \sum_{h=0}^{H-1}\! y_{r, h} \!\right)\!\right]\! - f(x_{\ast}) \leq \mathcal{O}\! \left(\! \frac{L \sqn{x_0 \! - x_{\ast}}}{RH} + \frac{\sigma \norm{x_0 \! - x_{\ast}}}{\sqrt{M R H}} + \frac{L^{\frac{1}{3}} \sigma^{\frac{2}{3}} \norm{x_0 \! - x_{\ast}}^{\frac{4}{3}}}{H^{\frac{1}{3}} R^{\frac{2}{3}}} \!  \right).
\end{equation}
The first two terms in the above convergence guarantee show that increasing the number of local steps has the same effect as increasing the number of communication rounds $R$, and are identical to the convergence guarantee of doing $RH$ steps of SGD with minibatch size $M$. Local SGD differs from ordinary minibatch SGD in the last term, which shows different scaling between $H$ and $R$, where increasing $R$ helps more than increasing $H$. This is because increasing $H$ incurrs additional \emph{client drift} that slows down the convergence of the algorithm in the presence of stochastic gradient noise. When the outer stepsize $\gamma$ is allowed to vary, the convergence of the algorithm is less clear. \citet{karimireddy19_scaff} gives the following convergence rate in the absence of data heterogeneity,
\begin{align*}\textstyle
  \ec{f\left (\frac{1}{R} \sum_{r=0}^{R-1} x_r \right)} - f(x_{\ast}) \leq \mathcal{O} \left( \frac{L \sqn{x_0 - x_{\ast}}}{R} + \frac{\sigma \norm{x_0 - x_{\ast}}}{\sqrt{M R }} \right),
\end{align*}
for specially chosen $\eta$ and $\gamma$ pairs. This rate matches that of Minibatch SGD, but does not recover the convergence rate of vanilla Local SGD given by \Cref{eq:vanilla-loc-sgd-rate}. \citet{jhunjhunwala23_fedex} also give a guarantee for Generalized Local SGD with a specific outer learning rate that is always at least $1$ and that depends on the heterogeneity of the iterates across the different clients. Since the analysis is conducted in the heterogeneous setting, the local stepsize required to scale with $1/H$. A guarantee that applies to any outer learning rate in the nonconvex, heterogeneous setting given by~\citep{sun23_role_server_momen_feder_learn}. 

The limiting factor in existing analysis is that we are forced to choose the local stepsize $\eta$ to scale as $\frac{1}{L H}$, whereas to obtain \cref{eq:vanilla-loc-sgd-rate} we sometimes need to choose $\eta$ to be much larger, on the order of $\frac{1}{L}$. If we aim to accurately characterize the convergence of~\eqref{eq:alg-loc-sgd}, our analysis has to encompass both large and small local stepsizes $\eta$.

\textbf{New analysis.} We now present our main convergence theorem for~\eqref{eq:alg-loc-sgd}.

\begin{theorem}\label{thm:gen-loc-sgd}
  Suppose that \Cref{asm:f-cvx-smooth,asm:stoch-gradients} hold. Then the iterates generated by Generalized Local SGD run with local stepsize $\eta > 0$ and outer stepsize $\gamma > 0$ for $R$ communication rounds and with $H$ local steps per round satisfy,
  \begin{align}
\label{eq:main-thm-guarantee}\textstyle
    \!\mathbb{E}\!\left[ f\! \left(\! \frac{1}{RH}\! \sum_{r=0}^{R-1} \! \sum_{h=0}^{H-1} \! y_{r, h} \! \right) \! \right] \! - f(x_{\ast}) \leq \mathcal{O} \left( \! \frac{\sqn{x_0 - x_{\ast}}}{\eta \gamma R H}  + \frac{\eta \sigma^2 \max(\gamma, 1)}{M} + L \eta^2 \sigma^2 H \! \right),
  \end{align}
  provided the stepsizes $\eta$ and $\gamma$ jointly satisfy $\eta L (1+ (\gamma-1)_+ H) \leq \frac{1}{4}$ and where $(a)_+ = \max(a, 0)$.
\end{theorem}
\textbf{Implications of \Cref{thm:gen-loc-sgd}.} Before giving a proof sketch for \Cref{thm:gen-loc-sgd}, we first discuss its implications. Observe the stepsize condition $\eta L (1+(\gamma-1)_+ H)$ is asymmetric in $\eta$ and $\gamma$; That is, when $\gamma \leq 1$, we are allowed to choose $\eta$ larger than $\Omega (\frac{1}{LH})$. This is crucial to obtain the rate of \Cref{eq:vanilla-loc-sgd-rate}. Indeed, when $\gamma = 1$, the requirement on $\eta$ reduces to $\eta L \leq \frac{1}{4}$ and we can choose $\eta$ following~\citep{woodworth20_is_local_sgd_better_than_minib_sgd} as
\begin{align*}\textstyle
    \eta = \min \left ( \frac{1}{4L}, \sqrt{\frac{M \sqn{x_0 - x_\ast}}{\sigma^2 R H}}, \left [ \frac{\sqn{x_0 - x_{\ast}}}{L \sigma^2 H^2 R} \right]^{\frac13} \right)
\end{align*}
Plugging this choice of $\eta$ yields the convergence guarantee of \Cref{eq:vanilla-loc-sgd-rate}. Alternatively, when $8 \eta L \leq 1$, the stepsize requirement is met if we choose $\eta \gamma L H \leq \frac{1}{8}$ and we immediately get the Minibatch SGD guarantee. In particular, choose $\eta\! = \!\mathcal{O}\! \left(\! \frac{1}{R L} \! \right)$ and $\gamma\! =\! \mathcal{O}\! \left( \!\frac{\gamma_{\ast}}{\eta L H}\!\right)$, the rate then becomes
\begin{align*}\textstyle
f(y_{\mathrm{out}}) - f(x_{\ast}) \leq \frac{8 L \sqn{x_0 - x_{\ast}}}{\gamma_{\ast} R} + \frac{\sigma^2 H}{8 R^2 L} + \frac{\gamma_{\ast} \sigma^2}{4L M H},
\end{align*}
where $y_{\mathrm{out}}$ denotes the average over all iterations and clients as in \Cref{eq:main-thm-guarantee}. Then for $R$ large enough we can choose $\gamma_{\ast} = \mathcal{O} \left( \sqrt{\frac{L D^2 \sigma^2 M H}{R \sigma^2}} \right)$ and this gives us the minibatch SGD rate
\begin{align*}\textstyle
f(y_{\mathrm{out}}) - f(x_{\ast}) \leq \frac{LD^2}{R} + \frac{\sigma D}{\sqrt{M R H}}.
\end{align*}
This confirms the intuition that at the extremes, manipulating the stepsizes $\gamma$ and $\eta$ allows us to interpolate between minibatch SGD and (vanilla) Local SGD, as observed by~\citep{woodworth20_minib_vs_local_sgd_heter_distr_learn}. In fact, \Cref{thm:gen-loc-sgd} allows us to go a step further and get an explicit expression for the optimal inner and outer stepsizes depending on the problem parameters. This is given by the following proposition.

\begin{proposition}\label{prop:optimal-eta-gamma}
\begingroup
\setlength{\abovedisplayskip}{1pt}  
\setlength{\belowdisplayskip}{3pt}  
  Let $h(\eta, \gamma)$ be defined as
  \begin{align}\textstyle
    h(\eta,\gamma) = \frac{D^2}{\eta\gamma R H} + L\sigma^2 H\eta^2 + \frac{\eta \max(\gamma, 1) \sigma^2}{M}.
  \end{align}
\vspace{-1mm}
Consider the optimization problem:
\begin{align}
\label{eq:stepsize-opt}\textstyle
\min_{\eta>0, \gamma>0} h(\eta,\gamma) \qquad \text{subject to} \quad &\eta L\left(1+(\gamma-1)_+ H\right)\leq\frac{1}{4}.
\end{align}
  The solution $(\eta^*, \gamma^*)$ is given by comparing the following two candidates.
  \vspace{-1mm}
  \begin{enumerate}
    \item Candidate $(\eta_A^*, \gamma_A^*)$ defined by $\gamma_A^{*} = 1$ and $\eta_A^{*} = \min (\frac{1}{4L}, \eta_A^{\prime})$ where $\eta_A^{\prime}$ is the unique positive root of the cubic equation
      \begin{align*}\textstyle
        2L H\sigma^2\eta^3 + \frac{\sigma^2}{M}\eta^2 - \frac{D^2}{RH} = 0.
      \end{align*}
    \item Candidate $(\eta_B^*, \gamma_B^*)$ for the regime $\gamma \geq 1$ with $4 \eta L < 1$, where (a) the constraint is enforced with equality:
            \begin{align*}\textstyle
              \gamma_B (\eta) = 1 + \frac{1}{H} \left( \frac{1}{4 L \eta} - 1 \right),
            \end{align*}
          and (b) $\eta_B^{\ast}$ is the unique positive root of the cubic equation
          \vspace{2mm}
          \begin{align*}\textstyle
            -\frac{4 L^{2} D^{2} (H-1)}{R} + 2 L \sigma^{2} H \eta \left(\eta L (H-1)+1\right)^{2}+ \frac{\sigma^{2}(H-1)}{M H} (\eta L (H-1) + 1)^2 = 0.
          \end{align*}
  \end{enumerate}
  The optimal solution $(\eta^*, \gamma^*)$ is the candidate pair from $\{(\eta_A^*, \gamma_A^*), (\eta_B^*, \gamma_B^*)\}$ that yields the smaller value of $h(\eta, \gamma)$.
\endgroup
\end{proposition}

The proof of the above proposition is straightforward and follows by writing the KKT conditions for the optimization problem in \Cref{eq:stepsize-opt}. A consequence of \Cref{prop:optimal-eta-gamma} is that in the case of ill-tuning of the inner stepsize $\eta$, a large outer stepsize $\gamma$ can make up for it. For example, if $\sigma \rightarrow 0$ and $\eta L H \ll \mathcal{O}(1)$, we can make up for this by choosing $\gamma$ as $\frac{1}{\eta L H}$. Thus, we can interpret the outer learning rate $\gamma$ as having \textbf{two dual roles.} (a) It allows us to interpolate between minibatch SGD $(\gamma > 1)$ and vanilla Local SGD $(\gamma = 1)$, giving us the better of the two rates, and (b) it provides us some additional leeway in hyperparameter tuning by making up for ill-tuned inner learning rate $\eta$.

Our theory suggests that \emph{in the worst case}, choices of $\gamma < 1$ are \emph{not} useful from an optimization perspective. We should either choose $\gamma = 1$ or $\gamma > 1$. This can be seen even on quadratic objectives, for example if $f(x) = \frac{x^{\top} Q x}{2}$ for some positive definite matrix $Q$, then a straightforward computation gives the expected iterate after $H$ local steps and $R$ communication rounds is $\ec{x_R} = ( (1-\gamma) I + \gamma (I- \eta Q)^H) x_0$. From this, it is clear that if $\eta$ is chosen such that $(I-\eta Q)^H$ has eigenvalues smaller than $1$, we should choose $\gamma \geq 1$. While if $(I-\eta Q)^H$ has any eigenvalues larger than $1$, we should just choose $\gamma = 0$ (i.e. just don't apply the algorithm at all). In other words, $\gamma$ can make up for a learning rate that is too small, but not a learning rate that is too large. This observation does not exclude that $\gamma < 1$ can be useful from a \emph{generalization} perspective, as noted for the case of a single client by~\citet{zhou21_lookahead}, in the presence of data heterogeneity, as noted by~\citet{charles21_conver_accur_trade_offs_feder}, or in the presence of specific stochastic gradient distributions (see \Cref{sec:data-depend-conv}).

\textbf{Proof sketch for \Cref{thm:gen-loc-sgd}.} We first start by expanding the update for the round iterate $x_{r+1} - x_{\ast} = x_{r+1} - x_r + x_r - x_{\ast}$ similar to~\citep{karimireddy19_scaff} to get,
\vspace{-2mm}
\begin{align*}\textstyle
  &\sqn{x_{r+1} - x_{\ast}} = \sqn{x_r - x_{\ast}} - 2 \gamma \eta \sum_{h=0}^{H-1} \ev{ x_r - x_{\ast} , g_{r, h} } + \gamma^2 \eta^2 \sqn{\sum_{h=0}^{H-1} g_{r, h}} \\[-2mm]
                           &\qquad = \sqn{x_r - x_{\ast}} - 2 \gamma \eta \sum_{h=0}^{H-1} \ev{ x_r - y_{r, h} , g_{r, h} } - 2 \gamma \eta \sum_{h=0}^{H-1} \ev{ y_{r, h} - x_{\ast} , g_{r, h} } + \gamma^2 \eta^2 \sqn{\sum_{h=0}^{H-1} g_{r, h}},
\vspace{-3mm}
\end{align*}
where $g_{r, h}$ is defined as in \Cref{eq:virtual-sequences}. \citet{karimireddy19_scaff,jhunjhunwala23_fedex} control the inner product $- \ev{ x_r - y_{r, h} , g_{r, h} }$ by either using smoothness or Young's inequality; This would force us to bound the stray $\sqn{y_{r, h} - x_r}$ and take the local stepsize $\eta$ to be small in order to ensure convergence. Instead, we rely on bounding this quantity directly by viewing it as the \emph{regret} in the online convex optimization sense with respect to the comparator $x_r$. Observe that the virtual sequence of averaged local iterates satisfies $y_{r, h+1} = y_{r, h} - \eta g_{r, h}$, and thus through standard regret analysis we have
\vspace{-1mm}
\begin{align}\textstyle
\label{eq:regret}
  \sum_{h=0}^{H-1} - \ev{ x_r - y_{r, h} , g_{r, h} } &= \frac{-\sqn{y_{r, h} - x_r}}{2 \eta} + \frac{\eta}{2} \sum_{h=0}^{H-1} \sqn{g_{r, h}}.
\end{align}
The negative terms $-\sqn{y_{r, H} - x_r}$ in \cref{eq:regret} turn out to be crucial in obtaining an analysis that works for all $\eta$ and not just small $\eta$. With this change and through carefully bounding the variance terms following~\citep{khaled19_tight_theor_local_sgd_ident_heter_data,woodworth20_is_local_sgd_better_than_minib_sgd}, we obtain the guarantee of \Cref{thm:gen-loc-sgd}. The full proof is provided in \Cref{sec:non-adaptive}.

\textbf{Comparison with results on related algorithms.} \citet{malinovsky22_server_side_steps_sampl_without} analyze a closely related variant of the algorithm that uses federated random reshuffling~\citep{mishchenko21_proxim_feder_random_reshuf} as a base. This is a significantly different algorithm that doesn't allow for an arbitrary number of local steps $H$ and depends on $f$ posessing finite-sum structure. Nevertheless, we can still specialize~\citep[Theorem 2]{malinovsky22_feder_random_reshuf_with_compr_varian_reduc} approximately to our setting, by using $H$ as the number of data points in an epoch. In our notation, their convergence guarantee reads
\vspace{-2mm}
\begin{align*}\textstyle
  \ec{f \left( \frac{1}{R} \sum_{r=0}^{R-1} x_r \right)} - f(x_\ast) \leq \mathcal{O} \left( \frac{\sqn{x_0 - x_{\ast}}}{\eta \gamma H R} + \eta^2 H^2 \sigma^2 \right),
\end{align*}
under the conditions $\eta H \leq \frac{1}{L}$ and $1 \leq \gamma \leq \frac{1}{L \eta H}$. Their theory thus also suggests that $\gamma \geq 1$ can be useful. Optimizing over $\eta$ and $\gamma$ yields the convergence rate
\begin{align*}\textstyle
\ec{f \left( \frac{1}{R} \sum_{r=0}^{R-1}x_r \right)} - f(x_\ast) \leq \mathcal{O} \left( \frac{L \sqn{x_0 - x_{\ast}}}{R}  \right),
\end{align*}
this rate is the same as gradient descent for $R$ steps (since the finite-sum structure means that per-epoch we approximate one step of gradient descent when $\eta$ is small). A similar rate is derived in \citep{li24_power_extrap_feder_learn,li24_conver_fedpr_with_extrap_inexac_prox} if we have access to the proximal operator (i.e. we can do \emph{many} local steps $H$ on a modified objective). \citet{li24_power_extrap_feder_learn} in particular show that an outer learning rate greater than $1$ can be particularly useful for improving the convergence of FedProx~\citep{li18_feder_optim_heter_networ} in the heterogeneous setting when the smoothness constant varies significantly between different clients.

\textbf{Analysis with momentum.} Our analysis suggests that values of $\gamma > 1$ are potentially very useful, but in practice such values are rarely used. One reason this might be the case is because the momentum effectively acts as a stepsize multiplier, i.e. in the presence of momentum parameter $\mu$ the effective outer stepsize becomes $\frac{\gamma}{1-\mu}$. Our next theorem establishes this rigorously.

\begin{theorem}\label{thm:momentum}
Suppose that \Cref{asm:f-cvx-smooth,asm:stoch-gradients} hold. Suppose that the outer update is gradient descent with momentum, $\mathrm{OuterUpdate}(x_r, -\Delta_{r, H}) = x_r + \gamma \Delta_{r, H} + \mu (x_r - x_{r-1})$ with momentum parameter $\mu \in [0,1)$ and the local update is gradient descent $\mathrm{LocalUpdate}(y, g) = y - \eta g$ in \Cref{alg:fed-opt}. Let the step sizes $\eta, \gamma$ satisfy $\eta L \left(1 + \left(\frac{\gamma}{1-\mu}-1\right)_+ H\right) \leq \frac{1}{4}$ and $\frac{\eta \gamma \mu L H}{1-\mu} \leq \frac{1}{16}$. Then after $R$ rounds of communication, the averaged iterate satisfies
\begin{align*}
\begin{split}\textstyle
      \mathbb{E}\left[f(\overline{y})\right] - f(x_*) &\leq \mathcal{O} \left( \frac{ (1-\mu) \sqn{z_0 - x_{\ast}}}{\eta \gamma H R} + L \eta^2 \sigma^2 H + \frac{\eta \sigma^2 \max ( \frac{\gamma}{1-\mu}, 1 )}{M} + \frac{\eta \gamma \mu}{1-\mu} \frac{\sigma^2}{M}  \right),
\end{split}
\end{align*}
where $\bar{y}$ is defined as the average of all local iterates across training (as in \Cref{eq:main-thm-guarantee}) and $(a)_+ = \max(a, 0)$.
\end{theorem}
The proof is provided in \Cref{sec:proof-momentum}. Theorem~\ref{thm:momentum} shows the requirement on the outer stepsize is relaxed from a requirement on $\gamma$ to a requirement on $\frac{\gamma}{1-\mu}$, allowing us to reap the same benefits of $\gamma > 1$ observed earlier if we also tune $\mu$. Momentum thus changes the range of stepsizes allowed but does not fundamentally alter the uter stepsize tradeoffs. This benefit was first observed in~\citep{sun23_role_server_momen_feder_learn} for nonconvex optimization with small local stepsize $\eta$ provided we use an additional momentum buffer. Our work gives direct theoretical support to this observation even with a single momentum buffer and allowing for large $\eta$.

\subsection{Convergence with accelerated outer optimizer}

We now consider the use of acceleration. To the best of our knowledge, the combination of an accelerated outer optimizer with an unaccelerated inner optimizer, as in e.g. DiLoCo~\citep{douillard23_diloc,jaghouar24_opend}, has not been analyzed in the literature before. We take steps towards addressing this gap and understanding the convergence properties of such algorithms by considering Nesterov's accelerated gradient descent~\citep{nesterov18_lectures_cvx_opt} as the outer optimizer and (stochastic) gradient descent as the inner optimizer. The following theorem gives a convergence guarantee for this setting.

\begin{theorem}\label{thm:accelerated}
Suppose that \Cref{asm:f-cvx-smooth,asm:stoch-gradients} hold and the stepsizes satisfy $2 L \eta \leq 1$ and $\gamma \leq 1$. Suppose that the outer update is accelerated gradient descent with Nesterov momentum as follows
\begin{align*}\textstyle
    u_{r+1} = x_r - \eta \Delta_{r,H}, && z_{r+1} = z_r - \gamma_r \eta \Delta_{r,H}, && x_{r+1} = (1-\tau_{r+1}) u_{r+1} + \tau_{r+1} z_{r+1},
\end{align*}
with parameters $\gamma_r = \frac{\gamma (r+1)}{2}$ and $\tau_r = \frac{2}{r+2}$, and the local update is gradient descent $\mathrm{LocalUpdate}(y, g) = y - \eta g$ in \Cref{alg:fed-opt}. Then after $R$ rounds of $H$ steps, the final iterate $u_R$ satisfies
\begin{align}
\ec{f(u_R)} - f(x_*) \leq \frac{2\sqn{x_0 - x_{\ast}}}{\gamma \eta R^2H} + \frac{RL\eta^2 \sigma^2 H}{2M} + \frac{RL^2 \eta^3 \sigma^2 H^2}{2} + \frac{\gamma \eta \sigma^2 R}{2M}. \label{eq:accelerated-final}
\end{align}
\end{theorem}

To understand the implications of the above guarantee, we specialize it with a tuned pair of learning rates $(\gamma, \eta)$ below.

\begin{corollary}\label{corr:acceleration}
In the same setting as \Cref{thm:accelerated}, setting $\gamma = 1$ in \Cref{eq:accelerated-final}, and choosing
\begin{align*}\textstyle
\eta=\min \left \{\frac{1}{2L}, \left(\frac{2M D^2}{R^3 L \sigma^2 H^2}\right)^{1/3},  \left(\frac{4D^2}{3 R^3 L^2 \sigma^2 H^3}\right)^{1/4}, \sqrt{\frac{4MD^2}{R^3 H \sigma^2}} \right \},
\end{align*}
where $D = \norm{x_0 - x_{\ast}}$ and the final iterate $u_R$ satisfies
\begin{align}\textstyle
  \mathbb{E}[f(u_R)]-f(x_*) &\leq \mathcal{O} \Bigl ( \frac{L D^2}{R^2 H} + \frac{L^{1/3}\sigma^{2/3}D^{4/3}}{R M^{1/3} H^{1/3}} + \frac{L^{1/2}\sigma^{1/2} D^{3/2}}{R^{5/4}H^{1/4}} + \frac{\sigma D}{\sqrt{M R H}} \Bigr). \label{eq:accelerated-specialized}
\end{align}
\end{corollary}

\Cref{eq:accelerated-specialized} shows that in the absence of noise, we obtain a rate accelerated in $R$ but not $H$. This intuitively makes sense, since we do acceleration only in the outer loop. In the presence of noise, we have in the worst-case the unimprovable $\frac{\sigma D}{\sqrt{M R H}}$ term and two additional noise terms that characterize the drift suffered by this algorithm. Notably, the drift terms have much better dependence on $R$ compared to Local SGD, as given by \Cref{eq:vanilla-loc-sgd-rate}. \citet{yuan20_feder_accel_stoch_gradien_descen} analyze FedAC, an accelerated variant of Local SGD that uses acceleration \emph{locally} and applies simple averaging as the outer optimizer. Their algorithm enjoys the convergence rate
\begin{align*}\textstyle
    \mathbb{E}[f(x_{\mathrm{out}})] - f(x_{\ast}) \leq \mathcal{O} \left ( \frac{L D^2}{R^2 H} + \frac{L^{1/3} \sigma^{2/3} D^{4/3}}{R H^{1/3}} + \frac{L^{1/2} \sigma^{1/2} D^{3/2}}{R H^{1/4}} + \frac{\sigma D}{\sqrt{M R H}} \right ).
\end{align*}
Comparing with \cref{eq:accelerated-specialized}, our algorithm enjoys better dependence on $R$ and $M$ in the denominators of the two drift terms while using momentum sequences only on the server.

\subsection{Data-dependent convergence result}
\label{sec:data-depend-conv}
To further understand the role of the outer stepsize, we now present a data-dependent, high-probability guarantee for Generalized Local SGD in Theorem~\ref{thm:loc-sgd-guarantee}, compared to the rather worst-case analysis of Theorem~\ref{thm:gen-loc-sgd}. This analysis may also provide insights into practical tuning of the outer learning rate
\begin{theorem}\label{thm:loc-sgd-guarantee}
Suppose that \Cref{asm:f-cvx-smooth,asm:stoch-gradients} hold. Then in
Algorithm~\ref{alg:fed-opt} with outer update $x = x - \gamma \Delta$ and local update $y = y - \eta g$, if the local stepsize satisfies $\eta \leq \frac{1}{L}$ then with probability at least $1-\delta$ the iterates generated satisfy
\begin{align*}\textstyle
\begin{split}\textstyle
&f \left ( \frac{1}{RH} \sum_{r=0}^{R-1} \sum_{h=0}^{H-1} y_{r, h} \right) - f(x_{\ast}) \leq \tilde{\mathcal{O}} \Biggl (\frac{\sqn{x_0 - x_{\ast}}}{\gamma \eta R H} + \frac{\gamma \eta}{RH} \sum_{r, h} \sqn{g_{r, h}} + \gamma \eta \sigma^2  \\
& + \frac{\abs{1-\gamma} \eta}{RH} \sum_r \left( \sum_h \norm{g_{r, h}} \right)^2 + \frac{\eta}{\gamma H} \left( \frac{1}{M} \max_r \sum_{m, h} \norm{g_{m, r, h}} \right)^2 + \eta \sigma \sqrt{\frac{1}{MR} \sum_{m, r, h} \sqn{g_{m, r, h}} } \Biggr ).
\end{split}
\end{align*}
\end{theorem}
The proof of Theorem~\ref{thm:loc-sgd-guarantee} is provided in \Cref{sec:data-dependent-guarantees}. Compared to Theorem~\ref{thm:gen-loc-sgd}, the guarantee we obtain here is weaker in some areas, e.g., the variance term $\gamma \eta \sigma^2$ does not benefit from increasing $M$. On the other hand, this guarantee is a high-probability and data-dependent guarantee. To the best of our knowledge, this is the first high-probability convergence guarantee for Local SGD in the literature. Theorem~\ref{thm:loc-sgd-guarantee} allows us to observe another potential benefit of using $\gamma \neq 1$.  To see how, let us make the simplifying assumption that $\norm{\hat{g}_{r, h}} \approxeq G_1$ and $\norm{g_{m, r, h}} \approxeq G_2$. Observe that by the triangle inequality we have
$G_1 \leq G_2$, but in fact $G_1$ can be significantly smaller than $G_2$,
particularly in the later stages of the optimization process, due to the
variance reduction effect of averaging together the gradients on different
nodes. Then we can rewrite the above guarantee as
\begin{align}\label{eq:15}
  \hspace{-1mm} \!f\! \left ( \overline{y} \right) \!- \!f(x_{\ast}) &\leq \tilde{\mathcal{O}}\! \left( \!\frac{d_0^2}{\gamma \eta RH} +\! \gamma \eta G_1^2 +\! \gamma \eta \sigma^2
  +\! \abs{1\!-\!\gamma} \eta H G_1^2 +\! \frac{\eta H G_2^2}{\gamma} +\! \eta \sigma \sqrt{H} G_2 \! \right )
\end{align}

\vspace{-3mm}
The $\gamma$ that minimizes this upper bound is given by the following proposition.
\begin{proposition}\label{prop:opt-gamma-adaptive}
Let $g(x) = \frac{a}{x} + bx + \abs{1-x}c$ for $a, b, c \geq 0$.
\vspace{-2mm}
\begin{itemize}
  \setlength\itemsep{.1em}
\item if $a \geq b+c$, then $\sqrt{a/(b+c)}$ minimizes $g$,
\item if $b-c \geq 0$ and $a \leq b-c$, then $\sqrt{a/(b-c)}$ minimizes $g$,
\item Otherwise, $x=1$ minimizes $g$.
\end{itemize}
\vspace{-2mm}
\end{proposition}
Applying this lemma to \cref{eq:15} one can see that simple averaging is
suboptimal depending on the variance and relative magnitudes of $G_1$ and
$G_2$. In particular, the first condition in our setting is
\begin{align*}
  \frac{d_0^2}{\eta RH} + \eta HG_2^2 \gtrsim \eta (G_1^2 + \sigma^2) + \eta H G_1^2,
\end{align*}
where $\gtrsim$ indicates that the inequality holds up to constant factors of the
terms on both sides. Since $G_2 \geq G_1$, we can simplify the above condition to
$\frac{d_0^2}{\eta^2 RH} + HG_2^2 \gtrsim \sigma^2$. This condition essentially asks if the noise is large relative to the
``optimization term'' $\frac{d_0^2}{\eta^2 RH}$ or not. In
the latter case, choosing $\gamma > 1$ is helpful, and the outer optimizer acts as a
form of momentum that helps reduce the optimization term further. On the other
hand, the second condition yields $\gamma < 1$ and requires that
$\sigma^2 \gtrsim \frac{d_0^2}{\eta^2 RH} + H G_2^2$. This is an especially noise-dominated regime, which we may expect to observe
towards the end of the training process. In this case, decaying the outer
learning rate to $\gamma \ll 1$ allows the algorithm to maintain convergence despite
the high noise magnitude. When the optimization term and the noise term are of the same order, then $\gamma = 1$ is the optimal choice. 

\section{Experiments}
We conduct two sets of experiments: (a) solving convex optimization problems to provide the most direct verification of the predictions of our theory, and (b) training transformer based language models.
Due to limitations of space, we present only highlights of the results here and most of the details and ablations are provided in the supplementary materials (\Cref{sec:experiments-details}). 

\begin{wrapfigure}{r}{0.56\textwidth}
    \vspace{-7mm}
    \centering
    \hspace{-8mm}
    \subfigure[Optimal $\gamma$ vs $\sigma$.]{
        \includegraphics[width=0.29\textwidth,trim={0cm 0cm 0cm 1cm},clip]{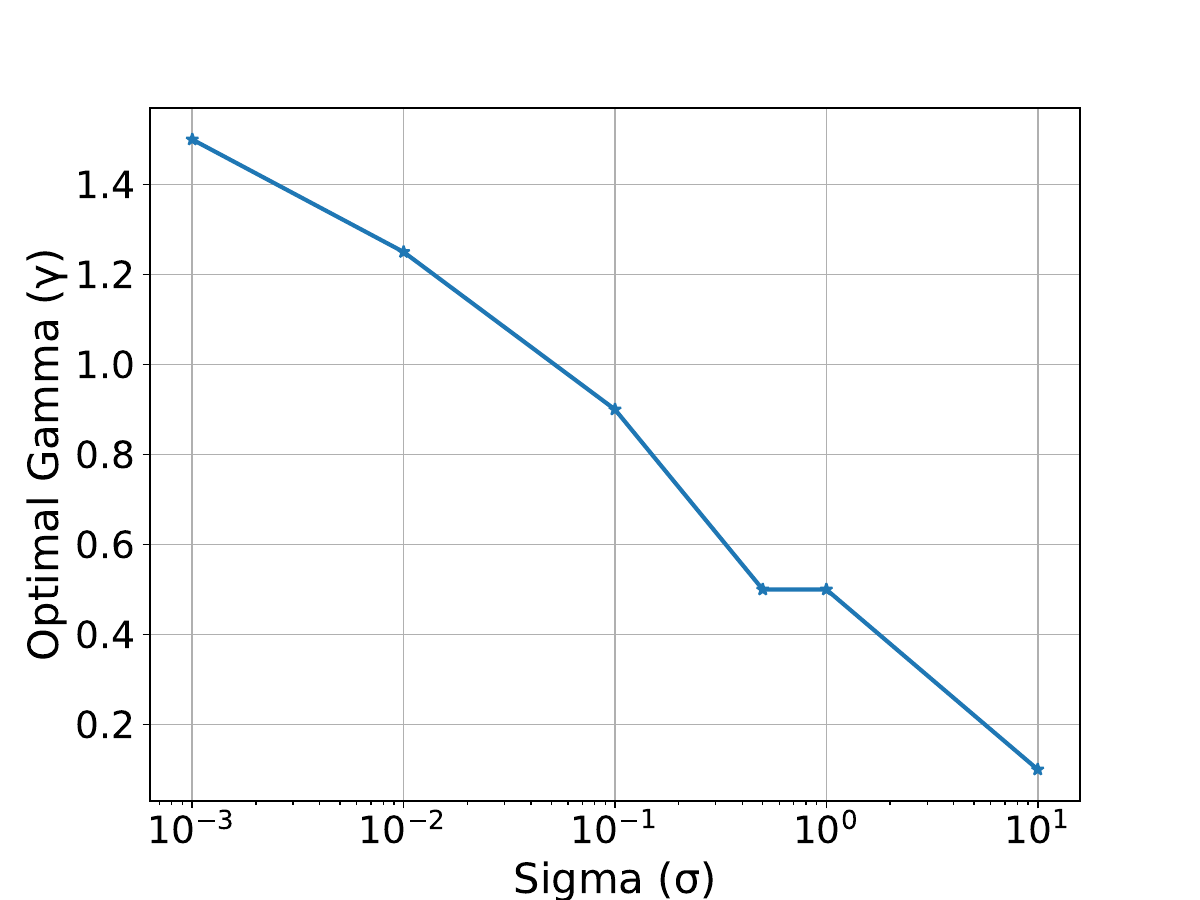}
        \label{fig:cvx_optimal_gamma}
        }
    \hspace{-5mm}
    \subfigure[Loss for different $\sigma$.]{
    \includegraphics[width=0.29\textwidth,trim={0cm 0cm 0cm 1cm},clip]{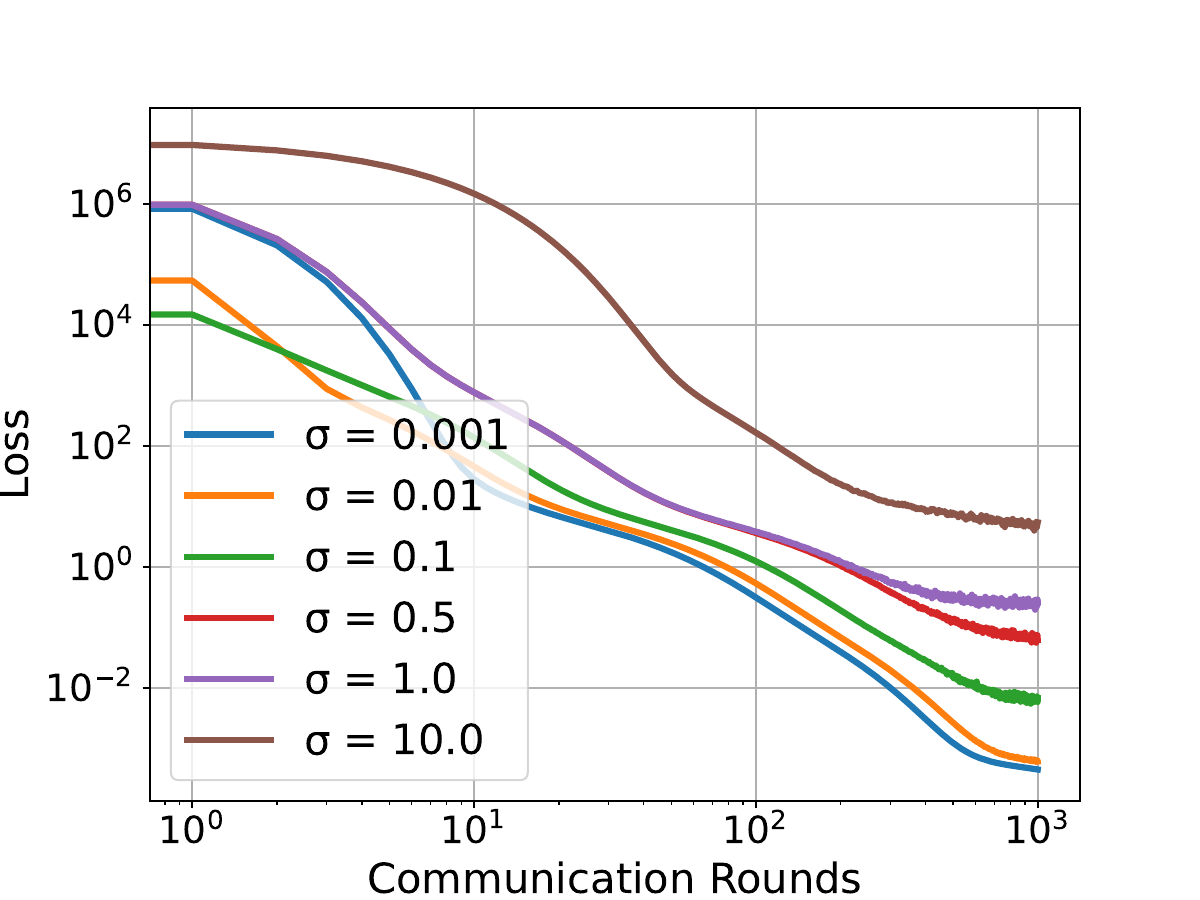}
        \label{fig:cvx_loss_trajectories}
        }
    \hspace{-10mm}
\vspace{-3mm}
\caption{Effect of varying noise magnitude $\sigma$ and outer learning rate $\gamma$ for quadratic optimization.}
\label{fig:combined_cvx_figures} 
\vspace{-5mm}
\end{wrapfigure}
\subsection{Convex optimization}
We conduct experiments on the quadratic objective
$f(x) = \frac{1}{2} \sqn{Q(x-x_{\ast})}$, where $Q = A^{\top} A \in \mathbb{R}^{d}$ for $d=50$
and the entries $A_{i, j}$ are all drawn from a normal distribution
$A_{i, j} \sim \mathcal{N}(0, 1)$ for $i=1, \ldots, d$ and $j=1, \ldots, d$, and $x_{\ast}$ is similarly
drawn from the standard $d$-dimensional Gaussian. We use stochastic gradients of
the form $g(x) = \nabla f(x) + v$, where the $v$'s are random vectors drawn from the
Gaussian with mean $0$ and variance $\sigma^2$, $v \sim \mathcal{N} (0, \sigma^2)$. We evaluate the
performance of Algorithm~\ref{alg:fed-opt} for various values of $\sigma$,
$\sigma \in \{ 10^{-3}, 10^{-2}, 10^{-1}, 0.5, 1, 5, 10, 15, 25, 50 \}$. For each $\sigma$ we perform an
extensive grid search over
$\gamma \in \{0.001, 0.01, 0.1, 0.5, 0.9, 1.0, 1.1, 1.25, 1.5, 2\}$ to determine the
best one in terms of minimum average loss over the last ten rounds. We use
$R=1000$ rounds and $H=50$ local steps, and fix $\eta = 0.001$ in all cases.

Figure~\ref{fig:cvx_optimal_gamma} shows how the optimal value of $\gamma$ varies with different noise levels $\sigma$. We observe that, as $\sigma$ increases, the optimal $\gamma$ decreases from $1.0$ to $0.1$, as predicted by our analysis. Figure~\ref{fig:cvx_loss_trajectories} also illustrates the loss trajectories for different noise levels $\sigma$ with the best $\gamma$.

\subsection{Transformer pretraining}
\vspace{-1mm}
\textbf{Setup}
Following the DiLoCo paper \citep{douillard23_diloc}, we experiment using a Chinchilla decoder transformer \citep{hoffmann2022chinchilla} on the C4 dataset \citep{c4}. 
The architecture hyperparameters are identical from the DiLoCo paper \citep{douillard23_diloc} and are given in \Cref{sec:language-models-hyperparameter}.
We fix the batch size at $512$ and the sequence length at $1024$. We experiment at different scales, from 150 million to 1 billion parameters. 
%
%
%
For all experiments, the inner optimizer is AdamW \citep{loshchilov17_decoup_weigh_decay_regul} trained with a cosine learning rate schedule defined across the total amount of steps. The inner optimizer state is never shared across replicas, and is passed from one round to the other. 

\begin{figure*}[t]
    \vspace{-2mm}
    \hfill
    \subfigure[150M]{
        \includegraphics[width=0.3\textwidth]
        {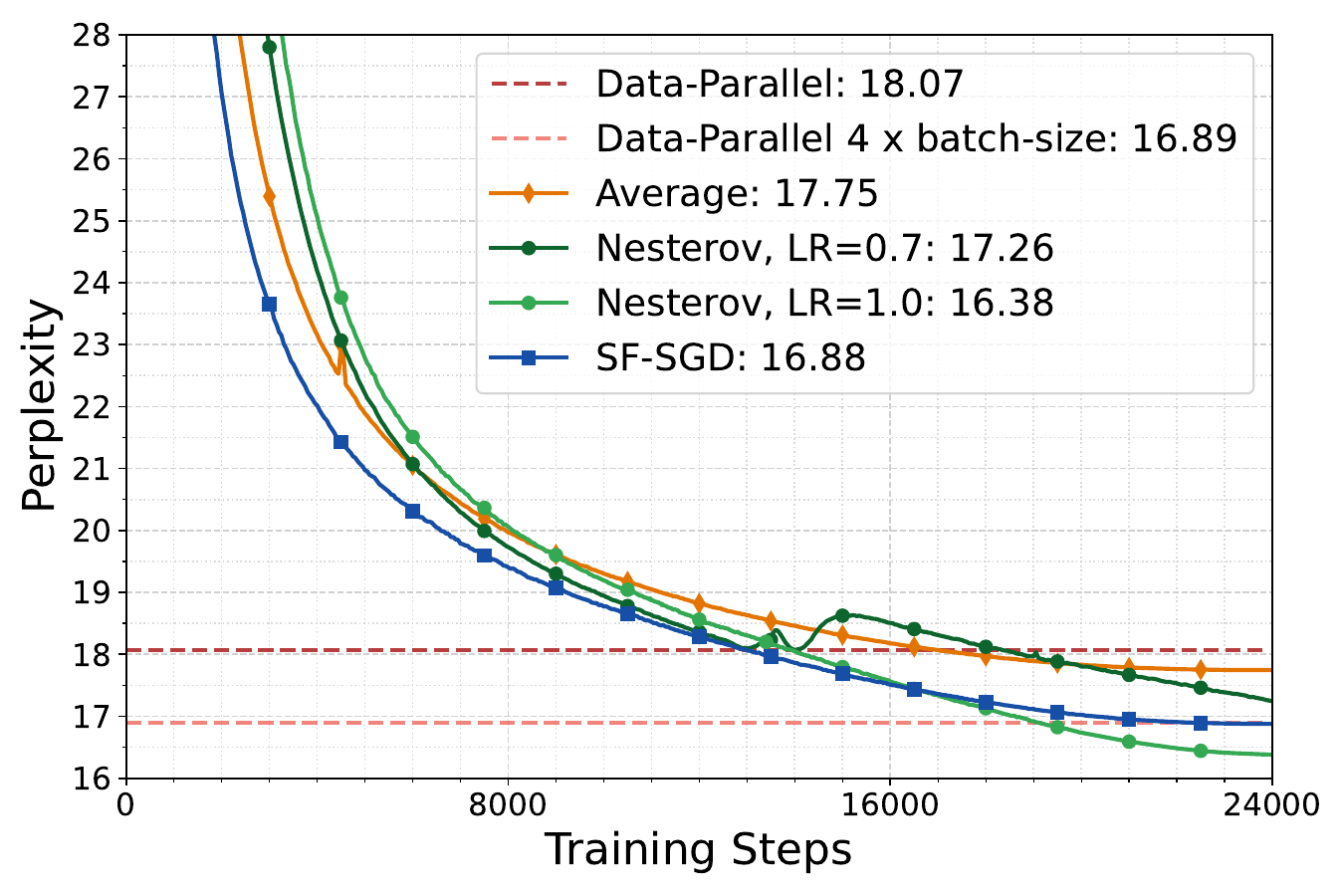}}
    \hfill
    \subfigure[400M]{
        \includegraphics[width=0.3\textwidth]
        {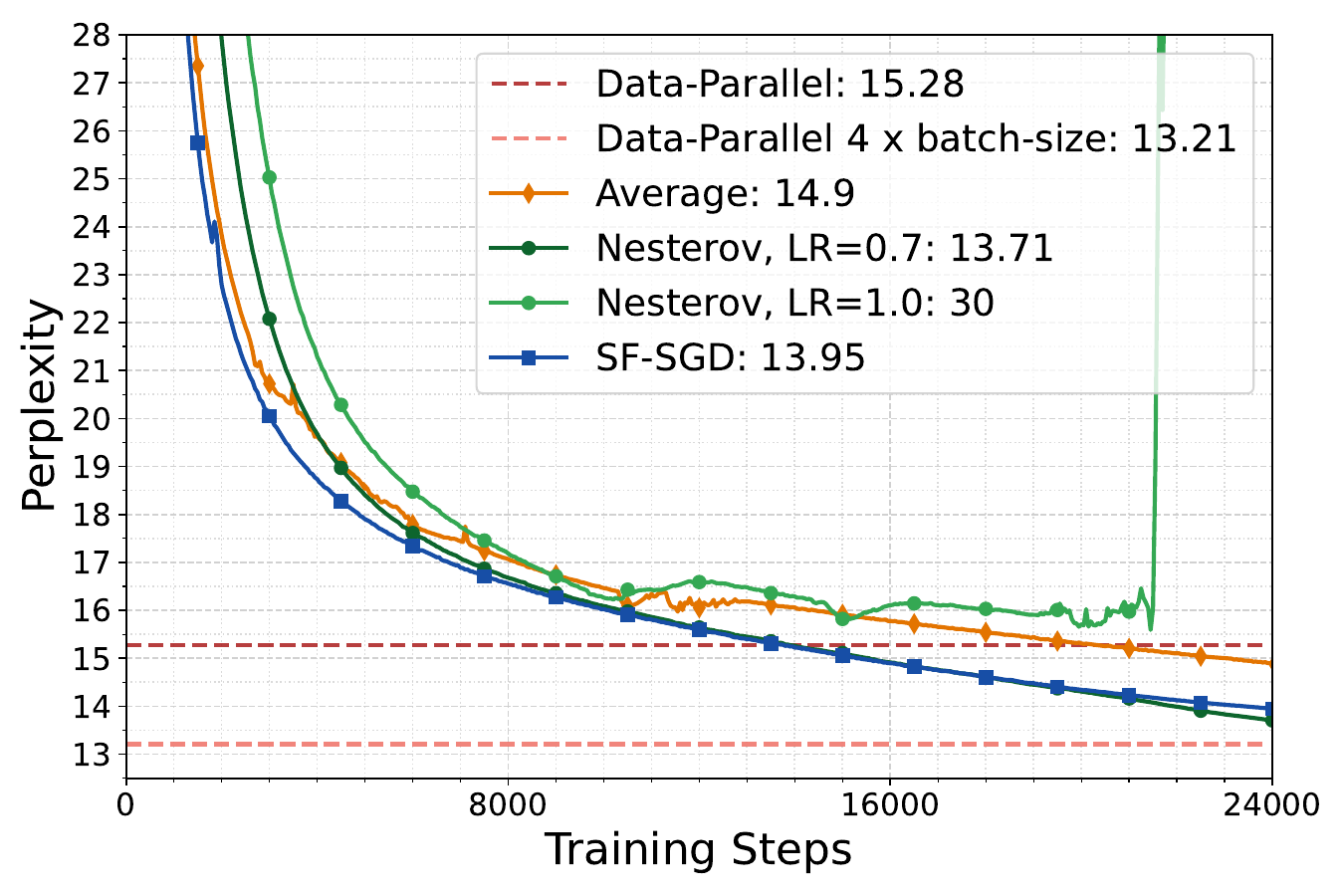}}
    \hfill
    \subfigure[1B]{
        \includegraphics[width=0.3\textwidth]
        {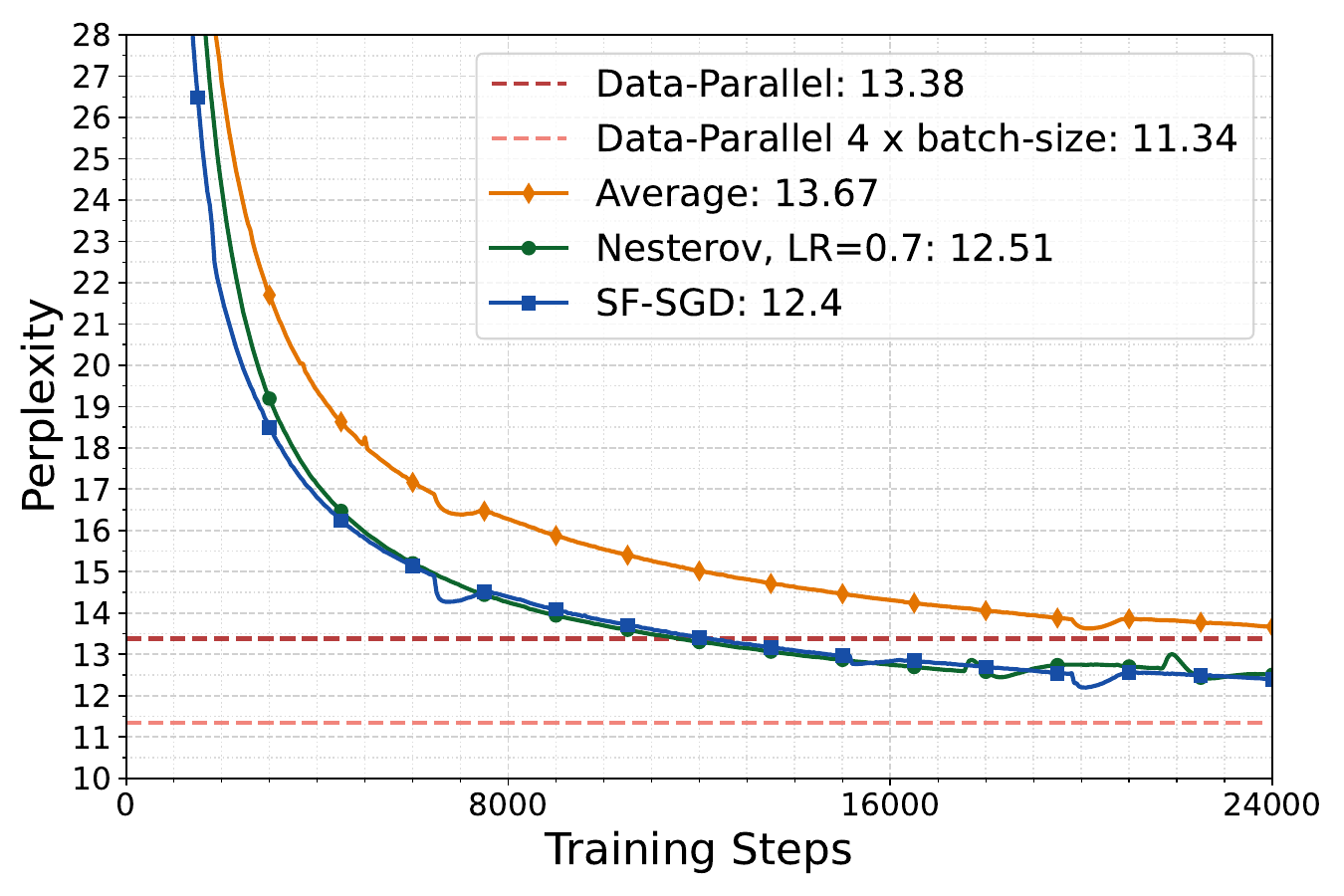}}
    \hfill
\vspace{-2mm}
\caption{\textbf{Scaling} distributed pretraining, at 150M, 400M, and 1B parameters. The x-axis shows the total number of training steps, including both local and communication steps. The y-axis shows the perplexity achieved by each method. Legend represents final perplexity values.}
\label{fig:pretraining_scaling}
\vspace{-3mm}
\end{figure*}

\textbf{Methods}
We compare three distributed methods, using different outer optimizers: SGD(lr=1) (equivalent to simple averaging of local models \citep{mcmahan2021fedavg}), Nesterov (equivalent to DiLoCo \citep{douillard23_diloc}), and ScheduleFree-SGD (SF-SGD) \citep{defazio24_road_less_sched}. We use SF-SGD to substitute for outer learning rate scheduling, though it still requires tuning hyperparameters. We also include two ``high-communication" data-parallel baselines: one with the global batch size as the local per-replica batch size used by the distributed methods, and one with the same batch size as the global batch size ($M \times$ the local per-replica batch size) used by the distributed methods. The latter requires either more GPUs and more thus communication, or gradient accumulation and thus more time. The latter also has an equal flops budget as the distributed methods.
We tuned all our optimizers on the pretraining setting on a separate validation set
. We also considered using SF-Nesterov, but it was hard to tune and unstable. 

\begin{wraptable}{r}{0.6\textwidth}
\vspace*{-4mm}
\small
\begin{tabular}{@{}l|ccc@{}}
\toprule
Hyperparameter & Selected & Range considered\\
\midrule
Number of inner steps H & 50, 500 & 50 to 2000\\
Peak outer LR for Nesterov & 0.7 & 0.1 to 2.0 \\
Peak outer LR for SF-SGD & 2.0 & $1e^{-4}$ to 10.0 \\
b1 for SF-SGD & 0.2 & 0.0 to 0.99 \\
\midrule
Peak inner learning rate (150M) & $4e^{-4}$ & $4e^{-4}$\\
Peak inner learning rate (400M) & $4e^{-4}$ & $4e^{-4}$\\
Peak inner learning rate (1B) & $2e^{-4}$ & $2e^{-4}$\\
\bottomrule
\end{tabular}
\vspace{-2mm}
\caption{\textbf{Optimizer hyperparameters} for the three evaluated sizes. All are based on the transformer architecture, chinchilla-style \citep{hoffmann2022chinchilla}.}
\label{tab:hyperparameters}
\vspace{-5mm}
\end{wraptable}
\textbf{Results}
\Cref{tab:hyperparameters} gives the optimal hyperparameters per scale, and \Cref{fig:pretraining_scaling} gives the perplexity curves. The perplexity was calculated on the C4 validation set. Consistent with the predictions of our theory, we found that an outer learning rate greater than $1.0$ performed best for SF-SGD and a relatively large effective outer learning rate also performed best for Nesterov; Moreover, acceleration consistently improved performance relative to the baseline Local SGD. In the supplementary material, we report the effect of varying the number of local steps (\Cref{sec:varying-inner-steps}), the number of clients/replicas and different ways of FLOPs allocation (\Cref{sec:pretraining_replicas}), and gradient variance (\Cref{sec:language-models-results3}). We also include the validation results for all the main experiments we ran in \Cref{tab:add_full_sweep,tab:add_lr_sweep,tab:add_beta_sweep}.

\vspace{-1mm}
\section{Conclusion and Future Work}
\vspace{-1mm}

In this paper, we studied the impact of the outer learning rate on the convergence of Local SGD through two novel convergence theorems that characterize its role in balancing a trade-off between convergence speed and stochastic gradient variance. We have also studied the impact of using momentum in the presence of an outer learning rate, and provided a new convergence analysis for using Nesterov acceleration in the outer optimizer. One limitation of our results is that we only consider the i.i.d. setting; Studying the impact of data heterogeneity is therefore a natural next step. Another avenue for future work is to investigate the role of adaptive outer optimizers in enhancing robustness to client failures and communication delays.

\printbibliography

@article{charles20_outsiz_impor_learn_rates_local_updat_method,
  author       = {Charles, Zachary and Konečný, Jakub},
  url          = {https://arXiv.org/abs/2007.00878},
  year         = {2020},
  journal      = {arXiv preprint},
  title        = {On the Outsized Importance of Learning Rates in Local Update Methods},
  volume       = {arXiv:2007.00878},
}

@inproceedings{charles21_conver_accur_trade_offs_feder,
  author    = {Charles, Zachary and Konecný, Jakub},
  editor    = {Banerjee, Arindam and Fukumizu, Kenji},
  publisher = {PMLR},
  url       = {http://proceedings.mlr.press/v130/charles21a.html},
  booktitle = {The 24th International Conference on Artificial Intelligence and Statistics, {AISTATS} 2021, April 13-15, 2021, Virtual Event},
  year      = {2021},
  pages     = {2575--2583},
  series    = {Proceedings of Machine Learning Research},
  title     = {Convergence and Accuracy Trade-Offs in Federated Learning and Meta-Learning},
  volume    = {130},
}

@inproceedings{karimireddy19_scaff,
  author    = {Karimireddy, Sai Praneeth and Kale, Satyen and Mohri, Mehryar and Reddi, Sashank J. and Stich, Sebastian U. and Suresh, Ananda Theertha},
  publisher = {PMLR},
  url       = {http://proceedings.mlr.press/v119/karimireddy20a.html},
  booktitle = {Proceedings of the 37th International Conference on Machine Learning, {ICML} 2020, 13-18 July 2020, Virtual Event},
  year      = {2020},
  pages     = {5132--5143},
  series    = {Proceedings of Machine Learning Research},
  title     = {{SCAFFOLD:} Stochastic Controlled Averaging for Federated Learning},
  volume    = {119},
}

@inproceedings{khaled19_tight_theor_local_sgd_ident_heter_data,
  author    = {Khaled, Ahmed and Mishchenko, Konstantin and Richtárik, Peter},
  editor    = {Chiappa, Silvia and Calandra, Roberto},
  publisher = {PMLR},
  url       = {http://proceedings.mlr.press/v108/bayoumi20a.html},
  booktitle = {The 23rd International Conference on Artificial Intelligence and Statistics, {AISTATS} 2020, 26-28 August 2020, Online [Palermo, Sicily, Italy]},
  year      = {2020},
  pages     = {4519--4529},
  series    = {Proceedings of Machine Learning Research},
  title     = {Tighter Theory for Local {SGD} on Identical and Heterogeneous Data},
  volume    = {108},
}

@inproceedings{konecny16_feder_learn,
  author    = {{Konečný}, Jakub and {Brendan McMahan}, H. and {Yu}, Felix X. and Richtárik, Peter and {Theertha Suresh}, Ananda and {Bacon}, Dave},
  booktitle = {NIPS Private Multi-Party Machine Learning Workshop},
  year      = {2016},
  title     = {{Federated Learning: Strategies for Improving Communication Efficiency}},
}

@inproceedings{mishchenko21_proxim_feder_random_reshuf,
  author    = {Mishchenko, Konstantin and Khaled, Ahmed and Richtárik, Peter},
  publisher = {PMLR},
  booktitle = {{ICML}},
  year      = {2022},
  pages     = {15718--15749},
  series    = {Proceedings of Machine Learning Research},
  title     = {Proximal and Federated Random Reshuffling},
  volume    = {162},
}

@article{mitra21_achiev_linear_conver_feder_learn,
  author       = {Mitra, Aritra and Jaafar, Rayana and Pappas, George J. and Hassani, Hamed},
  url          = {https://arXiv.org/abs/2102.07053},
  year         = {2021},
  journal      = {arXiv preprint},
  title        = {Achieving Linear Convergence in Federated Learning Under Objective and Systems Heterogeneity},
  volume       = {arXiv:2102.07053},
}

@article{ortiz21_trade_offs_local_sgd_at_scale,
  author       = {Ortiz, Jose Javier Gonzalez and Frankle, Jonathan and Rabbat, Mike and Morcos, Ari and Ballas, Nicolas},
  url          = {https://arXiv.org/abs/2110.08133},
  year         = {2021},
  journal      = {arXiv preprint},
  title        = {Trade-Offs of Local {SGD} At Scale: an Empirical Study},
  volume       = {arXiv:2110.08133},
}

@inproceedings{woodworth20_is_local_sgd_better_than_minib_sgd,
  author    = {Woodworth, Blake E. and Patel, Kumar Kshitij and Stich, Sebastian U. and Dai, Zhen and Bullins, Brian and McMahan, H. Brendan and Shamir, Ohad and Srebro, Nathan},
  publisher = {PMLR},
  url       = {http://proceedings.mlr.press/v119/woodworth20a.html},
  booktitle = {Proceedings of the 37th International Conference on Machine Learning, {ICML} 2020, 13-18 July 2020, Virtual Event},
  year      = {2020},
  pages     = {10334--10343},
  series    = {Proceedings of Machine Learning Research},
  title     = {Is Local {SGD} Better than Minibatch SGD?},
  volume    = {119},
}

@inproceedings{woodworth20_minib_vs_local_sgd_heter_distr_learn,
  author    = {Woodworth, Blake E. and Patel, Kumar Kshitij and Srebro, Nati},
  editor    = {Larochelle, Hugo and Ranzato, Marc'Aurelio and Hadsell, Raia and Balcan, Maria{-}Florina and Lin, Hsuan{-}Tien},
  url       = {https://proceedings.neurips.cc/paper/2020/hash/45713f6ff2041d3fdfae927b82488db8-Abstract.html},
  booktitle = {Advances in Neural Information Processing Systems 33: Annual Conference on Neural Information Processing Systems 2020, NeurIPS 2020, December 6-12, 2020, virtual},
  year      = {2020},
  title     = {Minibatch vs Local {SGD} for Heterogeneous Distributed Learning},
}

@inproceedings{yun21_minib_vs_local_sgd_with_shuff,
    title={Minibatch vs Local {SGD} with Shuffling: Tight Convergence Bounds and Beyond},
    author={Chulhee Yun and Shashank Rajput and Suvrit Sra},
    booktitle={International Conference on Learning Representations},
    year=2022,
    url={https://openreview.net/forum?id=LdlwbBP2mlq}
}

@inproceedings{glasgow21_sharp_bound_feder_averag_local,
  author    = {Glasgow, Margalit R. and Yuan, Honglin and Ma, Tengyu},
  editor    = {Camps{-}Valls, Gustau and Ruiz, Francisco J. R. and Valera, Isabel},
  publisher = {PMLR},
  url       = {https://proceedings.mlr.press/v151/glasgow22a.html},
  booktitle = {International Conference on Artificial Intelligence and Statistics, {AISTATS} 2022, 28-30 March 2022, Virtual Event},
  year      = {2022},
  pages     = {9050--9090},
  series    = {Proceedings of Machine Learning Research},
  title     = {Sharp Bounds for Federated Averaging (Local {SGD)} and Continuous Perspective},
  volume    = {151},
}

@inproceedings{murata21_bias_varian_reduc_local_sgd,
  author    = {Murata, Tomoya and Suzuki, Taiji},
  editor    = {Meila, Marina and Zhang, Tong},
  publisher = {PMLR},
  url       = {http://proceedings.mlr.press/v139/murata21a.html},
  booktitle = {Proceedings of the 38th International Conference on Machine Learning, {ICML} 2021, 18-24 July 2021, Virtual Event},
  year      = {2021},
  pages     = {7872--7881},
  series    = {Proceedings of Machine Learning Research},
  title     = {Bias-Variance Reduced Local {SGD} for Less Heterogeneous Federated Learning},
  volume    = {139},
}

@inproceedings{yuan20_feder_accel_stoch_gradien_descen,
  author    = {Yuan, Honglin and Ma, Tengyu},
  editor    = {Larochelle, Hugo and Ranzato, Marc'Aurelio and Hadsell, Raia and Balcan, Maria{-}Florina and Lin, Hsuan{-}Tien},
  url       = {https://proceedings.neurips.cc/paper/2020/hash/39d0a8908fbe6c18039ea8227f827023-Abstract.html},
  booktitle = {Advances in Neural Information Processing Systems 33: Annual Conference on Neural Information Processing Systems 2020, NeurIPS 2020, December 6-12, 2020, virtual},
  year      = {2020},
  title     = {Federated Accelerated Stochastic Gradient Descent},
}

@inproceedings{hanzely20_lower_bound_optim_algor_person_feder_learn,
  author    = {Hanzely, Filip and Hanzely, Slavomír and Horváth, Samuel and Richtárik, Peter},
  editor    = {Larochelle, Hugo and Ranzato, Marc'Aurelio and Hadsell, Raia and Balcan, Maria{-}Florina and Lin, Hsuan{-}Tien},
  url       = {https://proceedings.neurips.cc/paper/2020/hash/187acf7982f3c169b3075132380986e4-Abstract.html},
  booktitle = {Advances in Neural Information Processing Systems 33: Annual Conference on Neural Information Processing Systems 2020, NeurIPS 2020, December 6-12, 2020, virtual},
  year      = {2020},
  title     = {Lower Bounds and Optimal Algorithms for Personalized Federated Learning},
}

@article{malinovsky22_server_side_steps_sampl_without,
  author       = {Malinovsky, Grigory and Mishchenko, Konstantin and Richt{á}rik, Peter},
  url          = {https://arXiv.org/abs/2201.11066},
  year         = {2022},
  journal      = {arXiv preprint},
  title        = {Server-Side Stepsizes and Sampling Without Replacement Provably Help in Federated Optimization},
  volume       = {arXiv:2201.11066},
}

@book{nesterov18_lectures_cvx_opt,
  author    = {Nesterov, Yurii},
  publisher = {Springer Publishing Company, Incorporated},
  year      = {2018},
  edition   = {2nd},
  isbn      = {3319915770},
  title     = {Lectures on Convex Optimization},
}

@inproceedings{li18_feder_optim_heter_networ,
  author    = {Li, Tian and Sahu, Anit Kumar and Zaheer, Manzil and Sanjabi, Maziar and Talwalkar, Ameet and Smith, Virginia},
  editor    = {Dhillon, Inderjit S. and Papailiopoulos, Dimitris S. and Sze, Vivienne},
  publisher = {mlsys.org},
  url       = {https://proceedings.mlsys.org/book/316.pdf},
  booktitle = {Proceedings of Machine Learning and Systems 2020, MLSys 2020, Austin, TX, USA, March 2-4, 2020},
  year      = {2020},
  title     = {Federated Optimization in Heterogeneous Networks},
}

@article{bauschke09_baill_haddad_theor_revis,
  author       = {Bauschke, Heinz H. and Combettes, Patrick L.},
  url          = {https://arXiv.org/abs/0906.0807},
  year         = {2009},
  journal      = {arXiv preprint},
  title        = {The Baillon-Haddad Theorem Revisited},
  volume       = {arXiv:0906.0807},
}

@inproceedings{koloskova20_unified_theor_decen_sgd_with,
  author    = {Koloskova, Anastasia and Loizou, Nicolas and Boreiri, Sadra and Jaggi, Martin and Stich, Sebastian U.},
  publisher = {PMLR},
  url       = {http://proceedings.mlr.press/v119/koloskova20a.html},
  booktitle = {Proceedings of the 37th International Conference on Machine Learning, {ICML} 2020, 13-18 July 2020, Virtual Event},
  year      = {2020},
  pages     = {5381--5393},
  series    = {Proceedings of Machine Learning Research},
  title     = {A Unified Theory of Decentralized {SGD} with Changing Topology and Local Updates},
  volume    = {119},
}

@inproceedings{wang22_commun_effic_adapt_feder_learn,
  author    = {Wang, Yujia and Lin, Lu and Chen, Jinghui},
  publisher = {PMLR},
  booktitle = {{ICML}},
  year      = {2022},
  pages     = {22802--22838},
  series    = {Proceedings of Machine Learning Research},
  title     = {Communication-Efficient Adaptive Federated Learning},
  volume    = {162},
}

@article{malinovsky22_feder_random_reshuf_with_compr_varian_reduc,
  author       = {Malinovsky, Grigory and Richtárik, Peter},
  url          = {https://arXiv.org/abs/2205.03914},
  year         = {2022},
  journal      = {arXiv preprint},
  title        = {Federated Random Reshuffling With Compression and Variance Reduction},
  volume       = {arXiv:2205.03914},
}

@inproceedings{haddadpour20_feder_learn_with_compr,
  author    = {Haddadpour, Farzin and Kamani, Mohammad Mahdi and Mokhtari, Aryan and Mahdavi, Mehrdad},
  editor    = {Banerjee, Arindam and Fukumizu, Kenji},
  publisher = {PMLR},
  url       = {http://proceedings.mlr.press/v130/haddadpour21a.html},
  booktitle = {The 24th International Conference on Artificial Intelligence and Statistics, {AISTATS} 2021, April 13-15, 2021, Virtual Event},
  year      = {2021},
  pages     = {2350--2358},
  series    = {Proceedings of Machine Learning Research},
  title     = {Federated Learning with Compression: Unified Analysis and Sharp Guarantees},
  volume    = {130},
}

@article{wang21_field_guide_to_feder_optim,
  author       = {Wang, Jianyu and Charles, Zachary and Xu, Zheng and Joshi, Gauri and McMahan, H. Brendan and Arcas, Blaise Aguera y and Al-Shedivat, Maruan and Andrew, Galen and Avestimehr, Salman and Daly, Katharine and Data, Deepesh and Diggavi, Suhas and Eichner, Hubert and Gadhikar, Advait and Garrett, Zachary and Girgis, Antonious M. and Hanzely, Filip and Hard, Andrew and He, Chaoyang and Horvath, Samuel and Huo, Zhouyuan and Ingerman, Alex and Jaggi, Martin and Javidi, Tara and Kairouz, Peter and Kale, Satyen and Karimireddy, Sai Praneeth and Konecny, Jakub and Koyejo, Sanmi and Li, Tian and Liu, Luyang and Mohri, Mehryar and Qi, Hang and Reddi, Sashank J. and Richtarik, Peter and Singhal, Karan and Smith, Virginia and Soltanolkotabi, Mahdi and Song, Weikang and Suresh, Ananda Theertha and Stich, Sebastian U. and Talwalkar, Ameet and Wang, Hongyi and Woodworth, Blake and Wu, Shanshan and Yu, Felix X. and Yuan, Honglin and Zaheer, Manzil and Zhang, Mi and Zhang, Tong and Zheng, Chunxiang and Zhu, Chen and Zhu, Wennan},
  url          = {https://arXiv.org/abs/2107.06917},
  year         = {2021},
  journal      = {arXiv preprint},
  title        = {A Field Guide To Federated Optimization},
  volume       = {arXiv:2107.06917},
}

@inproceedings{safaryan21_smoot_matric_beat_smoot_const,
  author    = {Safaryan, Mher and Hanzely, Filip and Richtárik, Peter},
  booktitle = {NeurIPS},
  year      = {2021},
  pages     = {25688--25702},
  title     = {Smoothness Matrices Beat Smoothness Constants: Better Communication Compression Techniques for Distributed Optimization},
}

@inproceedings{stich18_local_sgd_conver_fast_commun_littl,
  author    = {Stich, Sebastian U.},
  publisher = {OpenReview.net},
  url       = {https://openreview.net/forum?id=S1g2JnRcFX},
  booktitle = {7th International Conference on Learning Representations, {ICLR} 2019, New Orleans, LA, USA, May 6-9, 2019},
  year      = {2019},
  title     = {Local {SGD} Converges Fast and Communicates Little},
}

@article{ivgi23_dog_is_sgds_best_frien,
  author       = {Ivgi, Maor and Hinder, Oliver and Carmon, Yair},
  url          = {https://arXiv.org/abs/2302.12022},
  year         = {2023},
  journal      = {arXiv preprint},
  title        = {{DoG} Is {SGD}'s Best Friend: a Parameter-Free Dynamic Step Size Schedule},
  volume       = {arXiv:2302.12022},
}

@article{gu23_why_when_does_local_sgd,
  author       = {Gu, Xinran and Lyu, Kaifeng and Huang, Longbo and Arora, Sanjeev},
  url          = {https://arXiv.org/abs/2303.01215},
  year         = {2023},
  journal      = {arXiv preprint},
  title        = {Why (and When) Does Local {SGD} Generalize Better Than {SGD}?},
  volume       = {arXiv:2303.01215},
}

@inproceedings{patel2023on,
  author    = {Patel, Kumar Kshitij and Glasgow, Margalit and Wang, Lingxiao and Joshi, Nirmit and Srebro, Nathan},
  url       = {https://openreview.net/forum?id=vhS68bKv7x},
  booktitle = {Federated Learning and Analytics in Practice: Algorithms, Systems, Applications, and Opportunities},
  year      = {2023},
  title     = {On the Still Unreasonable Effectiveness of Federated Averaging for Heterogeneous Distributed Learning},
}

@inproceedings{loshchilov17_decoup_weigh_decay_regul,
  author    = {Loshchilov, Ilya and Hutter, Frank},
  publisher = {OpenReview.net},
  url       = {https://openreview.net/forum?id=Bkg6RiCqY7},
  booktitle = {7th International Conference on Learning Representations, {ICLR} 2019, New Orleans, LA, USA, May 6-9, 2019},
  year      = {2019},
  title     = {Decoupled Weight Decay Regularization},
}

@inproceedings{brown20_languag_model_are_few_shot_learn,
  author    = {Brown, Tom B. and Mann, Benjamin and Ryder, Nick and Subbiah, Melanie and Kaplan, Jared and Dhariwal, Prafulla and Neelakantan, Arvind and Shyam, Pranav and Sastry, Girish and Askell, Amanda and Agarwal, Sandhini and Herbert{-}Voss, Ariel and Krueger, Gretchen and Henighan, Tom and Child, Rewon and Ramesh, Aditya and Ziegler, Daniel M. and Wu, Jeffrey and Winter, Clemens and Hesse, Christopher and Chen, Mark and Sigler, Eric and Litwin, Mateusz and Gray, Scott and Chess, Benjamin and Clark, Jack and Berner, Christopher and McCandlish, Sam and Radford, Alec and Sutskever, Ilya and Amodei, Dario},
  editor    = {Larochelle, Hugo and Ranzato, Marc'Aurelio and Hadsell, Raia and Balcan, Maria{-}Florina and Lin, Hsuan{-}Tien},
  url       = {https://proceedings.neurips.cc/paper/2020/hash/1457c0d6bfcb4967418bfb8ac142f64a-Abstract.html},
  booktitle = {Advances in Neural Information Processing Systems 33: Annual Conference on Neural Information Processing Systems 2020, NeurIPS 2020, December 6-12, 2020, virtual},
  year      = {2020},
  title     = {Language Models are Few-Shot Learners},
}

@article{douillard23_diloc,
  author       = {Douillard, Arthur and Feng, Qixuan and Rusu, Andrei A. and Chhaparia, Rachita and Donchev, Yani and Kuncoro, Adhiguna and Ranzato, Marc'Aurelio and Szlam, Arthur and Shen, Jiajun},
  url          = {https://arXiv.org/abs/2311.08105},
  year         = {2023},
  journal      = {arXiv preprint},
  title        = {{DiLoCo}: Distributed Low-Communication Training of Language Models},
  volume       = {arXiv:2311.08105},
}

@article{defazio24_road_less_sched,
  author       = {Defazio, Aaron and Yang, Xingyu Alice and Mehta, Harsh and Mishchenko, Konstantin and Khaled, Ahmed and Cutkosky, Ashok},
  url          = {https://arXiv.org/abs/2405.15682},
  year         = {2024},
  journal      = {arXiv preprint},
  title        = {The Road Less Scheduled},
  volume       = {arXiv:2405.15682},
}

@article{jaghouar24_opend,
  author       = {Jaghouar, Sami and Ong, Jack Min and Hagemann, Johannes},
  url          = {https://arXiv.org/abs/2407.07852},
  year         = {2024},
  journal      = {arXiv preprint},
  title        = {{OpenDiLoCo}: an Open-Source Framework for Globally Distributed Low-Communication Training},
  volume       = {arXiv:2407.07852},
}

@article{liu24_async_local_sgd_train_languag_model,
  author       = {Liu, Bo and Chhaparia, Rachita and Douillard, Arthur and Kale, Satyen and Rusu, Andrei A. and Shen, Jiajun and Szlam, Arthur and Ranzato, Marc'Aurelio},
  url          = {https://arXiv.org/abs/2401.09135},
  year         = {2024},
  journal      = {arXiv preprint},
  title        = {Asynchronous Local-{SGD} Training for Language Modeling},
  volume       = {arXiv:2401.09135},
}

@inproceedings{reddi20_adapt_feder_optim,
  author    = {Reddi, Sashank J. and Charles, Zachary and Zaheer, Manzil and Garrett, Zachary and Rush, Keith and Konečný, Jakub and Kumar, Sanjiv and McMahan, Hugh Brendan},
  publisher = {OpenReview.net},
  url       = {https://openreview.net/forum?id=LkFG3lB13U5},
  booktitle = {9th International Conference on Learning Representations, {ICLR} 2021, Virtual Event, Austria, May 3-7, 2021},
  year      = {2021},
  title     = {Adaptive Federated Optimization},
}

@article{patel24_limit_poten_local_sgd_distr,
  author       = {Patel, Kumar Kshitij and Glasgow, Margalit and Zindari, Ali and Wang, Lingxiao and Stich, Sebastian U. and Cheng, Ziheng and Joshi, Nirmit and Srebro, Nathan},
  url          = {https://arXiv.org/abs/2405.11667},
  year         = {2024},
  journal      = {arXiv preprint},
  title        = {The Limits and Potentials of Local {SGD} for Distributed Heterogeneous Learning With Intermittent Communication},
  volume       = {arXiv:2405.11667},
}

@inproceedings{lin18_dont_use_large_mini_batch,
  author    = {Lin, Tao and Stich, Sebastian U. and Patel, Kumar Kshitij and Jaggi, Martin},
  publisher = {OpenReview.net},
  url       = {https://openreview.net/forum?id=B1eyO1BFPr},
  booktitle = {8th International Conference on Learning Representations, {ICLR} 2020, Addis Ababa, Ethiopia, April 26-30, 2020},
  year      = {2020},
  title     = {Don't Use Large Mini-batches, Use Local {SGD}},
}

@article{yang24_sa_fedlor,
      title={SPD-CFL: Stepwise Parameter Dropout for Efficient Continual Federated Learning}, 
      author={Yuning Yang and Han Yu and Chuan Sun and Tianrun Gao and Xiaohong Liu and Xiaodong Xu and Ping Zhang and Guangyu Wang},
      year={2025},
      journal={arXiv preprint arxiv:2405.09394},
      url={https://arxiv.org/abs/2405.09394}, 
}

@article{douillard24_dipac,
  author       = {Douillard, Arthur and Feng, Qixuan and Rusu, Andrei A. and Kuncoro, Adhiguna and Donchev, Yani and Chhaparia, Rachita and Gog, Ionel and Ranzato, Marc'Aurelio and Shen, Jiajun and Szlam, Arthur},
  url          = {https://arXiv.org/abs/2403.10616},
  year         = {2024},
  journal      = {arXiv preprint},
  title        = {{DiPaCo}: Distributed Path Composition},
  volume       = {arXiv:2403.10616},
}

@article{li24_conver_fedpr_with_extrap_inexac_prox,
  author          = {Li, Hanmin and Richt{\'a}rik, Peter},
  title           = {On the Convergence of {FedProx} With Extrapolation and
                  Inexact Prox},
  journal         = {CoRR},
  year            = 2024,
  url             = {http://arxiv.org/abs/2410.01410v1},
  archivePrefix   = {arXiv},
  eprint          = {2410.01410},
  primaryClass    = {math.OC}
  }

@article{li24_power_extrap_feder_learn,
  title={The power of extrapolation in federated learning},
  author={Li, Hanmin and Acharya, Kirill and Richt{\'a}rik, Peter},
  journal={Advances in Neural Information Processing Systems},
  volume={37},
  pages={124236--124291},
  year={2024}
}

@inproceedings{sun23_role_server_momen_feder_learn,
author = {Sun, Jianhui and Wu, Xidong and Huang, Heng and Zhang, Aidong},
title = {On the role of server momentum in federated learning},
year = {2024},
isbn = {978-1-57735-887-9},
publisher = {AAAI Press},
url = {https://doi.org/10.1609/aaai.v38i13.29439},
doi = {10.1609/aaai.v38i13.29439},
booktitle = {Proceedings of the Thirty-Eighth AAAI Conference on Artificial Intelligence and Thirty-Sixth Conference on Innovative Applications of Artificial Intelligence and Fourteenth Symposium on Educational Advances in Artificial Intelligence},
articleno = {1691},
numpages = {9},
series = {AAAI'24/IAAI'24/EAAI'24}
}

@article{jhunjhunwala23_fedex,
    title={FedExP: Speeding Up Federated Averaging via Extrapolation},
    author={Divyansh Jhunjhunwala and Shiqiang Wang and Gauri Joshi},
    booktitle={The Eleventh International Conference on Learning Representations },
    year={2023},
    url={https://openreview.net/forum?id=IPrzNbddXV}
}

@inproceedings{zhou21_lookahead,
 author = {Zhou, Pan and Yan, Hanshu and Yuan, Xiaotong and Feng, Jiashi and Yan, Shuicheng},
 booktitle = {Advances in Neural Information Processing Systems},
 editor = {M. Ranzato and A. Beygelzimer and Y. Dauphin and P.S. Liang and J. Wortman Vaughan},
 pages = {27290--27304},
 publisher = {Curran Associates, Inc.},
 title = {Towards Understanding Why {Lookahead} Generalizes Better Than {SGD} and Beyond},
 url = {https://proceedings.neurips.cc/paper_files/paper/2021/file/e53a0a2978c28872a4505bdb51db06dc-Paper.pdf},
 volume = {34},
 year = {2021}
}

@article{jaghouar2024intellect,
  title={INTELLECT-1 Technical Report},
  author={Jaghouar, Sami and Ong, Jack Min and Basra, Manveer and Obeid, Fares and Straube, Jannik and Keiblinger, Michael and Bakouch, Elie and Atkins, Lucas and Panahi, Maziyar and Goddard, Charles and others},
  journal={arXiv preprint arXiv:2412.01152},
  year={2024}
}

@article{hoffmann2022chinchilla,
      title={Training Compute-Optimal Large Language Models}, 
      author={Jordan Hoffmann and Sebastian Borgeaud and Arthur Mensch and Elena Buchatskaya and Trevor Cai and Eliza Rutherford and Diego de Las Casas and Lisa Anne Hendricks and Johannes Welbl and Aidan Clark and Tom Hennigan and Eric Noland and Katie Millican and George van den Driessche and Bogdan Damoc and Aurelia Guy and Simon Osindero and Karen Simonyan and Erich Elsen and Jack W. Rae and Oriol Vinyals and Laurent Sifre},
      year={2022},
      journal={Advances in Neural Information Processing Systems (NeurIPS)}}

@article{c4,
author = {Raffel, Colin and Shazeer, Noam and Roberts, Adam and Lee, Katherine and Narang, Sharan and Matena, Michael and Zhou, Yanqi and Li, Wei and Liu, Peter J.},
title = {Exploring the limits of transfer learning with a unified text-to-text transformer},
year = {2020},
issue_date = {January 2020},
publisher = {JMLR.org},
volume = {21},
number = {1},
issn = {1532-4435},
journal = {J. Mach. Learn. Res.},
month = jan,
articleno = {140},
numpages = {67}
}

@inproceedings{mcmahan2021fedavg,
  author    = {H. Brendan McMahan and Eider Moore and Daniel Ramage and Seth Hampson and Blaise Agüera y Arcas},
  booktitle = {Proceedings of the 20 th International Conference on Artificial Intelligence and Statistics (AISTATS)},
  year      = {2017},
  title     = {Communication-Efficient Learning of Deep Networks from Decentralized Data}
 }

@inproceedings{gu2024schedulinginnersteps,
  author    = {Xinran Gu and Kaifeng Lyu and Sanjeev Arora and Jingzhao Zhang and Longbo Huang},
  booktitle = {International Conference on Learning Representations},
  year      = {2024},
  title     = {A Quadratic Synchronization Rule for Distributed Deep Learning}
 }

@article{wei2020federated,
  title={Federated learning with differential privacy: Algorithms and performance analysis},
  author={Wei, Kang and Li, Jun and Ding, Ming and Ma, Chuan and Yang, Howard H and Farokhi, Farhad and Jin, Shi and Quek, Tony QS and Poor, H Vincent},
  journal={IEEE Transactions on Information Forensics and Security},
  volume={15},
  pages={3454--3469},
  year={2020},
  publisher={IEEE}
}

@inproceedings{eichner2019semi,
  title={Semi-cyclic stochastic gradient descent},
  author={Eichner, Hubert and Koren, Tomer and McMahan, Brendan and Srebro, Nathan and Talwar, Kunal},
  booktitle={International Conference on Machine Learning},
  pages={1764--1773},
  year={2019},
  organization={PMLR}
}

@inproceedings{jiang2024megascale,
  title={MegaScale: Scaling large language model training to more than 10,000 GPUs},
  author={Ziheng Jiang and Haibin Lin and Yinmin Zhong and Qi Huang and Yangrui Chen and Zhi Zhang and Yanghua Peng and Xiang Li and Cong Xie and Shibiao Nong and Yulu Jia and Sun He and Hongmin Chen and Zhihao Bai and Qi Hou and Shipeng Yan and Ding Zhou and Yiyao Sheng and Zhuo Jiang and Haohan Xu and Haoran Wei and Zhang Zhang and Pengfei Nie and Leqi Zou and Sida Zhao and Liang Xiang and Zherui Liu and Zhe Li and Xiaoying Jia and Jianxi Ye and Xin Jin and Xin Liu},
  booktitle={21st USENIX Symposium on Networked Systems Design and Implementation (NSDI 24)},
  pages={745--760},
  year={2024}
}

@article{sani2024future,
  title={The Future of Large Language Model Pre-training is Federated},
  author={Lorenzo Sani and Alex Iacob and Zeyu Cao and Bill Marino and Yan Gao and Tomas Paulik and Wanru Zhao and William F. Shen and Preslav Aleksandrov and Xinchi Qiu and Nicholas D. Lane},
  journal={arXiv preprint arXiv:2405.10853},
  year={2024}
}

@article{iacob2024worldwide,
  title={Worldwide Federated Training of Language Models},
  author={Iacob, Alex and Sani, Lorenzo and Marino, Bill and Aleksandrov, Preslav and Lane, Nicholas Donald},
  journal={arXiv preprint arXiv:2405.14446},
  year={2024}
}

@article{liang24_commun_effic_large_scale_distr,
  author       = {Liang, Feng and Zhang, Zhen and Lu, Haifeng and Leung, Victor C. M. and Guo, Yanyi and Hu, Xiping},
  url          = {https://arXiv.org/abs/2404.06114},
  year         = {2024},
  journal      = {arXiv preprint},
  title        = {Communication-Efficient Large-Scale Distributed Deep Learning: a Comprehensive Survey},
  volume       = {arXiv:2404.06114},
}

@article{diskin2021distributedcollab,
    title={Distributed Deep Learning in Open Collaborations},
    author={Diskin, Michael and Bukhtiyarov, Alexey and Ryabinin, Max and Saulnier, Lucile and Lhoest, Quentin and Sinitsin, Anton and Popov, Dmitry and Pyrkin, Dmitry and Kashirin, Maxim and Borzunov, Alexander and Villanova del Moral, Albert and Mazur, Denis and Kobelev, Ilia and Jernite, Yacine and Wolf, Thomas and Pekhimenko, Gennady},
    journal={Advances in Neural Information Processing Systems (NeurIPS)},
    year={2021}
}

@article{borzunov2023petals,
  title = {Petals: Collaborative Inference and Fine-tuning of Large Models},
  author = {Borzunov, Alexander and Baranchuk, Dmitry and Dettmers, Tim and Ryabinin, Max and Belkada, Younes and Chumachenko, Artem and Samygin, Pavel and Raffel, Colin},
  journal={arXiv preprint arXiv:2209.01188},
  year = {2022},
}

@misc{huang2022crosssilofederatedlearningchallenges,
      title={Cross-Silo Federated Learning: Challenges and Opportunities}, 
      author={Chao Huang and Jianwei Huang and Xin Liu},
      year={2022},
      eprint={2206.12949},
      archivePrefix={arXiv},
      primaryClass={cs.LG},
      url={https://arxiv.org/abs/2206.12949}, 
}

@article{rush2024drjax,
    title={DrJAX: Scalable and Differentiable MapReduce Primitives in JAX},
    author={Keith Rush and Zachary Charles and Zachary Garrett and Sean Augenstein and Nicole Elyse Mitchell},
    journal={International Conference on Machine Learning (ICML) Workshop},
    year={2024}
}

@article{liu2024deepseek,
  title={{DeepSeek-V3} technical report},
  author={DeepSeek-AI and Aixin Liu and Bei Feng and Bing Xue and Bingxuan Wang and Bochao Wu and Chengda Lu and Chenggang Zhao and Chengqi Deng and Chenyu Zhang and Chong Ruan and Damai Dai and Daya Guo and Dejian Yang and Deli Chen and Dongjie Ji and Erhang Li and Fangyun Lin and Fucong Dai and Fuli Luo and Guangbo Hao and Guanting Chen and Guowei Li and H. Zhang and Han Bao and Hanwei Xu and Haocheng Wang and Haowei Zhang and Honghui Ding and Huajian Xin and Huazuo Gao and Hui Li and Hui Qu and J. L. Cai and Jian Liang and Jianzhong Guo and Jiaqi Ni and Jiashi Li and Jiawei Wang and Jin Chen and Jingchang Chen and Jingyang Yuan and Junjie Qiu and Junlong Li and Junxiao Song and Kai Dong and Kai Hu and Kaige Gao and Kang Guan and Kexin Huang and Kuai Yu and Lean Wang and Lecong Zhang and Lei Xu and Leyi Xia and Liang Zhao and Litong Wang and Liyue Zhang and Meng Li and Miaojun Wang and Mingchuan Zhang and Minghua Zhang and Minghui Tang and Mingming Li and Ning Tian and Panpan Huang and Peiyi Wang and Peng Zhang and Qiancheng Wang and Qihao Zhu and Qinyu Chen and Qiushi Du and R. J. Chen and R. L. Jin and Ruiqi Ge and Ruisong Zhang and Ruizhe Pan and Runji Wang and Runxin Xu and Ruoyu Zhang and Ruyi Chen and S. S. Li and Shanghao Lu and Shangyan Zhou and Shanhuang Chen and Shaoqing Wu and Shengfeng Ye and Shengfeng Ye and Shirong Ma and Shiyu Wang and Shuang Zhou and Shuiping Yu and Shunfeng Zhou and Shuting Pan and T. Wang and Tao Yun and Tian Pei and Tianyu Sun and W. L. Xiao and Wangding Zeng and Wanjia Zhao and Wei An and Wen Liu and Wenfeng Liang and Wenjun Gao and Wenqin Yu and Wentao Zhang and X. Q. Li and Xiangyue Jin and Xianzu Wang and Xiao Bi and Xiaodong Liu and Xiaohan Wang and Xiaojin Shen and Xiaokang Chen and Xiaokang Zhang and Xiaosha Chen and Xiaotao Nie and Xiaowen Sun and Xiaoxiang Wang and Xin Cheng and Xin Liu and Xin Xie and Xingchao Liu and Xingkai Yu and Xinnan Song and Xinxia Shan and Xinyi Zhou and Xinyu Yang and Xinyuan Li and Xuecheng Su and Xuheng Lin and Y. K. Li and Y. Q. Wang and Y. X. Wei and Y. X. Zhu and Yang Zhang and Yanhong Xu and Yanhong Xu and Yanping Huang and Yao Li and Yao Zhao and Yaofeng Sun and Yaohui Li and Yaohui Wang and Yi Yu and Yi Zheng and Yichao Zhang and Yifan Shi and Yiliang Xiong and Ying He and Ying Tang and Yishi Piao and Yisong Wang and Yixuan Tan and Yiyang Ma and Yiyuan Liu and Yongqiang Guo and Yu Wu and Yuan Ou and Yuchen Zhu and Yuduan Wang and Yue Gong and Yuheng Zou and Yujia He and Yukun Zha and Yunfan Xiong and Yunxian Ma and Yuting Yan and Yuxiang Luo and Yuxiang You and Yuxuan Liu and Yuyang Zhou and Z. F. Wu and Z. Z. Ren and Zehui Ren and Zhangli Sha and Zhe Fu and Zhean Xu and Zhen Huang and Zhen Zhang and Zhenda Xie and Zhengyan Zhang and Zhewen Hao and Zhibin Gou and Zhicheng Ma and Zhigang Yan and Zhihong Shao and Zhipeng Xu and Zhiyu Wu and Zhongyu Zhang and Zhuoshu Li and Zihui Gu and Zijia Zhu and Zijun Liu and Zilin Li and Ziwei Xie and Ziyang Song and Ziyi Gao and Zizheng Pan},
  journal={arXiv preprint arXiv:2412.19437},
  year={2024}
}

@artcle{guo2025deepseek,
      title={DeepSeek-R1: Incentivizing Reasoning Capability in LLMs via Reinforcement Learning}, 
      author={DeepSeek-AI and Daya Guo and Dejian Yang and Haowei Zhang and Junxiao Song and Ruoyu Zhang and Runxin Xu and Qihao Zhu and Shirong Ma and Peiyi Wang and Xiao Bi and Xiaokang Zhang and Xingkai Yu and Yu Wu and Z. F. Wu and Zhibin Gou and Zhihong Shao and Zhuoshu Li and Ziyi Gao and Aixin Liu and Bing Xue and Bingxuan Wang and Bochao Wu and Bei Feng and Chengda Lu and Chenggang Zhao and Chengqi Deng and Chenyu Zhang and Chong Ruan and Damai Dai and Deli Chen and Dongjie Ji and Erhang Li and Fangyun Lin and Fucong Dai and Fuli Luo and Guangbo Hao and Guanting Chen and Guowei Li and H. Zhang and Han Bao and Hanwei Xu and Haocheng Wang and Honghui Ding and Huajian Xin and Huazuo Gao and Hui Qu and Hui Li and Jianzhong Guo and Jiashi Li and Jiawei Wang and Jingchang Chen and Jingyang Yuan and Junjie Qiu and Junlong Li and J. L. Cai and Jiaqi Ni and Jian Liang and Jin Chen and Kai Dong and Kai Hu and Kaige Gao and Kang Guan and Kexin Huang and Kuai Yu and Lean Wang and Lecong Zhang and Liang Zhao and Litong Wang and Liyue Zhang and Lei Xu and Leyi Xia and Mingchuan Zhang and Minghua Zhang and Minghui Tang and Meng Li and Miaojun Wang and Mingming Li and Ning Tian and Panpan Huang and Peng Zhang and Qiancheng Wang and Qinyu Chen and Qiushi Du and Ruiqi Ge and Ruisong Zhang and Ruizhe Pan and Runji Wang and R. J. Chen and R. L. Jin and Ruyi Chen and Shanghao Lu and Shangyan Zhou and Shanhuang Chen and Shengfeng Ye and Shiyu Wang and Shuiping Yu and Shunfeng Zhou and Shuting Pan and S. S. Li and Shuang Zhou and Shaoqing Wu and Shengfeng Ye and Tao Yun and Tian Pei and Tianyu Sun and T. Wang and Wangding Zeng and Wanjia Zhao and Wen Liu and Wenfeng Liang and Wenjun Gao and Wenqin Yu and Wentao Zhang and W. L. Xiao and Wei An and Xiaodong Liu and Xiaohan Wang and Xiaokang Chen and Xiaotao Nie and Xin Cheng and Xin Liu and Xin Xie and Xingchao Liu and Xinyu Yang and Xinyuan Li and Xuecheng Su and Xuheng Lin and X. Q. Li and Xiangyue Jin and Xiaojin Shen and Xiaosha Chen and Xiaowen Sun and Xiaoxiang Wang and Xinnan Song and Xinyi Zhou and Xianzu Wang and Xinxia Shan and Y. K. Li and Y. Q. Wang and Y. X. Wei and Yang Zhang and Yanhong Xu and Yao Li and Yao Zhao and Yaofeng Sun and Yaohui Wang and Yi Yu and Yichao Zhang and Yifan Shi and Yiliang Xiong and Ying He and Yishi Piao and Yisong Wang and Yixuan Tan and Yiyang Ma and Yiyuan Liu and Yongqiang Guo and Yuan Ou and Yuduan Wang and Yue Gong and Yuheng Zou and Yujia He and Yunfan Xiong and Yuxiang Luo and Yuxiang You and Yuxuan Liu and Yuyang Zhou and Y. X. Zhu and Yanhong Xu and Yanping Huang and Yaohui Li and Yi Zheng and Yuchen Zhu and Yunxian Ma and Ying Tang and Yukun Zha and Yuting Yan and Z. Z. Ren and Zehui Ren and Zhangli Sha and Zhe Fu and Zhean Xu and Zhenda Xie and Zhengyan Zhang and Zhewen Hao and Zhicheng Ma and Zhigang Yan and Zhiyu Wu and Zihui Gu and Zijia Zhu and Zijun Liu and Zilin Li and Ziwei Xie and Ziyang Song and Zizheng Pan and Zhen Huang and Zhipeng Xu and Zhongyu Zhang and Zhen Zhang},
      year={2025},
      journal={arXiv preprint arXiv:2501.12948},
}
\clearpage
\appendix

\part*{Supplementary material}

\section{Supplementary experimental details} \label{sec:experiments-details}
In this section we provide the details on the language model pretraining experiments discussed in the main text.

\subsection{Language model pretraining}

We study the impact of using various outer optimizers on large language model pretraining. We utilized Chinchilla-style decoder transformer architectures \citep{hoffmann2022chinchilla} trained on the C4 dataset \citep{c4}, consistent with common practices in large-scale model training \citep{douillard23_diloc}. The following subsections detail the specific hyperparameters, variations in training configurations (such as the number of inner steps and replicas/clients), and analyses of optimizer behavior, including learning rate scheduling and observed gradient cosine similarities.

\label{sec:language-models-experiments}
\subsubsection{Hyperparameters details}
\label{sec:language-models-hyperparameter}

We show in \autoref{tab:hyperparameters} the hyperparameters considered and kept, and in \autoref{tab:model_config} the architectural hyperparameters. We use the SentencePiece tokenizer with a sequence length of $1024$ for all models. We tuned all our optimizers on a separate validation set. We also considered using the Schedule-Free Optimizer with Nesterov acceleration on top but it was hard to tune and unstable. We include the validation results for all the main experiments we ran in \Cref{tab:add_full_sweep,tab:add_lr_sweep,tab:add_beta_sweep}.

\begin{table}[h]
\centering
\caption{\textbf{Model Configuration} for the three evaluated sizes. All are based on the transformer architecture, chinchilla-style \citep{hoffmann2022chinchilla}.}
\begin{tabular}{@{}l|ccc@{}}
\toprule
Hyperparameter &  150M & 400M & 1B\\
\midrule
Number of layers & 12 & 12 & 24\\
Hidden dim & 896 & 1536 & 2048\\
Number of heads & 16 & 12 & 16\\
K/V size & 64 & 128 & 128\\
Vocab size & \multicolumn{3}{c}{$32{,}000$}\\
\bottomrule
\end{tabular}
\label{tab:model_config}
\end{table}
\subsubsection{Varying inner steps}\label{sec:varying-inner-steps}

In \autoref{fig:pretraining_inner_steps}, we compare the stability of different outer optimizers when varying the synchronization frequency. We experiments a different amount of inner steps, from 50, to 2000. All experiments are run in pretraining from scratch, with 150 millions (150M) parameters. We note that as the synchronization frequency decreases (number of inner/local steps increases), performance decreases. Notably, averaging (in \textcolor{color_sgd}{orange}), is relatively constant w.r.t the synchronization frequency: its performance stay stable from $H=250$ to $H=2000$. On the other hand, using Nesterov with high outer learning rate (in \textcolor{color_nesterov2}{light green}) is particularly unstable, its performance decreases by $10.7\%$, this indicates that the learning rate should be tuned alongside the synchronization frequency. On the hand, SF-SGD (in \textcolor{color_sf}{blue}) has minimal degradation of performance ($4.2\%$), highlighting the \textit{schedule-free} property when varying hyperparameters.

\begin{figure}[t]
\centering
    \vspace{-1mm}
    \includegraphics[width=0.93\linewidth,trim={0cm 0cm 0cm 0cm},clip]{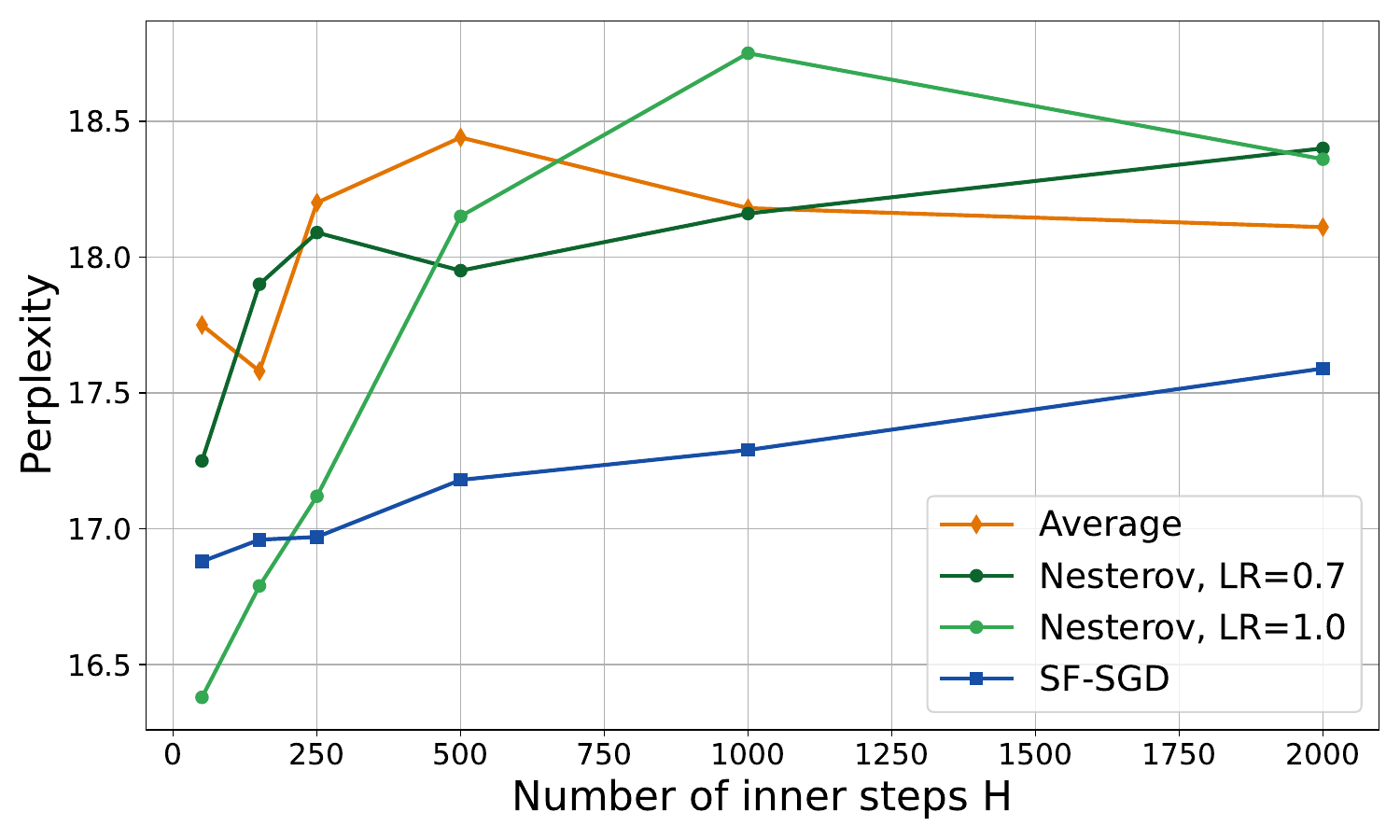}
    \vspace{-5mm}
    \caption{\textbf{Varying the communication frequency}, i.e.  number of inner steps $H$, when pretraining from scratch at 150M parameters.}
\label{fig:pretraining_inner_steps}
\vspace{-4.5mm}
\end{figure}

\begin{figure}[t]
\vspace{-1mm}
\centering
    \includegraphics[width=0.93\linewidth,trim={0cm 0cm 0cm 0cm},clip]{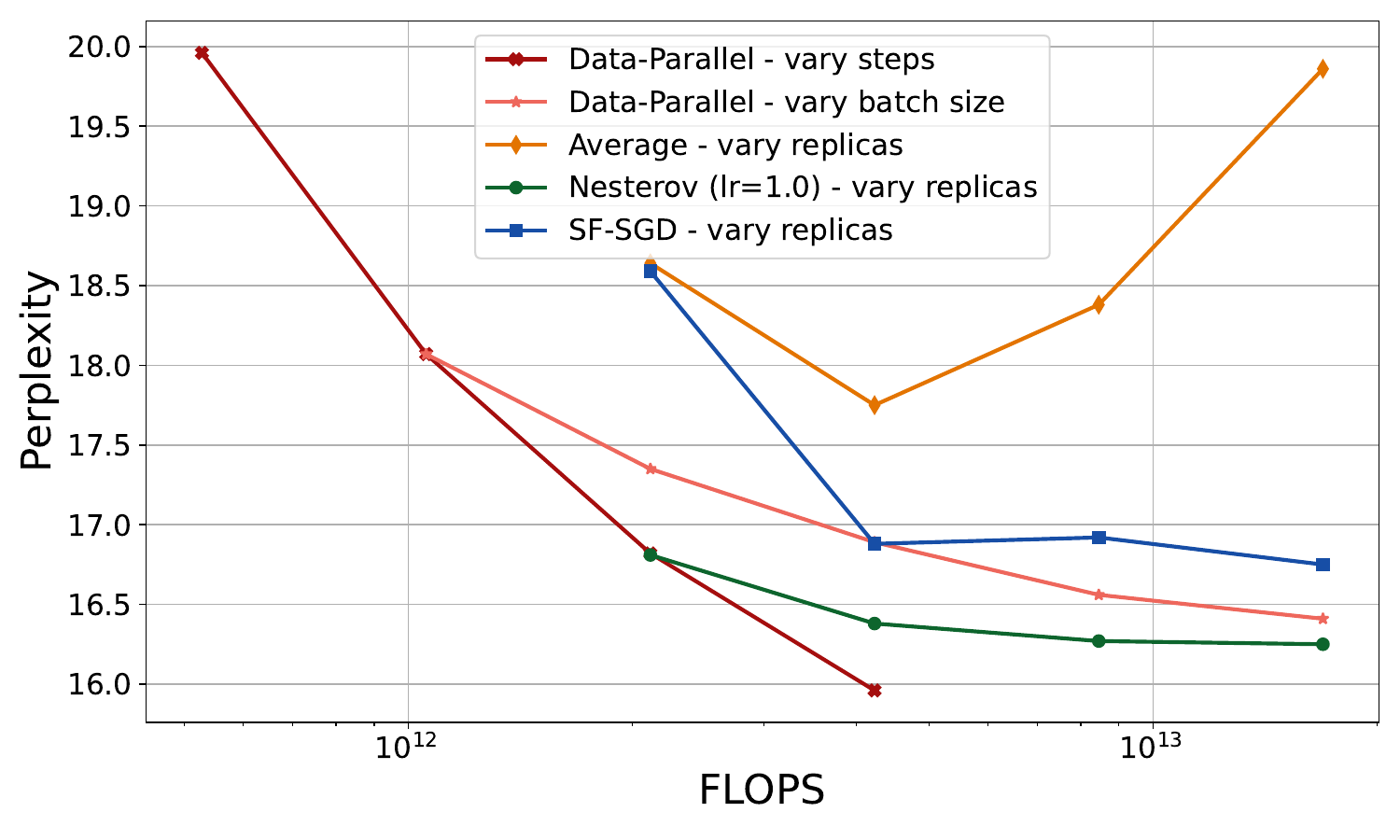}
    \vspace{-5mm}
    \caption{\textbf{Pareto front} of the flops vs perplexity, comparing various approach scaling the flops budget: increasing the number of steps, increasing the batch size in data-parallel, and increasing the number of replicas for federated learning.}
\label{fig:pretraining_flops_pareto_front}
\vspace{-4.5mm}
\end{figure}

\subsubsection{Varying replicas / flops budget}\label{sec:pretraining_replicas}

When increasing the number of distributed replicas, two options are possible: (a) Keeping the local per-replica batch size constant and thus increasing global batch size and flops budget, and (b) Keeping the global batch size/flops budget constant and thus reducing the local per-replica batch size.

We present in \autoref{fig:pretraining_flops_pareto_front} results of the first option with x-axis the flops budget for a single model size (150M). It is worth noting that increasing the number of replicas improves the performance of Nesterov (in \textcolor{color_nesterov}{green}) and SF-SGD (in \textcolor{color_sf}{blue}) but the gain quickly plateau. On the other hand, increasing the batch size for data-parallel (at the cost of more communication, because more DP replicas) or the number of steps (at the cost of longer training) still rapidly improves perplexity. Therefore, we wish to highlight here a disadvantage of federated learning methods seldom mentioned: while those methods are extremely communication-efficient, and can be made flops-efficient, their flops-efficiency disappear as the number of replicas increases.

To this problem, several hypotheses could be raised, such as the decreasing cosine similarity between outer gradients as the number of replicas increase, even when using an \textit{i.i.d.} data split across replicas. In \autoref{fig:pretraining_replica_cos_all}, we report the average similarity across a whole training for different number of replicas. For momentum-based methods (Nesterov, SF-SGD), the similarity decreases from 30\% at $M=2$ replicas to 10\% at $M=16$ replicas. Full details across training steps can be found in the appendix.

Finally, note that we didn't investigate further the second option of keeping the global batch size/flops budget constant and thus reducing the local per-replica batch size. We found that dividing the batch size by the number of replicas leads quickly to a local per-replica batch size that is critically low, and further reduces the flops-efficiency. More investigations should be pushed in that direction.

\begin{figure}[!b]
\centering
    \vspace{-3mm}
    \hspace*{-1mm}
    \includegraphics[width=0.49\linewidth,trim={0cm 0cm 0cm 0cm},clip]{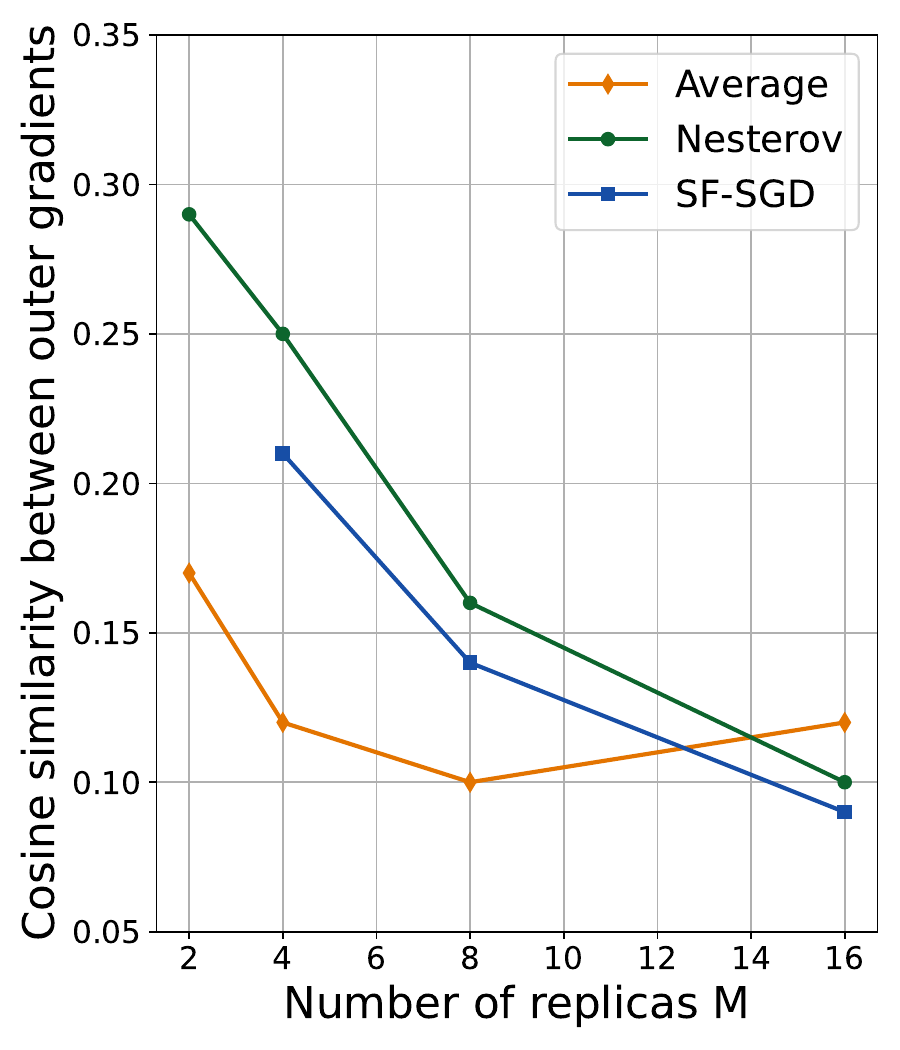}
    \includegraphics[width=0.49\linewidth,trim={0cm 0cm 0cm 0cm},clip]{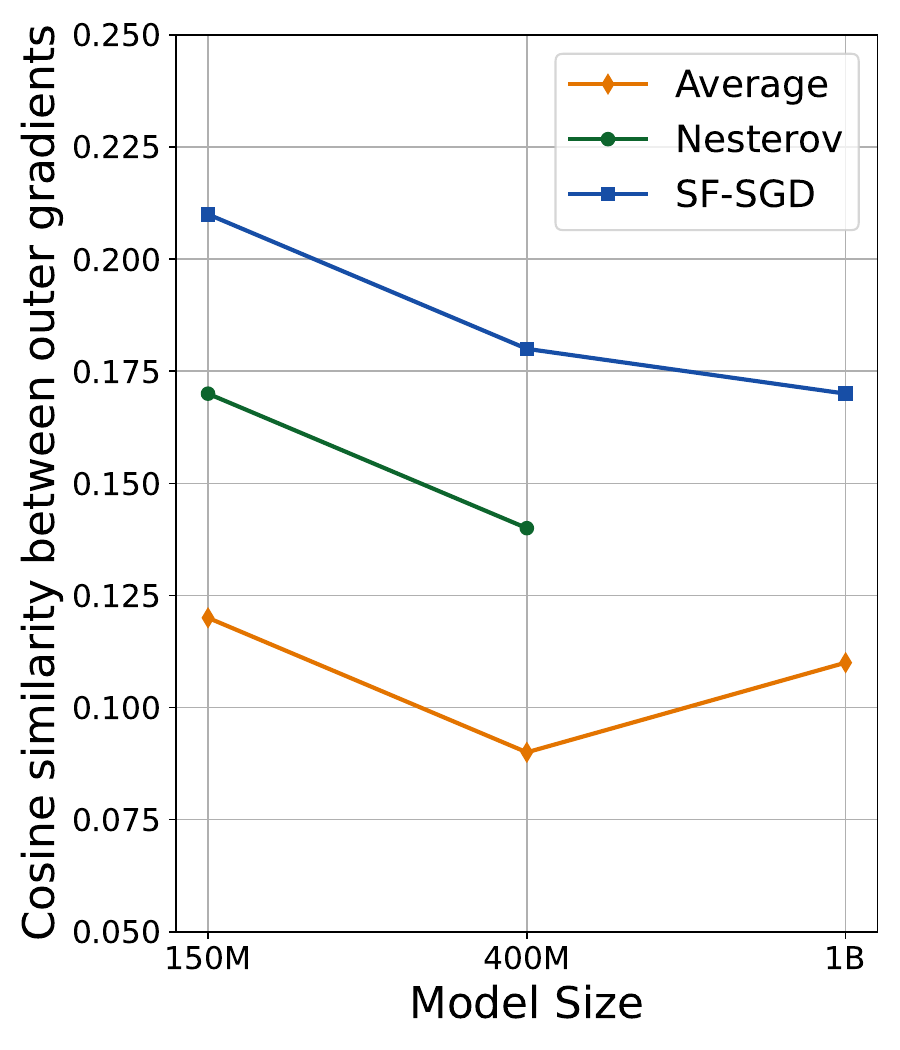}
    \vspace{-2mm}
    \caption{\textbf{Cosine similarity} between outer gradients across different number of replicas (\textit{left}) and model scales (\textit{right}). We average the similarity across the middle 50\% of the training.
    }
\label{fig:pretraining_replica_cos_all}
\vspace{-3mm}
\end{figure}

\subsubsection{Schedule-free but not tuning-free}
\label{sec:language-models-results1}

The schedule-\textit{free} method of \cite{defazio24_road_less_sched} enables not doing any learning rate scheduling, greatly simplifying training configuration. However, it doesn't mean it is hyperparameters-tuning-\textit{free}. Indeed, we found out that we had to extensively tune the initial learning rate (to $2.0$), remove learning rate warm-up contrarily to what is advised, and use a particularly low $b1$ decay: $0.2$, as illustrated in \autoref{fig:pretraining_b1}.

\begin{figure}[t]
    \centering
    \hfill
    \begin{minipage}{0.48\textwidth}
        \includegraphics[width=0.96\linewidth,trim={0cm 0cm 0cm 0cm},clip]{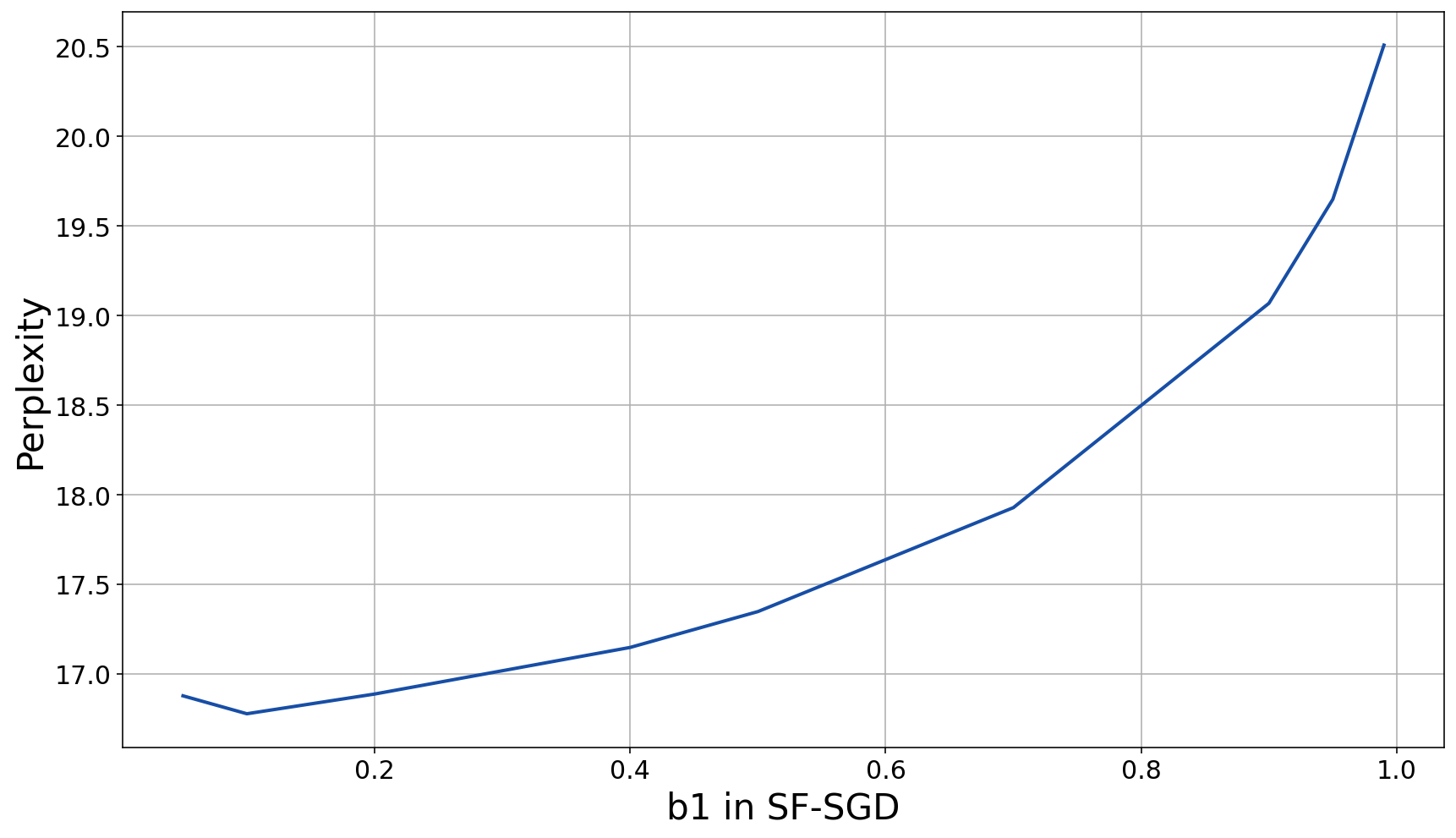}
        \caption{\textbf{Tuning b1 decay} has a major impact on performance, and its value must be very low.}
        \label{fig:pretraining_b1}
    \end{minipage}
    \hfill
    \begin{minipage}{0.48\textwidth}
        \centering
        \vspace{-3mm}
        \includegraphics[width=0.96\linewidth,trim={0cm 0cm 0cm 0cm},clip]{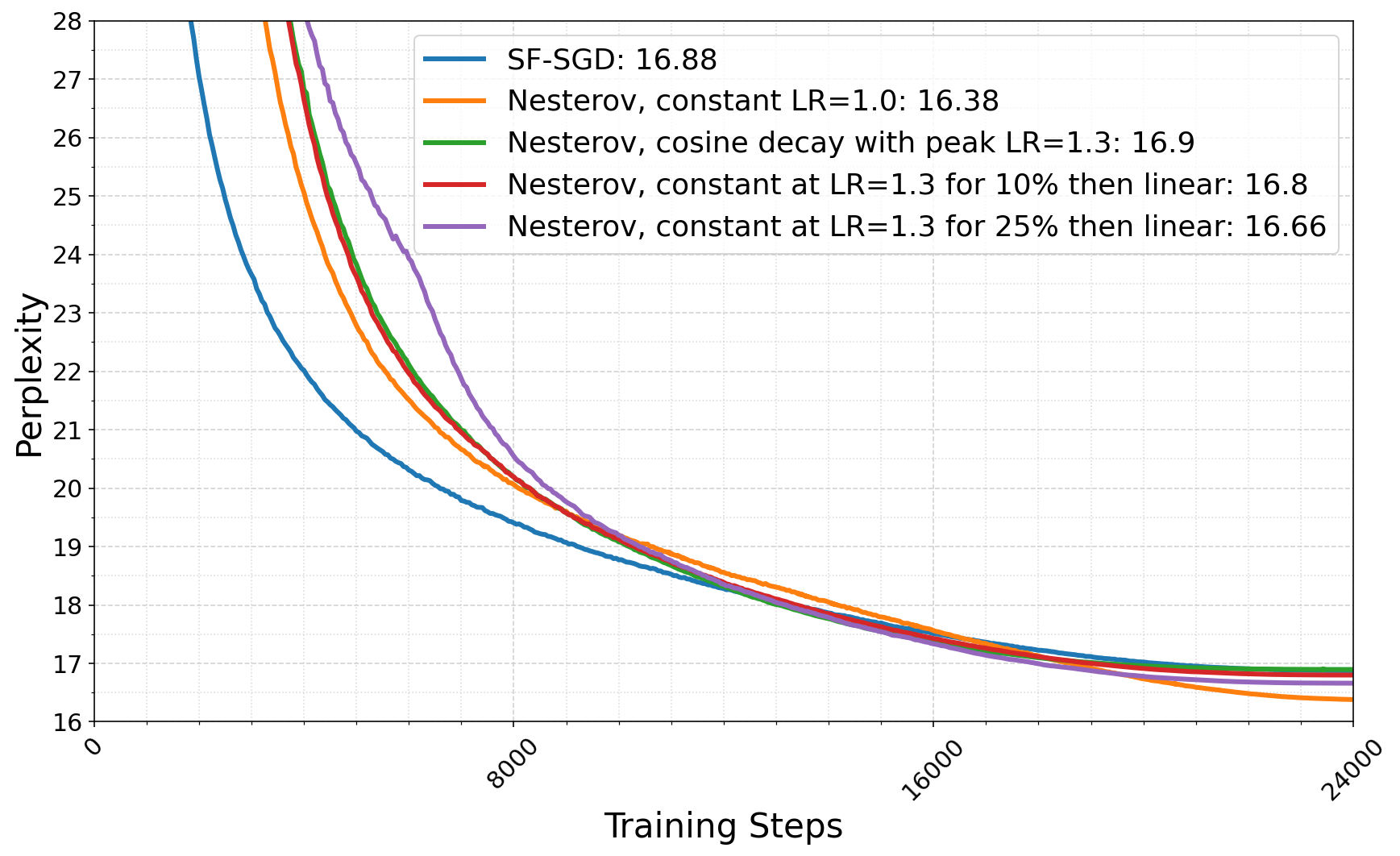}
        \vspace{-3mm}
        \caption{Which outer \textbf{ learning rate schedule} to use?}
        \label{fig:pretraining_lr_scheduling}
    \end{minipage}
    \hfill
\vspace{-2mm}
\end{figure}

\subsubsection{Pretraining: outer learning rate scheduling}
\label{sec:language-models-results2}

Schedule-free SGD enables not having to manually scheduling the outer learning rate. Therefore, we wondered if we could improve the SotA federated learning baseline, DiLoCo (Nesterov outer optimizer), with an outer learning rate schedule. We investigate in \autoref{fig:pretraining_lr_scheduling} three schedules: \textit{constant} as in \citep{douillard23_diloc}, \textit{cosine decay}, and \textit{linear after a plateau}. For the latter we consider a constant plateau for 10\% and 25\% of the total steps. For each method, we also tuned the peak outer learning rate. We don't use any warm-up in the outer optimization as we always found it to be harmful.

We find that constant outer learning rate is the best performing schedule. It's unclear how the other schedules are interacting with the inner learning rate scheduling. A possible solution, not investigated in this report, would be to increase the number of inner steps $H$ as the inner learning rate decreases \citep{gu2024schedulinginnersteps}.

\subsubsection{Cosine similarity between outer gradients}
\label{sec:language-models-results3}

We display the cosine similarity between outer gradients, across scales (150M, 400M, and 1B) in \autoref{fig:pretraining_scaling_cos}, and across replicas (for 150M, from 2 to 16 replicas) in \autoref{fig:pretraining_replicas_cos}. The solid line represent the mean, and the shaded area the standard deviation. We normalize the x-axis as a percentage of the training in order to compare models which have done different amount of steps (e.g. $24{,}000$ steps for 150M vs $30{,}000$ for 400M). 

\begin{figure*}[ht]
    \centering
    \subfigure[150M]{
        \includegraphics[width=0.30\textwidth]
        {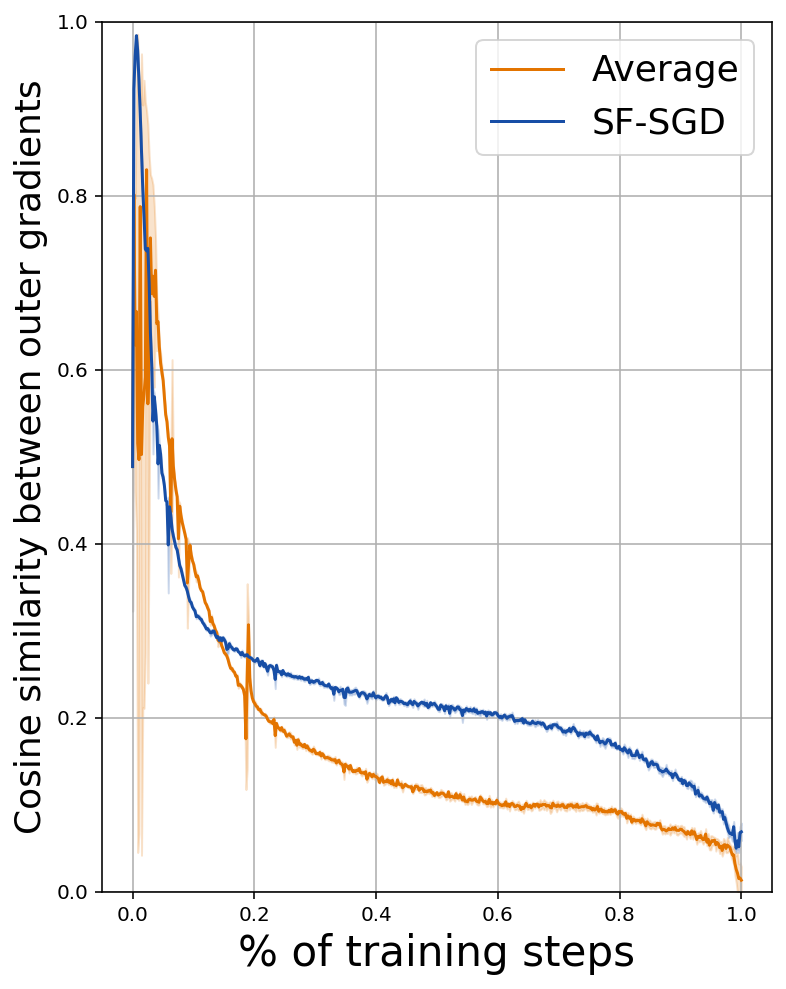}}
    \hfill
    \subfigure[400M]{
        \includegraphics[width=0.30\textwidth]
        {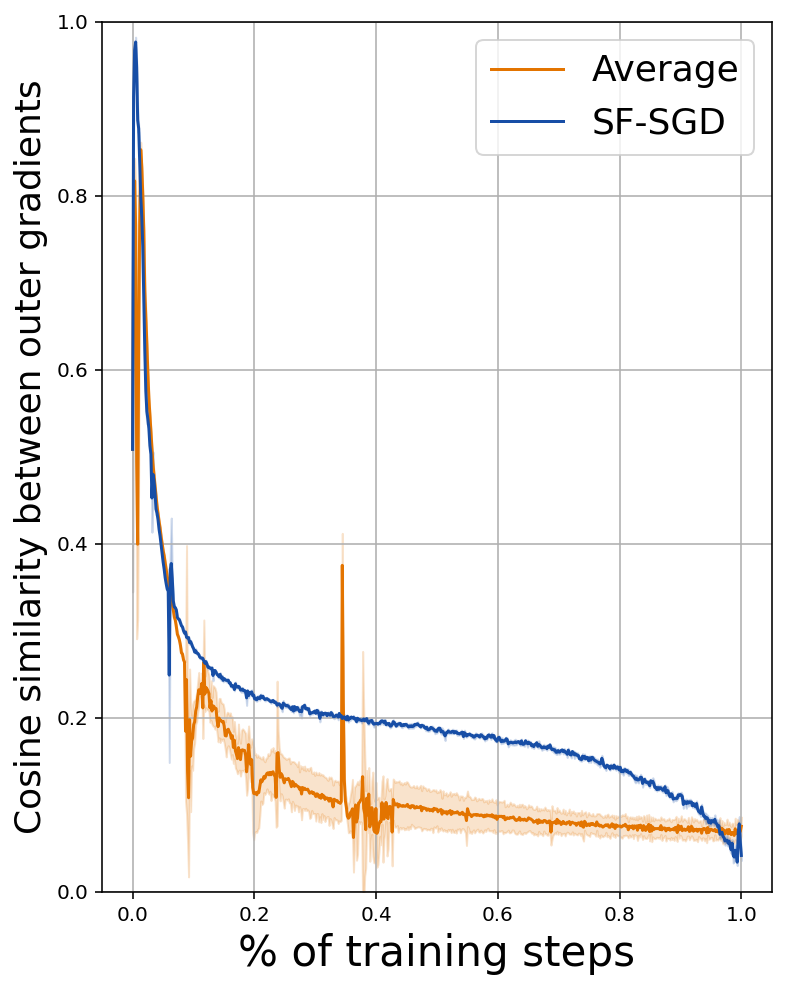}}
    \hfill
    \subfigure[1B]{
        \includegraphics[width=0.30\textwidth]
        {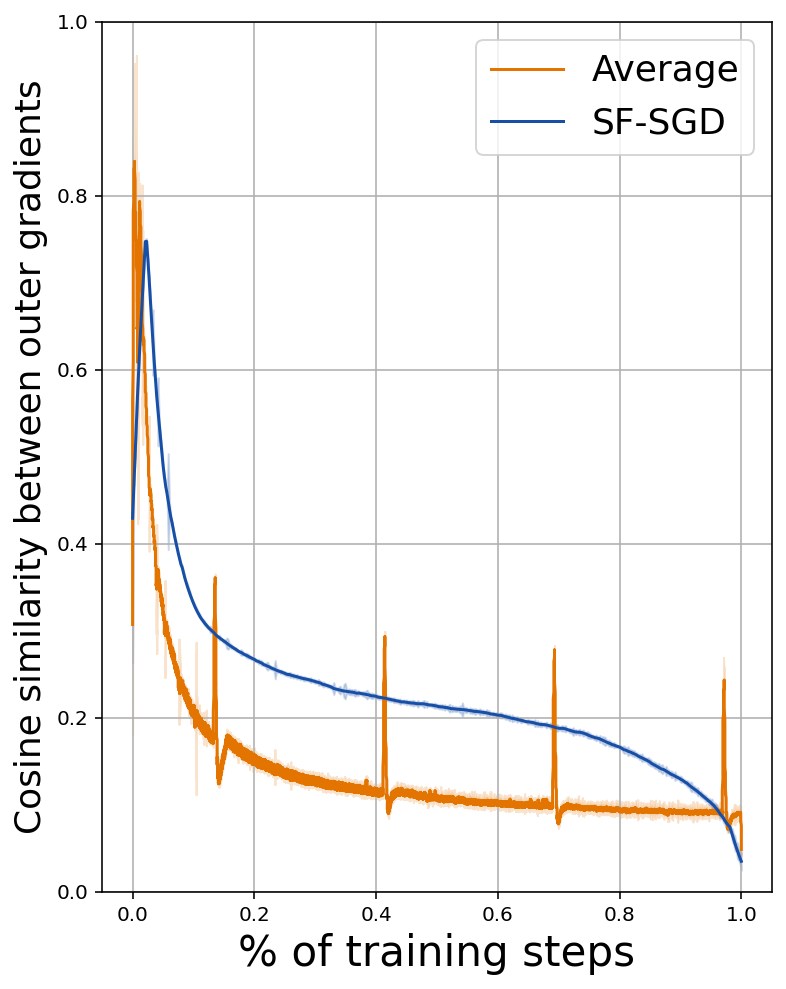}}
\caption{\textbf{Similarity} between outer gradients across steps and scales.}
\label{fig:pretraining_scaling_cos}
\end{figure*}

\begin{figure*}[ht]
    \centering
    \subfigure[M=2 replicas]{
        \includegraphics[width=0.45\textwidth]
        {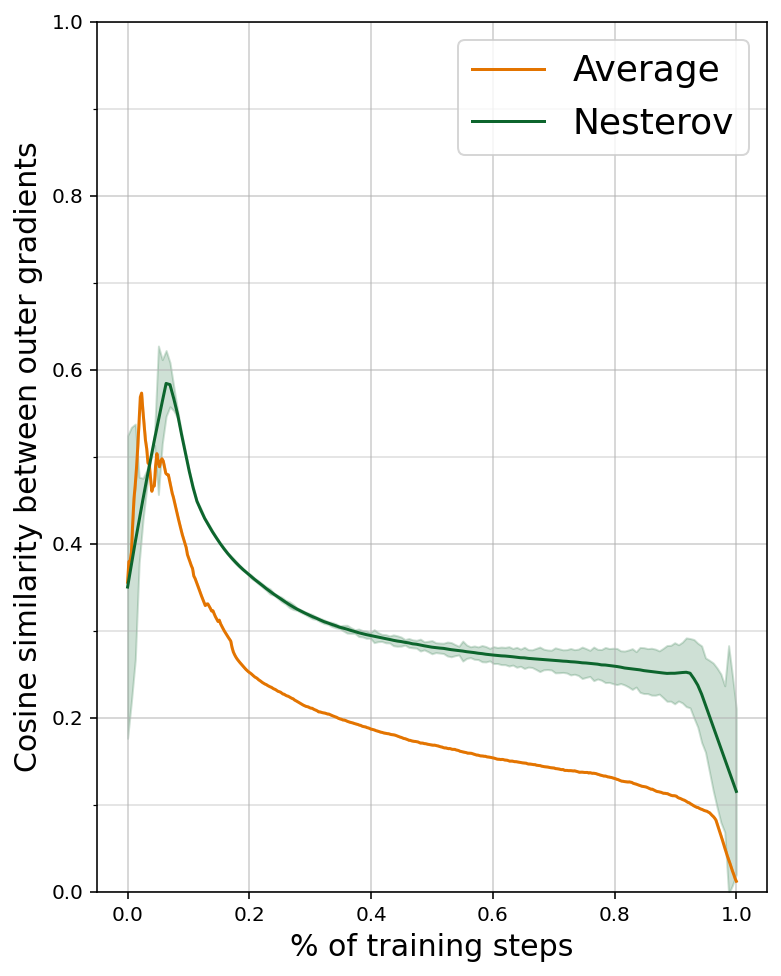}}
    \hfill
    \subfigure[M=4 replicas]{
        \includegraphics[width=0.45\textwidth]
        {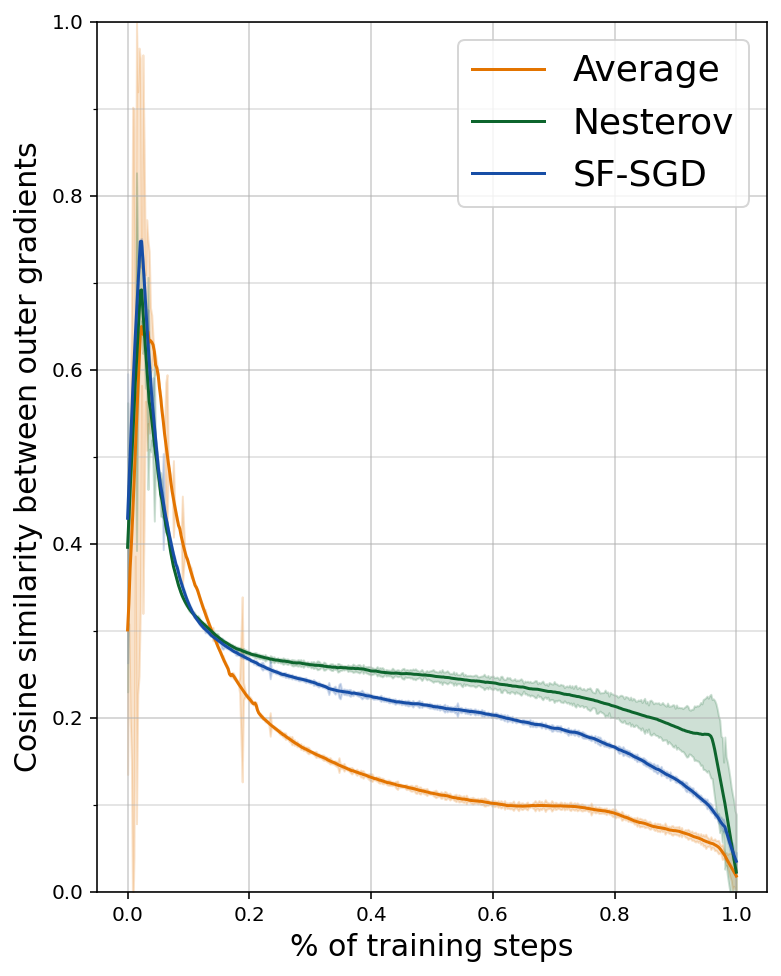}}
    \\
    \subfigure[M=8 replicas]{
        \includegraphics[width=0.45\textwidth]
        {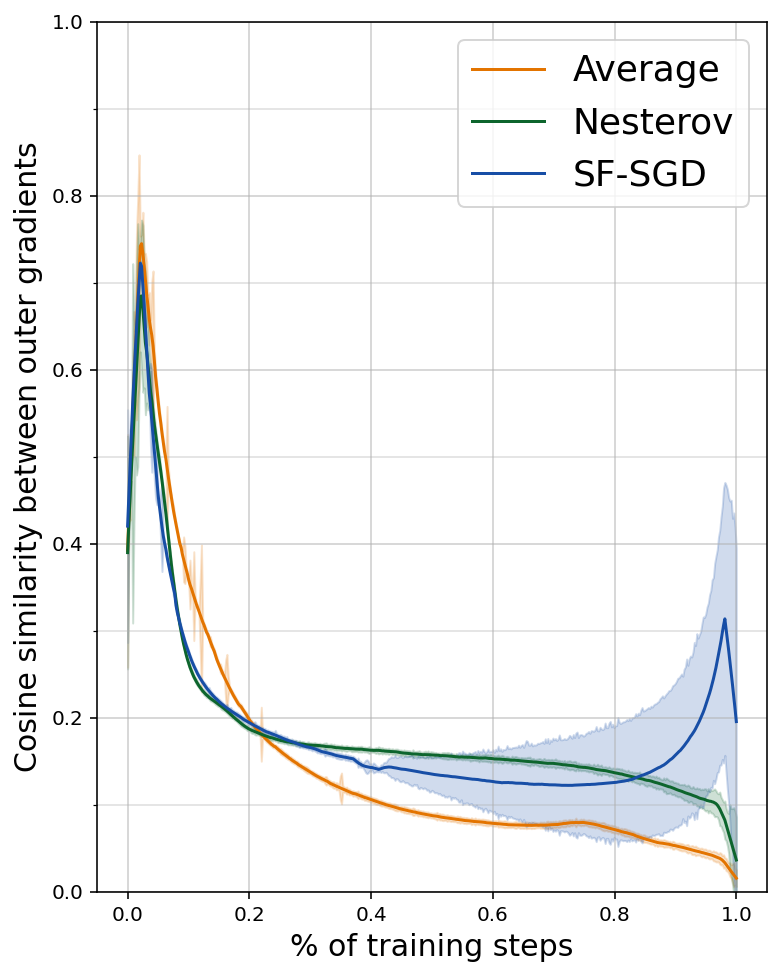}}
    \hfill
    \subfigure[M=16 replicas]{
        \includegraphics[width=0.45\textwidth]
        {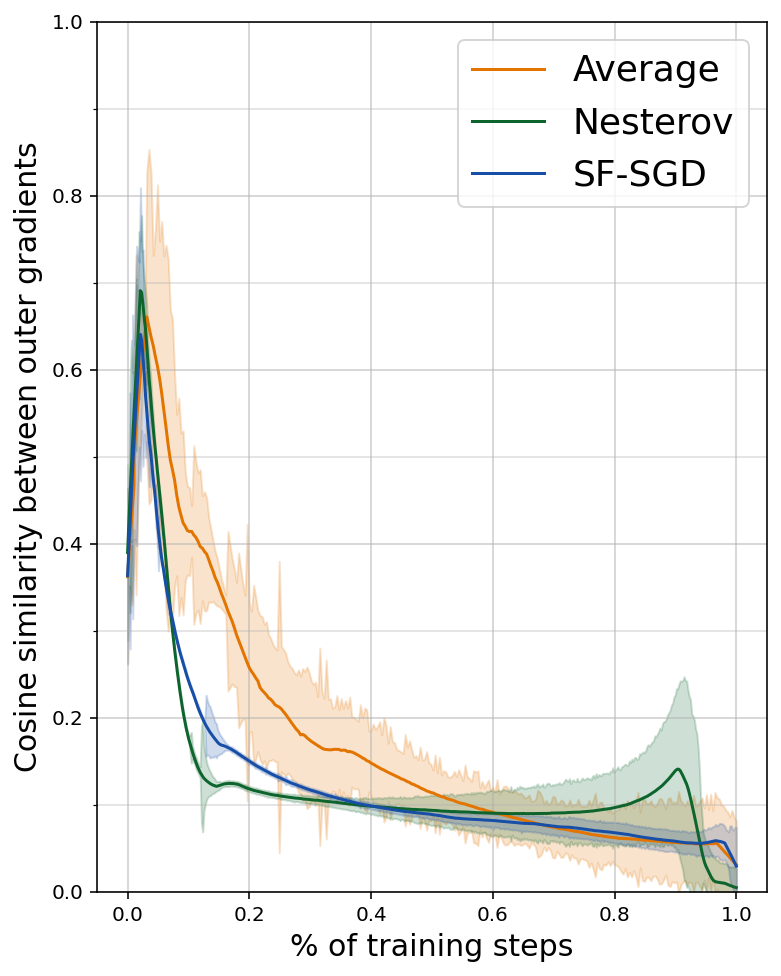}}
\caption{\textbf{Cosine similarity} between outer gradients across steps and number of replicas.}
\label{fig:pretraining_replicas_cos}
\end{figure*}

\begin{table}[h]
\centering
\caption{Complete hyperparameter sweep results across model scales and configurations. All experiments use C4 validation set with sequence length 1024 and batch size 512.}
\label{tab:add_full_sweep}
\begin{tabular}{cccccc}
\toprule
$H$ & $M$ & Algorithm & Learning Rate & Perplexity & Model Size \\
\midrule
\multicolumn{6}{c}{\textit{Data-Parallel Baselines}} \\
\midrule
- & 1 & Data-Parallel & - & 18.07 & 150M \\
- & 1 & Data-Parallel & 4x BS & 16.89 & 150M \\
- & 1 & Data-Parallel & - & 15.28 & 400M \\
- & 1 & Data-Parallel & 4x BS & 13.21 & 400M \\
- & 1 & Data-Parallel & - & 13.38 & 1B \\
- & 1 & Data-Parallel & 4x BS & 11.34 & 1B \\
\midrule
\multicolumn{6}{c}{\textit{Local SGD Experiments}} \\
\midrule
50 & 4 & SGD & 1.0 & 17.75 & 150M \\
50 & 4 & Nesterov & 0.7 & 17.25 & 150M \\
50 & 4 & Nesterov & 1.0 & 16.38 & 150M \\
50 & 4 & SF-SGD & 2.0 ($\beta$=0.2) & 16.88 & 150M \\
50 & 4 & SGD & 1.0 & 14.90 & 400M \\
50 & 4 & Nesterov & 0.7 & 13.71 & 400M \\
50 & 4 & Nesterov & 1.0 & $>$30 & 400M \\
50 & 4 & SF-SGD & 2.0 ($\beta$=0.2) & 13.95 & 400M \\
50 & 4 & SGD & 1.0 & 13.67 & 1B \\
50 & 4 & Nesterov & 0.7 & 12.51 & 1B \\
50 & 4 & SF-SGD & 2.0 ($\beta$=0.2) & 12.40 & 1B \\
\midrule
\multicolumn{6}{c}{\textit{Varying $H$ (Local Steps) at 150M, $M=4$}} \\
\midrule
150 & 4 & SGD & 1.0 & 17.58 & 150M \\
150 & 4 & Nesterov & 0.7 & 17.90 & 150M \\
150 & 4 & Nesterov & 1.0 & 16.79 & 150M \\
150 & 4 & SF-SGD & 2.0 ($\beta$=0.2) & 16.96 & 150M \\
250 & 4 & SGD & 1.0 & 18.20 & 150M \\
250 & 4 & Nesterov & 0.7 & 18.09 & 150M \\
250 & 4 & Nesterov & 1.0 & 17.12 & 150M \\
250 & 4 & SF-SGD & 2.0 ($\beta$=0.2) & 16.97 & 150M \\
500 & 4 & SGD & 1.0 & 18.44 & 150M \\
500 & 4 & Nesterov & 0.7 & 17.95 & 150M \\
500 & 4 & Nesterov & 1.0 & 18.15 & 150M \\
500 & 4 & SF-SGD & 2.0 ($\beta$=0.2) & 17.18 & 150M \\
1000 & 4 & SGD & 1.0 & 18.18 & 150M \\
1000 & 4 & Nesterov & 0.7 & 18.16 & 150M \\
1000 & 4 & Nesterov & 1.0 & 18.75 & 150M \\
1000 & 4 & SF-SGD & 2.0 ($\beta$=0.2) & 17.29 & 150M \\
2000 & 4 & SGD & 1.0 & 18.11 & 150M \\
2000 & 4 & Nesterov & 0.7 & 18.40 & 150M \\
2000 & 4 & Nesterov & 1.0 & 18.36 & 150M \\
2000 & 4 & SF-SGD & 2.0 ($\beta$=0.2) & 17.59 & 150M \\
\midrule
\multicolumn{6}{c}{\textit{Varying $M$ (Number of Nodes) at 150M, $H=50$}} \\
\midrule
50 & 2 & SGD & 1.0 & 18.64 & 150M \\
50 & 2 & Nesterov & 1.0 & 16.81 & 150M \\
50 & 2 & SF-SGD & 2.0 ($\beta$=0.2) & 17.13 & 150M \\
50 & 8 & SGD & 1.0 & 18.38 & 150M \\
50 & 8 & Nesterov & 1.0 & 16.27 & 150M \\
50 & 8 & SF-SGD & 2.0 ($\beta$=0.2) & 16.92 & 150M \\
50 & 16 & SGD & 1.0 & 19.86 & 150M \\
50 & 16 & Nesterov & 1.0 & 16.25 & 150M \\
50 & 16 & SF-SGD & 2.0 ($\beta$=0.2) & 16.75 & 150M \\
\bottomrule
\end{tabular}%
\end{table}

\begin{table}[h]
\centering
\caption{Additional outer learning rate sweeps for different outer optimizers. All experiments at 150M model size with $H=50$ and $M=4$.}
\label{tab:add_lr_sweep}
\begin{tabular}{lcc}
\toprule
Algorithm & Learning Rate & Perplexity \\
\midrule
\multicolumn{3}{c}{\textit{SF-SGD Learning Rate Sweep ($\beta=0.2$)}} \\
\midrule
SF-SGD & 0.1 & $>$30 \\
SF-SGD & 0.5 & 22.89 \\
SF-SGD & 1.0 & 19.42 \\
SF-SGD & 1.5 & 18.32 \\
SF-SGD & 2.0 & 17.98 \\
SF-SGD & 3.0 & 17.96 \\
SF-SGD & 4.0 & 18.09 \\
SF-SGD & 5.0 & 17.51 \\
\midrule
\multicolumn{3}{c}{\textit{Nesterov Learning Rate Sweep (Cosine Schedule)}} \\
\midrule
Nesterov & 0.3 & 17.16 \\
Nesterov & 0.5 & 17.06 \\
Nesterov & 0.7 & 16.93 \\
Nesterov & 0.9 & 17.19 \\
Nesterov & 1.1 & 17.56 \\
\midrule
\multicolumn{3}{c}{\textit{SGD Learning Rate Sweep}} \\
\midrule
SGD & 0.3 (fixed) & 21.04 \\
SGD & 0.3 (cosine) & 17.68 \\
SGD & 0.5 (cosine) & 16.63 \\
SGD & 0.7 (cosine) & 18.84 \\
SGD & 1.0 (cosine) & 19.21 \\
\bottomrule
\end{tabular}
\end{table}

\begin{table}[h]
\centering
\caption{SF-SGD $\beta$ parameter sweep at 150M model size with $H=50$, $M=4$, and outer learning rate $\gamma=2.0$.}
\label{tab:add_beta_sweep}
\begin{tabular}{cc}
\toprule
$\beta$ Value & Perplexity \\
\midrule
0.0 & $>$30 \\
0.05 & 16.88 \\
0.1 & 16.78 \\
0.2 & 16.89 \\
0.4 & 17.15 \\
0.5 & 17.35 \\
0.7 & 17.93 \\
0.9 & 19.07 \\
0.95 & 19.65 \\
0.99 & 20.51 \\
\bottomrule
\end{tabular}
\end{table}

\clearpage
\part*{Theory}
\section{Guarantees for Local SGD}

First, we recall our setting and define some notation. We consider the problem of minimizing a function $f$ in a distributed setting with $M$ workers performing Local SGD. Let $x_r$ denote the global model parameters at the beginning of round $r$. Each worker $m$ initializes its local parameters as $y_{m,r,0} = x_r$ and performs $H$ local SGD steps according to
$$y_{m,r,h+1} = y_{m,r,h} - \eta g_{m,r,h},$$
where $g_{m,r,h} = \nabla f(y_{m,r,h}) + n_{m,r,h}$ is the stochastic gradient with noise $n_{m,r,h}$, and $\overline{g}_{m,r,h} = \nabla f(y_{m,r,h})$ is the true gradient. By Assumption~\ref{asm:stoch-gradients} we have $\ec{g_{m, r, h}} = \overline{g}_{m, r, h}$. After $H$ local steps, the global model update can be equivalently written as $x_{r+1} = x_r - \gamma\eta\sum_{h=0}^{H-1}g_{r,h}$ where $g_{r,h} = \frac{1}{M}\sum_{m=1}^M g_{m,r,h}$ is the average gradient across workers and $y_{r,h} = \frac{1}{M}\sum_{m=1}^M y_{m,r,h}$ is the average model. Note that these two last sequences are virtual sequences and not actually computed. We also define $x_{r,h} = x_r - \gamma\eta\sum_{h=0}^{H-1}g_{r,h}$ as an intermediate quantity used in the analysis. \Cref{tab:notation_compact} summarizes some of the notation we use throughout this section.

\begin{table}[h]
\centering
\caption{Key notation.}
\label{tab:notation_compact}
\small
\begin{tabular}{cl|cl}
\toprule
\textbf{Symbol} & \textbf{Description} & \textbf{Symbol} & \textbf{Description} \\
\midrule
$M$ & Number of nodes & $x_r$ & Global iterate at round $r$ \\
$H$ & Local steps per round & $y_{r,h}$ & Averaged local iterate \\
$R$ & Communication rounds & $g_{r,h}$ & Averaged stochastic gradient \\
$\eta$ & Inner learning rate & $L$ & Smoothness constant \\
$\gamma$ & Outer learning rate & $\sigma^2$ & Gradient variance bound \\
$\mu$ & Momentum parameter & $D$ & $\|x_0 - x_*\|$ initial distance to optimum \\
\bottomrule
\end{tabular}
\end{table}

\subsection{Algorithm-independent results}

\begin{lemma}\label{lem:contractivity} \citep[Lemma 6]{karimireddy19_scaff}
Let $f$ be a convex and $L$-smooth function. Suppose that $\eta \leq \frac{2}{L}$, let $T_{\eta} (x) = x - \eta \nabla f(x)$. Then
\begin{align*}
\sqn{T_{\eta} (x) - T_{\eta} (y)} \leq \sqn{x-y}.
\end{align*}
\end{lemma}
\begin{proof}
  The proof is provided for completeness only. We have
  \begin{align}
\label{eq:contraction-1}
  \sqn{T_{\eta} (x) - T_{\eta} (y)} &= \sqn{x-y} + \eta^2 \sqn{\nabla f(x) - \nabla f(y)} - 2 \eta \ev{ x-y , \nabla f(x) - \nabla f(y) }.
  \end{align}
  By the Baillon-Haddad theorem~\citep{bauschke09_baill_haddad_theor_revis} we have
  \begin{align*}
  \ev{ x-y , \nabla f(x) - \nabla f(y) } \geq \frac{1}{L} \sqn{\nabla f(x) - \nabla f(y)}.
  \end{align*}
  Using this in \Cref{eq:contraction-1} gives
  \begin{align*}
  \sqn{T_{\eta} (x) - T_{\eta} (y)} \leq \sqn{x-y} - \eta \left( \frac{2}{L} - \eta \right) \sqn{\nabla f(x) -  \nabla f(y)}.
  \end{align*}
  If $\eta \leq \frac{1}{L}$ then $\frac{2}{L} -\eta \geq 0$ and therefore $\sqn{T_{\eta} (x) - T_{\eta} (y)} \leq \sqn{x-y}$.
\end{proof}

\begin{lemma}\label{lem:jensen-application}
  Let $y_1, \ldots, y_n$ be real numbers. Then,
\begin{align*}
  \frac{1}{n} \sum_{k=1}^n \abs{y_i} \leq \sqrt{\frac{1}{n} \sum_{k=1}^n y_i^2}.
\end{align*}
\end{lemma}
\begin{proof}
This is just the arithmetic mean-root mean square inequality and we include the proof solely for completeness. Let $Y$ be a random variable that takes the value $y_i^2$ with probability
$\frac{1}{n}$, and let $g(x) = \sqrt{x}$. Observe that
\begin{align*}
  \frac{1}{n} \sum_{k=1}^n \abs{y_i} = \ec{g(Y)}.
\end{align*}
Since $g$ is a concave function, by Jensen's inequality we have that
$\ec{g(Y)} \leq g(\ec{Y})$. Therefore,
\begin{align*}
  \frac{1}{n} \sum_{k=1}^n \abs{y_i} = \ec{g(Y)} \leq g(\ec{Y}) = \sqrt{\frac{1}{n} \sum_{k=1}^n y_i^2}.
\end{align*}
\end{proof}

\begin{lemma}\label{lem:var-generic}
  (Variance of Sum of Conditionally Independent Random Variables).
  Let $Z_1, \ldots, Z_n$ be random variables such that $Z_i$ satisfies
  \begin{align*}
  \ec[i-1]{Z_i} = 0, \qquad\qquad \text { and, } && \ecn{Z_i} = \sigma_i^2,
  \end{align*}
  where $\ec[i]{\cdot}$ denotes expectation conditional on $Z_1, Z_2, \ldots, Z_i$. Then,
  \begin{align*}
    \ecn{\sum_{i=1}^n Z_i} = \sum_{i=1}^n \sigma_i^2.
  \end{align*}
\end{lemma}
\begin{proof}
  \begin{align*}
    \ecn{\sum_{i=1}^n Z_i} &= \ec{\ecn[n-1]{\sum_{i=1}^n Z_i}} \\
                           &= \ec{\ec[n-1]{\sqn{\sum_{i=1}^{n-1} Z_i} + \sqn{Z_n} + 2\ev{\sum_{i=1}^{n-1} Z_i, Z_n}}} \\
                           &= \ec{\ecn[n-1]{\sum_{i=1}^{n-1} Z_i} + \sigma_n^2}.
  \end{align*}

  The cross-term $\ec[n-1]{2\ev{\sum_{i=1}^{n-1} Z_i, Z_n}}$ vanishes because $\ec[n-1]{Z_n} = 0$ and $\sum_{i=1}^{n-1} Z_i$ is measurable with respect to the sigma-algebra generated by $Z_1,\ldots,Z_{n-1}$. Continuing,
  \begin{align*}
    \ecn{\sum_{i=1}^n Z_i} &= \ecn{\sum_{i=1}^{n-1} Z_i} + \sigma_n^2.
  \end{align*}
  Recursing we get,
  \begin{align*}
    \ecn{\sum_{i=1}^n Z_i} &= \sum_{i=1}^n \sigma_i^2.
  \end{align*}
  This completes the proof.
\end{proof}

\begin{lemma}\label{lem:dog-concentration}
\citep[Lemma 7]{ivgi23_dog_is_sgds_best_frien}. Let $S$ be the set of
nonnegative and nondecreasing sequences. Let $y_1, y_2, \ldots$ be a sequence in $S$. Let $C_t \in \mathcal{F}_{t-1}$ for all $t=1, 2, \ldots, T$ and let $X_t$ be a
martingale difference sequence adapted to $\mathcal{F}_t$ such that $\abs{X_t} \leq C_t$ with
probability $1$ for $t = 1, 2, \ldots, T$. Then for all $\delta \in (0, 1)$ and $\hat{X}_t \in \mathcal{F}_{t-1}$
such that $\abs{\hat{X}_t} \leq C_t$ with probability $1$, we have that with
probability at least $1-\delta-\pr[ \exists t \leq T \mid C_t > c ]$ that for all $c > 0$
\begin{align*}
  \abs{\sum_{i=1}^t y_i X_i} \leq 8 y_t \sqrt{\theta_{t, \delta} \sum_{i=1}^t (X_i - \hat{X}_i)^2 + c^2 \theta_{t, \delta}^2},
\end{align*}
where $\theta_{t, \delta} = \log \frac{60 \log 6t}{\delta}$.
\end{lemma}

\begin{lemma}\label{lem:recursion}
Suppose we have
\begin{align*}
r_{k+1} \leq (1+a) r_k - b \delta_k + c
\end{align*}
Then,
\begin{align*}
\min_j \delta_j &\leq \frac{r_0 e^{aK}}{b K} + \frac{c}{b}.
\end{align*}
\end{lemma}
\begin{proof}
Let $w_{k+1} = \frac{w_k}{1+a}$. We have
\begin{align*}
w_{k+1} r_{k+1} &\leq (1+a) w_{k+1} r_k - b w_{k+1} \delta_k + c w_{k+1} \\
&= w_k r_k - b w_{k+1} \delta_k + c w_{k+1}.
\end{align*}
Telescoping,
\begin{align*}
  w_K r_K &\leq w_0 r_0 - b \sum_{j=0}^{K-1} w_{j+1} \delta_j + c \sum_{j=0}^{K-1} w_{j+1}.
\end{align*}
Rearranging,
\begin{align*}
  \frac{1}{\sum_{j=0}^{K-1} w_{j+1}}  \sum_{j=0}^{K-1} w_{j+1} \delta_j \leq \frac{w_0 r_0}{b \sum_{j=0}^{K-1} w_{j+1}} + \frac{c}{b}.
\end{align*}
We have $w_k = \frac{w_{k-1}}{1+a} = \frac{w_0}{(1+a)^k}$. Therefore,
\begin{align*}
  \sum_{j=0}^{K-1} w_{j+1} &= \sum_{j=0}^{K-1} \frac{w_0}{(1+a)^{k+1}} \\
                           &\geq \sum_{j=0}^{K-1} \frac{w_0}{(1+a)^K}  \\
                           &= \frac{w_0 K}{(1+a)^K}.
\end{align*}
Therefore,
\begin{align*}
\frac{1}{\sum_{j=0}^{K-1} w_{j+1}}  \sum_{j=0}^{K-1} w_{j+1} \delta_j &\leq \frac{r_0 (1+a)^K}{b K} + \frac{c}{b}.
\end{align*}
Finally, it remains to use that $1+a \leq e^a$.
\end{proof}

\subsection{Convergence guarantees without momentum}\label{sec:non-adaptive}

We begin with a lemma that establishes the regret of the local optimizer. Often the regret is measured against the optimal point (like $x_{\ast}$) but here we instead utilize it against the \emph{initial} point $y_{r, 0} = x_r$.

\begin{lemma}[Regret against starting point]\label{lem:inner-prod-1}
  For any learning rate $\eta > 0$, the inner product between the displacement from the initial average iterate and the average gradient satisfies,
  \begin{align*}
  \sum_{h=0}^{H-1} \langle y_{r,h} - y_{r,0}, g_{r,h} \rangle \leq \frac{\eta}{2}\sum_{h=0}^{H-1}\sqn{g_{r,h}} - \frac{1}{2\eta} \sqn{y_{r,H} - y_{r,0}}.
  \end{align*}
\end{lemma}
\begin{proof}
We begin by using that $y_{r, h+1} = y_{r, h} - \eta g_{r, h}$ and expanding the square as
\begin{align*}
\sqn{y_{r,h+1} - y_{r,0}} &= \sqn{y_{r,h} - \eta g_{r,h} - y_{r,0}} \\
&= \sqn{y_{r,h} - y_{r,0}} + \eta^2\sqn{g_{r,h}} - 2\eta\ev{y_{r,h} - y_{r,0}, g_{r,h}}.
\end{align*}
Rearranging to isolate the inner product term, we obtain
\begin{align*}
\ev{y_{r,h} - y_{r,0}, g_{r,h}} = \frac{\sqn{y_{r,h} - y_{r,0}} - \sqn{y_{r,h+1} - y_{r,0}}}{2\eta} + \frac{\eta}{2}\sqn{g_{r,h}}.
\end{align*}
Summing over $h$ from $0$ to $H-1$,
\begin{align*}
\sum_{h=0}^{H-1}\ev{y_{r,h} - y_{r,0}, g_{r,h}} &= \sum_{h=0}^{H-1}\left(\frac{\sqn{y_{r,h} - y_{r,0}} - \sqn{y_{r,h+1} - y_{r,0}}}{2\eta} + \frac{\eta}{2}\sqn{g_{r,h}}\right) \\
&= \frac{1}{2\eta}\sum_{h=0}^{H-1}(\sqn{y_{r,h} - y_{r,0}} - \sqn{y_{r,h+1} - y_{r,0}}) + \frac{\eta}{2}\sum_{h=0}^{H-1}\sqn{g_{r,h}}.
\end{align*}
The first sum telescopes
\begin{align*}
\sum_{h=0}^{H-1}(\sqn{y_{r,h} - y_{r,0}} - \sqn{y_{r,h+1} - y_{r,0}}) &= \sqn{y_{r,0} - y_{r,0}} - \sqn{y_{r,H} - y_{r,0}} \\
&= -\sqn{y_{r,H} - y_{r,0}}.
\end{align*}
Therefore,
\begin{align*}
\sum_{h=0}^{H-1}\ev{y_{r,h} - y_{r,0}, g_{r,h}} &= -\frac{\sqn{y_{r,H} - y_{r,0}}}{2\eta} + \frac{\eta}{2}\sum_{h=0}^{H-1}\sqn{g_{r,h}} \\
&\leq \frac{\eta}{2}\sum_{h=0}^{H-1}\sqn{g_{r,h}} - \frac{\sqn{y_{r,H} - y_{r,0}}}{2\eta}.
\end{align*}
\end{proof}

\begin{lemma}\label{lem:Vr-bound}(Local client drift bound). Suppose that \Cref{asm:f-cvx-smooth,asm:stoch-gradients} hold. Then in Algorithm~\ref{eq:alg-loc-sgd} for all $r$ and $h$, if $\eta \leq \frac{1}{L}$, then
\begin{align*}
\ec{\frac{1}{M^2} \sum_{m, s = 1}^M \sqn{y_{m, r, h} - y_{s, r, h}}} \leq 2 \eta^2 \sigma^2 h.
\end{align*}
\end{lemma}
\begin{proof}
Let $\tilde{T}_{\eta}(y_{m, r, h}) = y_{m, r, h} - \eta g_{m,  r, h}$ where $g_{m, r, h}$ is the stochastic gradient, and $T_{\eta}(y_{m, r, h}) = y - \eta \overline{g}_{m, r, h}$ is the corresponding expected gradient update. We have
\begin{align*}
y_{m, r, h+1} - y_{s, r, h+1} &= \tilde{T}_{\eta}(y_{m, r, h}) - \tilde{T}_{\eta}(y_{s, r, h}) \\
&= T_{\eta}(y_{m, r, h}) - T_{\eta}(y_{s, r, h}) + \left[ \tilde{T}_{\eta}(y_{m, r, h}) - \tilde{T}_{\eta}(y_{s, r, h}) - (T_{\eta}(y_{m, r, h}) - T_{\eta}(y_{s, r, h})) \right] \\
&= T_{\eta}(y_{m, r, h}) - T_{\eta}(y_{s, r, h}) + \left[ \xi_{m, r, h} - \xi_{s, r, h} \right],
\end{align*}
where $\xi_{m, r,h} = \tilde{T}_{\eta}(y_{m, r, h}) - T_{\eta}(y_{m, r, h}) = -\eta n_{m,r,h}$ is the noise term. Define $\mathcal{V}_{r,h} = \frac{1}{M^2} \sum_{m, s=1}^M \sqn{y_{m, r, h} - y_{s, r, h}}$. It follows that
\begin{align*}
\mathcal{V}_{r,h+1} &= \frac{1}{M^2} \sum_{m, s=1}^M \sqn{y_{m, r, h+1} - y_{s, r, h+1}} \\
\begin{split}
&= \frac{1}{M^2} \sum_{m, s=1}^M \Bigg [ \sqn{T_{\eta}(y_{m, r, h}) - T_{\eta}(y_{s, r, h})} + \sqn{\xi_{m, r,h} - \xi_{s, r,h}} \\
&\qquad + 2 \ev{T_{\eta}(y_{m, r, h}) - T_{\eta}(y_{s, r, h}), \xi_{m, r,h} - \xi_{s, r,h}} \Bigg].
\end{split}
\end{align*}
Taking conditional expectation gives
\begin{align*}
\begin{split}
\ec[r, h]{\mathcal{V}_{r, h+1}} &= \frac{1}{M^2} \sum_{m, s=1}^M \Bigg [ \sqn{T_{\eta}(y_{m, r, h}) - T_{\eta}(y_{s, r, h})} + \ecn[h]{\xi_{m, r,h} - \xi_{s, r,h}} \Bigg].
\end{split}
\end{align*}
Finally, using the fact that $\sqn{T_{\eta}(x) - T_{\eta}(y)} \leq \sqn{x-y}$ whenever $\eta \leq \frac{2}{L}$ (Lemma~\ref{lem:contractivity}) and Assumption~\ref{asm:stoch-gradients}, we get
\begin{align*}
  \ec[r, h]{\mathcal{V}_{r, h+1}} &\leq \frac{1}{M^2} \sum_{m,s=1}^M \left[ \sqn{y_{m, r, h} - y_{s, r, h}} + 2 \eta^2 \sigma^2 \right] \\
  &= \mathcal{V}_{r, h} + 2 \eta^2 \sigma^2.
\end{align*}
Therefore by taking unconditional expectation and recursing from $h=0$ where all local iterates are equal to $x_r$ (so $\mathcal{V}_{r,0} = 0$), we get $\ec{\mathcal{V}_{r, h}} \leq 2 \eta^2 \sigma^2 h$.
\end{proof}

\begin{proof}[Proof of \Cref{thm:gen-loc-sgd}]
W begin by analyzing how the squared distance to the optimal solution changes after one round of communication. From the update rule, we have,
\begin{align}
\label{eq:3}
\sqn{x_{r+1} - x_*} &= \sqn{x_r - x_*} - 2\eta\gamma\sum_{h=0}^{H-1}\ev{x_r - x_*, g_{r,h}} + \eta^2\gamma^2\sqn{\sum_{h=0}^{H-1}g_{r,h}}.
\end{align}
We rewrite the inner product term as
\begin{align*}
-\ev{x_r - x_*, g_{r,h}} &= \ev{x_* - x_r, g_{r,h}} \\
&= \ev{x_* - y_{r,h}, g_{r,h}} + \ev{y_{r,h} - x_r, g_{r,h}}.
\end{align*}
Summing over all local steps we obtain
\begin{align*}
-\sum_{h=0}^{H-1}\ev{x_r - x_*, g_{r,h}} &= \sum_{h=0}^{H-1}\ev{x_* - y_{r,h}, g_{r,h}} + \sum_{h=0}^{H-1}\ev{y_{r,h} - x_r, g_{r,h}}.
\end{align*}

Applying \Cref{lem:inner-prod-1} we get
\begin{align}
\label{eq:2}
-\sum_{h=0}^{H-1}\ev{x_r - x_*, g_{r,h}} &= \sum_{h=0}^{H-1}\ev{x_* - y_{r,h}, g_{r,h}} - \frac{\sqn{y_{r,H} - y_{r,0}}}{2\eta} + \frac{\eta}{2}\sum_{h=0}^{H-1}\sqn{g_{r,h}}.
\end{align}
Observe that since $y_{r,H} - y_{r,0} = -\eta\sum_{h=0}^{H-1}g_{r,h}$, \cref{eq:2} becomes,
\begin{align*}
-\sum_{h=0}^{H-1}\ev{x_r - x_*, g_{r,h}} &= \sum_{h=0}^{H-1}\ev{x_* - y_{r,h}, g_{r,h}} - \frac{\eta}{2}\sqn{\sum_{h=0}^{H-1}g_{r,h}} + \frac{\eta}{2}\sum_{h=0}^{H-1}\sqn{g_{r,h}}.
\end{align*}
Plugging this back into \cref{eq:3},
\begin{align*}
\sqn{x_{r+1} - x_*} &\leq \sqn{x_r - x_*} + 2\eta\gamma\sum_{h=0}^{H-1}\ev{x_* - y_{r,h}, g_{r,h}} \\
&\quad + \gamma\eta^2\sum_{h=0}^{H-1}\sqn{g_{r,h}} + \eta^2\gamma(\gamma-1)\sqn{\sum_{h=0}^{H-1}g_{r,h}}.
\end{align*}
Let us take expectation conditional on $x_1, \ldots, x_r$,
  \begin{align}
    \begin{split}
      \ecn[r]{x_{r+1} - x_{\ast}} &\leq \sqn{x_r - x_{\ast}} + 2 \eta \gamma \sum_{h=0}^{H-1} \ec[r]{\ev{ x_{\ast} - y_{r, h} , g_{r, h} }} \\
      &+ \gamma \eta^2 \sum_{h=0}^{H-1} \ecn[r]{g_{r, h}} + \eta^2 \gamma (\gamma - 1) \ecn[r]{\sum_{h=0}^{H-1} g_{r, h}}.
    \end{split}
   \label{eq:4}
  \end{align}
For the squared norm of the average gradient:
\begin{align*}
\ecn[r]{g_{r,h}} &= \ec[r]{\ecn[r,h-1]{g_{r,h}}} \\
&= \ec[r]{\ecn[r,h-1]{g_{r,h} - \overline{g}_{r,h}} + \sqn{\overline{g}_{r,h}}} \\
&= \frac{\sigma^2}{M} + \ecn[r]{\overline{g}_{r,h}},
\end{align*}
where we use $\ec[r, h-1]{\cdot}$ to denote expectation conditional on the $\sigma$-algebra generated by all the stochastic gradients up to and including step $h-1$. Substituting this into \cref{eq:4},
  \begin{align}
    \begin{split}
      \ecn[r]{x_{r+1} - x_{\ast}} &\leq \sqn{x_r - x_{\ast}} + 2 \eta \gamma \sum_{h=0}^{H-1} \ec[r]{\ev{ x_{\ast} - y_{r, h} , g_{r, h} }} + \frac{\gamma \eta^2 H \sigma^2}{M} \\
      &+ \gamma \eta^2 \sum_{h=0}^{H-1} \ecn[r]{\overline{g}_{r, h}} + \eta^2 \gamma (\gamma - 1) \ecn[r]{\sum_{h=0}^{H-1} g_{r, h}}.
    \end{split}
   \label{eq:4-1}
  \end{align}

Now we bound the inner product term:
  \begin{align}
    \ec[r]{\ev{ x_{\ast} - y_{r, h} , g_{r, h} }} &= \ec[r] { \ec[h-1]{ \ev{ x_{\ast} - y_{r, h} , g_{r, h} } } } \nonumber\\
                                                  &= \ec[r]{ \ev{ x_{\ast} - y_{r, h} , \overline{g}_{r, h} } } \nonumber\\
    &=\frac{1}{M} \sum_{m=1}^M \ec[r]{ \ev{ x_{\ast} - y_{r, h} , \overline{g}_{m, r, h} } } \nonumber\\
                                                  &= \frac{1}{M} \sum_{m=1}^M \ec[r]{ \ev{ x_{\ast} - y_{m, r, h} + y_{m, r, h} - y_{r, h} , \overline{g}_{m, r, h} } } \nonumber\\
                                                  &= \frac{1}{M} \sum_{m=1}^M \ec[r]{\ev{ x_{\ast} - y_{m, r, h} , \overline{g}_{m, r, h} }} + \frac{1}{M} \sum_{m=1}^M \ec[r]{ \ev{ y_{m, r, h} - y_{r, h} , \overline{g}_{m, r, h} } }. \nonumber
  \end{align}

Using Young's inequality for the second term,
\begin{align}
  &\ec[r]{\ev{x_* - y_{r,h}, g_{r,h}}} \\
  &\quad \leq \frac{1}{M}\sum_{m=1}^M\ec[r]{\ev{x_* - y_{m,r,h}, \overline{g}_{m,r,h}}} + \frac{1}{M}\sum_{m=1}^M\ec[r]{\frac{\sqn{y_{m,r,h} - y_{r,h}}}{2\alpha} + \frac{\alpha}{2}\sqn{\overline{g}_{m,r,h}}} \nonumber \\
&\quad = \frac{1}{M}\sum_{m=1}^M\ec[r]{\ev{x_* - y_{m,r,h}, \overline{g}_{m,r,h}}} + \frac{V_{r,h}}{2\alpha} + \frac{\alpha}{2M}\sum_{m=1}^M\ecn[r]{\overline{g}_{m,r,h}},
\label{eq:8}
\end{align}
where $V_{r, h} = \frac{1}{M} \sum_{m=1}^M \ecn[r]{y_{m, r, h} - y_{r, h}}$ by definition. By the convexity of $f$,
  \begin{align}
    \ev{ x_{\ast} - y_{m, r, h} , \overline{g}_{m, r, h} } &= \ev{ x_{\ast} - y_{m, r, h} , \nabla f(y_{m, r, h}) } \nonumber\\
                                                           &\leq f(x_{\ast}) - f(y_{m, r, h}) \nonumber\\
\label{eq:9}
    &= - \left( f(y_{m, r, h}) - f(x_{\ast}) \right).
  \end{align}

For the variance term, when $\eta \leq \frac{1}{L}$ we use \Cref{lem:Vr-bound}
\begin{align}
  V_{r,h} &= \frac{1}{M} \sum_{m=1}^M \ecn[r]{y_{m, r, h} - y_{r, h}} \nonumber
\\           &\leq \frac{1}{M} \sum_{m=1}^M \frac{1}{M} \sum_{s=1}^M \ecn[r]{y_{m, r, h} - y_{s, r, h}}  \nonumber \\
          &= \frac{1}{M^2}  \sum_{m=1}^M \sum_{s=1}^M \ecn[r]{y_{m,r , h} - y_{s, r, h}} \nonumber  \\
          &\leq 2\eta^2\sigma^2h  \leq 2\eta^2\sigma^2H.
\label{eq:10}
\end{align}

By smoothness,
\begin{align}
\label{eq:11}
\sqn{\overline{g}_{m,r,h}} = \sqn{\nabla f(y_{m,r,h})} \leq 2L(f(y_{m,r,h}) - f(x_*)).
\end{align}

  Plugging \cref{eq:9,eq:10,eq:11} back into \cref{eq:8} we get
  \begin{align}
\label{eq:13}
    \ec[r]{\ev{ x_{\ast} - y_{r, h} , g_{r, h} }} \leq \frac{- (1 - \alpha L)}{M} \sum_{m=1}^M (\ec[r]{ f(y_{m, r, h})}  - f(x_{\ast})) + \frac{\eta^2 \sigma^2 H}{\alpha}.
  \end{align}

Substituting \eqref{eq:13} back into our main recursion (\Cref{eq:4}),
  \begin{align}
   \begin{split}
    \ecn[r]{x_{r+1} - x_{\ast}} &\leq \sqn{x_r - x_{\ast}} - \frac{2 \eta \gamma (1-\alpha L)}{M} \sum_{h=0}^{H-1} \sum_{m=1}^M (\ec[r]{f(y_{m, r, h})}  - f(x_{\ast})) + \frac{2 \eta^3 \gamma \sigma^2 H^2}{\alpha} \\
    &+ \frac{\gamma \eta^2 H \sigma^2}{M}
      + \gamma \eta^2 \sum_{h=0}^{H-1} \ecn[r]{\overline{g}_{r, h}} + \eta^2 \gamma (\gamma - 1) \ecn[r]{\sum_{h=0}^{H-1} g_{r, h}}.
   \end{split}
   \label{eq:14}
  \end{align}

  We now have two cases. \textbf{Case 1}. If $\gamma \geq 1$, then we have by \Cref{lem:var-generic} and Jensen's inequality applied to $\sqn{\cdot}$,
  \begin{align}
    \ecn[r]{\sum_{h=0}^{H-1} g_{r, h}} &= \ecn[r]{\sum_{h=0}^{H-1} (g_{r, h} - \ec[r]{g_{r, h}})} + \sqn{\sum_{h=0}^{H-1} (\ec[r]{g_{r, h}})} \nonumber\\
                                       &= \ecn[r]{\sum_{h=0}^{H-1} (g_{r, h} - \ec[r]{g_{r, h}})} + \sqn{\sum_{h=0}^{H-1} (\ec[r]{ \ec[r, h-1]{g_{r, h}}})} \nonumber\\
                                       &= \ecn[r]{\sum_{h=0}^{H-1} (g_{r, h} - \ec[r]{g_{r, h}})} + \sqn{\sum_{h=0}^{H-1} \ec[r]{\overline{g}_{r, h}}} \nonumber\\
\nonumber
                                       &\leq \frac{\sigma^2 H}{M} + \ecn[r]{ \sum_{h=0}^{H-1} \overline{g}_{r, h} } \\
                                       &\leq \frac{\sigma^2 H}{M} + H \sum_{h=0}^{H-1} \ecn[r]{\overline{g}_{r, h}}.
\label{eq:1}
  \end{align}
  Using Jensen's inequality and smoothness we have
  \begin{align}
    \ecn[r]{\overline{g}_{r, h}} &= \ecn[r]{\frac{1}{M} \sum_{m=1}^M \nabla f(y_{m, r, h})} \nonumber\\
                                 &\leq \frac{1}{M} \sum_{m=1}^M \ecn[r]{\nabla f(y_{m, r, h})} \nonumber\\
\label{eq:7}
                                 &\leq \frac{2L}{M} \sum_{m=1}^M \ec[r]{ f(y_{m, r, h}) - f(x_{\ast})}.
  \end{align}
  Using \cref{eq:1,eq:7} into \cref{eq:14} we get
  \begin{align}
   \begin{split}
     &\ecn[r]{x_{r+1} - x_{\ast}} \leq \sqn{x_r - x_{\ast}} \\
     &\quad - \frac{2 \eta \gamma (1-\alpha L) - 2L \gamma \eta^2 (1+ (\gamma-1) H)}{M} \sum_{h=0}^{H-1} \sum_{m=1}^M (\ec[r]{f(y_{m, r, h})} - f(x_{\ast})) + \frac{2 \eta^3 \gamma \sigma^2 H^2}{\alpha} \nonumber\\
     &\qquad + \frac{\gamma^2 \eta^2 H \sigma^2}{M}.
   \end{split} \nonumber\\
    \begin{split}
      &= \sqn{x_r - x_{\ast}} - \frac{2 \eta \gamma \left[ 1-\alpha L - L \eta (1+(\gamma-1)H) \right]}{M} \sum_{h=0}^{H-1} \sum_{m=1}^M  (\ec[r]{f(y_{m, r, h})} - f(x_{\ast})) \nonumber\\
      &\qquad + \frac{2 \eta^3 \gamma \sigma^2 H^2}{\alpha} + \frac{\gamma^2 \eta^2 H \sigma^2}{M}.
    \end{split} \nonumber\\
\label{eq:16}
    \begin{split}
    &= \sqn{x_r - x_{\ast}} - 2 \eta \gamma H (1-\alpha L - L \eta (1+(\gamma-1) H)) \ec[r]{\hat{\delta}_{r+1}} + \frac{2 \eta^3 \gamma \sigma^2 H^2}{\alpha} + \frac{\eta^2 \gamma^2 H \sigma^2}{M},
    \end{split}
  \end{align}
  where in the last line we defined
  \begin{align}
\label{eq:17}
    \hat{\delta}_{r+1} &= \frac{1}{MH} \sum_{h=0}^{H-1} \sum_{m=1}^M \left( f(y_{m, r, h}) - f(x_{\ast}) \right)
  \end{align}

  \textbf{Case 2.} If $\gamma \leq 1$, then we can simply drop the last term in \cref{eq:14} and use \Cref{eq:11} to get
  \begin{align}
   \begin{split}
     \ecn[r]{x_{r+1} - x_{\ast}} &\leq \sqn{x_r - x_{\ast}} - \frac{2\eta \gamma (1-\alpha L-\eta L)}{M} \sum_{h=0}^{H-1} \sum_{m=1}^M \left( \ec[r]{f(y_{m, r, h})} - f(x_{\ast}) \right) \nonumber\\
     &\qquad + \frac{2 \eta^3 \gamma \sigma^2 H^2}{\alpha} + \frac{\gamma \eta^2 H \sigma^2}{M}
   \end{split} \nonumber\\
\label{eq:18}
    \begin{split}
    &= \sqn{x_r - x_{\ast}} - 2 \eta \gamma H (1-\alpha L - \eta L) \ec[r]{\hat{\delta}_{r+1}} + \frac{2 \eta^3 \gamma \sigma^2 H^2}{\alpha} + \frac{\gamma \eta^2 H \sigma^2}{M},
    \end{split}
  \end{align}
  where in \cref{eq:18} we again used the definition in \cref{eq:17}. Looking at both \cref{eq:16,eq:18} and taking the maximum we get that for \emph{any} $\gamma$,
    \begin{align*}
   \begin{split}
     \ecn[r]{x_{r+1} - x_{\ast}} &\leq \sqn{x_r - x_{\ast}} - 2 \eta \gamma H (1-\alpha L - \eta L (1+(\gamma-1)_+ H)) \ec[r]{\hat{\delta}_{r+1}} \\
     &\qquad + \frac{2 \eta^3 \gamma \sigma^2 H^2}{\alpha} + \frac{\eta^2 \max\{\gamma^2, \gamma\} H \sigma^2}{M},
   \end{split}
  \end{align*}
   where $(x)_+ = \max(x, 0)$ is the ReLU function. Putting $\alpha = \frac{1}{2L}$ we get
  \begin{align*}
    \ecn[r]{x_{r+1} - x_{\ast}} &\leq \sqn{x_r - x_{\ast}} - \eta \gamma H  (1-2 \eta L (1+(\gamma-1)_+ H)) \ec[r]{\hat{\delta}_{r+1}} \\
    &\qquad + 4 L \eta^3 \gamma \sigma^2 H^2 + \frac{\eta^2 \max\{\gamma^2, \gamma\} H \sigma^2}{M}.
  \end{align*}
  Under the requirement that the stepsizes $\eta, \gamma$ satisfy
  \begin{align*}
  \eta L (1+(\gamma-1)_+ H) &\leq \frac{1}{4},
  \end{align*}
  we obtain our recursion
  \begin{align*}
  \ecn[r]{x_{r+1} - x_{\ast}} &\leq \sqn{x_r - x_{\ast}} - \frac{\eta \gamma H}{2} \ec[r]{\hat{\delta}_{r+1}} + 4 L \eta^3 \gamma \sigma^2 H^2 + \frac{\eta^2 \max\{\gamma^2, \gamma\} H \sigma^2}{M}.
  \end{align*}
  Taking unconditional expectations and rearranging we obtain,
  \begin{align*}
   \ec{\hat{\delta}_{r+1}} \leq \frac{2}{\gamma \eta H} \left [ \ecn{x_r - x_{\ast}} - \ecn{x_{r+1} - x_{\ast}} \right ] + 8 L \eta^2 \sigma^2 H + \frac{2 \eta \max(\gamma, 1) \sigma^2}{M}.
  \end{align*}
  Summing up both sides as $r$ varies from $0$ to $R-1$ and dividing by $1/R$ we get
  \begin{align*}
   \frac{1}{R} \sum_{r=0}^{R-1} \ec{\hat{\delta}_{r+1}} &\leq \frac{2}{\gamma \eta R H} \left [ \sqn{x_0 - x_{\ast}} - \ecn{x_{R} - x_{\ast}} \right ] + 8 L \eta^2 \sigma^2 H + \frac{2 \eta \max(\gamma, 1) \sigma^2}{M}.
  \end{align*}
  Dropping the negative term and using Jensen's inequality gives
  \begin{align*}
      &\ec{ f \left ( \frac{1}{MRH} \sum_{r=0}^{R-1} \sum_{h=0}^{H-1} \sum_{m=1}^M f(y_{m, r, h}) \right) } - f(x_{\ast}) \leq \frac{1}{R} \sum_{r=0}^{R-1} \ec{\hat{\delta}_{r+1}} \\
      &\qquad\qquad \leq \frac{2 \sqn{x_0 - x_{\ast}}}{\gamma \eta R H} + 8 L \eta^2 \sigma^2 H + \frac{2 \eta \max(\gamma, 1) \sigma^2}{M},
  \end{align*}
  and this is the statement of our theorem.
\end{proof}

\subsection{Convergence guarantees with momentum}\label{sec:proof-momentum}

\begin{proof}[Proof of Theorem~\ref{thm:momentum}]
We analyze the momentum variant of Local SGD:
\begin{align*}
  x_{r+1} &= x_r - \eta \gamma \left( \sum_{h=0}^{H-1} g_{r, h} \right) + \mu (x_r - x_{r-1}).
\end{align*}

Define
\begin{align*}
z_r &= x_r + \frac{\mu}{1-\mu} (x_r - x_{r-1}).
\end{align*}
Then
\begin{align*}
  z_{r+1} &= z_r - \frac{\eta \gamma}{1-\mu} \sum_{h=0}^{H-1} g_{r, h}.
\end{align*}

We have
\begin{align}
\label{eq:new-momentum-1}
\begin{split}
      &\sqn{z_{r+1} - x_{\ast}} = \sqn{z_r - x_{\ast}} + \frac{\eta^2 \gamma^2}{(1-\mu)^2} \sqn{\sum_{h=0}^{H-1} g_{r, h}} - \frac{2 \eta \gamma}{1-\mu} \sum_{h=0}^{H-1}  \ev{ z_r - x_{\ast} , g_{r, h} } \\
      &= \sqn{z_r - x_{\ast}} + \frac{\eta^2 \gamma^2}{(1-\mu)^2} \sqn{\sum_{h=0}^{H-1} g_{r, h}} - \frac{2 \eta \gamma}{1-\mu} \sum_{h=0}^{H-1}  \ev{ x_r - x_{\ast} , g_{r, h} } \\
      &\qquad- \frac{2 \eta \gamma \mu}{1-\mu} \sum_{h=0}^{H-1} \ev{ x_r - x_{r-1} , g_{r, h} }.
\end{split}
\end{align}
Following the same proof as \Cref{thm:gen-loc-sgd}, we can bound (in expectation)
\begin{align}
\label{eq:new-momentum-2}
\begin{split}
  - &\frac{2 \eta \gamma}{1-\mu} \sum_{h=0}^{H-1} \ec[r]{\ev{ x_r - x_{\ast} , g_{r, h} }} + \frac{\eta^2 \gamma^2}{(1-\mu)^2} \ec[r]{\sqn{\sum_{h=0}^{H-1} g_{r, h}}} \leq - \frac{\eta \gamma H}{2 (1-\mu)} \ec[r]{\hat{\delta}_{r+1}} \\
  &\qquad\qquad + 4 L \eta^3 \frac{\gamma}{1-\mu} \sigma^2 H^2 + \frac{\eta^2 H \sigma^2}{M} \max\left( \left(\frac{\gamma}{1-\mu}\right)^2, \frac{\gamma}{1-\mu} \right),
\end{split}
\end{align}
because the local optimization procedure is the same-- the same analysis holds line-by-line, only replacing $\gamma$ by $\frac{\gamma}{1-\mu}$, and requiring instead that
\begin{align}
\label{eq:new-momentum-3}
\eta L \left(1 + \left(\frac{\gamma}{1-\mu} -1\right)_+ H\right) \leq \frac{1}{4}.
\end{align}
Using \cref{eq:new-momentum-2} in \cref{eq:new-momentum-1} (after taking expectation in the latter) we obtain
\begin{align}
 \begin{split}
  &\ecn[r]{z_{r+1} - x_{\ast}} \leq \sqn{z_r - x_{\ast}} - \frac{\eta \gamma H}{2 (1-\mu)} \ec[r]{\hat{\delta}_{r+1}} + 4 L \eta^3 \frac{\gamma \sigma^2 H^2}{1-\mu}  \\
  &\qquad + \frac{\eta^2 H \sigma^2}{M} \max \left( \left( \frac{\gamma}{1-\mu}  \right)^2, \frac{\gamma}{1-\mu}  \right) - \frac{2 \eta \gamma \mu}{1-\mu} \sum_{h=0}^{H-1} \ev{ x_r - x_{r-1} , \overline{g}_{r, h}}.
 \end{split}
  \label{eq:new-momentum-6}
\end{align}
In the following, we use the shorthand $G_r \eqdef \sum_{h=0}^{H-1} g_{r, h}$. We now proceed to bound $\sum_{h=0}^{H-1} \ev{x_{r-1} - x_r, g_{r, h}} = \ev{x_{r-1} - x_r, G_r}$ without using the bounded iterates assumption. We note that by definition:
\[x_r - x_{r-1} = -\eta \gamma G_{r-1} + \mu(x_{r-1} - x_{r-2}).\]
Expanding this out recursively, we get the following formula:
\[x_r - x_{r-1} = -\eta \gamma \sum_{s=0}^{r-1} \mu^{r-1-s} G_s.\]
For our analysis, we'll bound the inner product
\begin{align*}
\ev{x_{r-1} - x_r, G_r} &= \ev{\eta \gamma \sum_{s=0}^{r-1} \mu^{r-1-s} G_s, G_r} \\
&= \eta \gamma \sum_{s=0}^{r-1} \mu^{r-1-s} \ev{G_s, G_r}
\end{align*}
We will actually bound the sum of the momentum terms over $r$, i.e. $\sum_r \ev{x_{r-1} - x_r, G_r}$. We have
\begin{align*}
\sum_r \ev{x_{r-1} - x_r, G_r} &= \frac{\eta \gamma}{\mu} \sum_r \sum_{s < r} \ev{\mu^{r-s} G_s, G_r} \\
&= \frac{\eta \gamma}{2\mu} \left[\sum_r \sum_s \ev{\mu^{|r-s|} G_s, G_r}- \sum_r \|G_r\|^2\right].
\end{align*}
To bound the first term above, let $A$ be the $R \times R$ matrix whose $(r, s)$th entry equals $\mu^{|r-s|}$, and let $\Gamma = [G_1 | G_2 | \ldots | G_R]$. Then
\[\sum_r \sum_s \ev{\mu^{|r-s|} G_s, G_r} = \text{Tr}(\Gamma A \Gamma^\top).\]
We now apply the Gershgorin circle theorem to bound this sum, observe that largest sum of absolute values of entries in a row satisfy
\begin{align*}
1 + 2\sum_{r=1}^{(R-1)/2} \mu^r = 1 + 2\mu\frac{1-\mu^{(R-1)/2}}{1-\mu} = \frac{1 + \mu - 2\mu^{(R+1)/2}}{1-\mu} \leq \frac{1+\mu}{1-\mu}.
\end{align*}
Then, we have
\[\text{Tr}(\Gamma A \Gamma^\top) \leq \frac{1+\mu}{1-\mu}\text{Tr}(\Gamma \Gamma^\top) = \frac{1+\mu}{1-\mu} \sum_r \sqn{G_r}. \]
Therefore, taking expectations we have
\begin{align}
  - \frac{2 \eta \gamma \mu}{1-\mu} &\sum_{r=0}^{R-1} \sum_{h=0}^{H-1} \ec{\ev{ x_r - x_{r-1} , g_{r, h} }} = \frac{2 \eta \gamma \mu}{1-\mu} \sum_{r=0}^{R-1} \ec{\ev{ x_{r-1} - x_r , G_r }}\nonumber\\
\label{eq:new-momentum-5}
                                    &\leq \frac{2 \eta \gamma \mu}{1-\mu} \frac{\eta \gamma}{1-\mu} \sum_{r=0}^{R-1} \ecn{\sum_{h=0}^{H-1} g_{r, h}}.
\end{align}
Using \Cref{lem:var-generic} we have
\begin{align*}
  \ecn{\sum_{h=0}^{H-1} g_{r, h}} &\leq \frac{\sigma^2 H}{M} + \ecn{\sum_{h=0}^{H-1} \overline{g}_{r, h}} \\
                                  &\leq \frac{\sigma^2 H}{M} + H \sum_{h=0}^{H-1}  \ecn{\overline{g}_{r, h}} \\
  &\leq \frac{\sigma^2 H}{M} + 2LH^2 \ec{\hat{\delta}_{r+1}},
\end{align*}
where in the last line we used Jensen's inequality and smoothness. Using this result in \cref{eq:new-momentum-5} we get
\begin{align}
\label{eq:new-momentum-7}
  - \frac{2 \eta \gamma \mu}{1-\mu} &\sum_{r=0}^{R-1} \sum_{h=0}^{H-1} \ec{\ev{ x_r - x_{r-1} , g_{r, h} }} \leq \frac{\eta \gamma}{2(1-\mu)} \frac{4 \eta \gamma \mu}{1-\mu} \left[ \frac{\sigma^2 R H}{M} + 2 L H^2 \sum_{r=0}^{R-1} \ec{\hat{\delta}_{r+1}} \right].
\end{align}
Rearranging and summing up \cref{eq:new-momentum-6} then using \cref{eq:new-momentum-7} we have
\begin{align*}
 \begin{split}
  &\ecn{z_{R} - x_{\ast}} \leq \sqn{z_0 - x_{\ast}} - \frac{\eta \gamma H}{2 (1-\mu)} \left [ 1 - \frac{8 \eta \gamma \mu L H}{1-\mu} \right] \sum_{r=0}^{R-1} \ec{\hat{\delta}_{r+1}} \\
  &+ 4 L \eta^3 \frac{\gamma \sigma^2 H^2}{1-\mu} R + \frac{\eta^2 H \sigma^2}{M} \max \left( \left( \frac{\gamma}{1-\mu}  \right)^2, \frac{\gamma}{1-\mu}  \right) R + \frac{\eta \gamma H}{1-\mu} \frac{2 \eta \gamma \mu}{1-\mu} \frac{\sigma^2 R}{M}.
 \end{split}
\end{align*}
Observe that under the condition
\begin{align*}
     \frac{\eta \gamma \mu L H}{1-\mu} \leq \frac{1}{16}
\end{align*}
the last inequality becomes
\begin{align*}
\begin{split}
  &\ecn{z_{R} - x_{\ast}} \leq \sqn{z_0 - x_{\ast}} - \frac{\eta \gamma H}{4 (1-\mu)} \sum_{r=0}^{R-1} \ec{\hat{\delta}_{r+1}} \\
  &+ 4 L \eta^3 \frac{\gamma \sigma^2 H^2}{1-\mu} R + \frac{\eta^2 H \sigma^2}{M} \max \left( \left( \frac{\gamma}{1-\mu}  \right)^2, \frac{\gamma}{1-\mu}  \right) R + \frac{\eta \gamma H}{1-\mu} \frac{2 \eta \gamma \mu}{1-\mu} \frac{\sigma^2 R}{M}.
 \end{split}
\end{align*}
Continuing the proof and rearranging we get
\begin{align*}
    \frac{1}{R} \sum_{r=0}^{R-1} \ec{\hat{\delta}_{r+1}} \leq \frac{ 4 (1-\mu) \sqn{z_0 - x_{\ast}}}{\eta \gamma H R} + 16 L \eta^2 \sigma^2 H + \frac{4 \eta \sigma^2}{M} \max \left ( \frac{\gamma}{1-\mu}, 1 \right) + \frac{8 \eta \gamma \mu}{1-\mu} \frac{\sigma^2}{M}.
\end{align*}
It remains to use Jensen's inequality.
\end{proof}

\subsection{Acceleration proofs}\label{sec:acc-proofs}

We recall the algorithm under analysis as
\begin{align*}
  y_{m, r, 0} &= x_r \text{ for $m = 1, \ldots, M$ } \\
  y_{m, r, h+1} &= y_{m, r, h} - \eta g_{m, r, h} \text { for $h=0,1,\ldots, H-1$ } \\
  u_{r+1} &= x_r - \eta \sum_{h=0}^{H-1} g_{r, h} \\
  z_{r+1} &= z_r - \gamma_r \eta \sum_{h=0}^{H-1} g_{r, h} \\
  x_{r+1} &= (1-\tau_{r+1}) u_{r+1} + \tau_{r+1} z_{r+1},
\end{align*}
where $g_{r, h} = \frac{1}{M} \sum_{m=1}^M g_{m, r, h}$, $\gamma_r = \frac{\gamma (r+1)}{2}$, and $\tau_r = \frac{2}{r+2}$. Note that under the above, $u_{m, r, h} = y_{m, r, h}$ and $u_{r, h} = \frac{1}{M} \sum_{m=1}^M u_{m, r, h}$. We first derive two intermediate lemmas, then proceed to the main proof.

\begin{lemma}\label{lem:loc-sgd-descent}
Suppose that the local stepsize $\eta$ satisfies $\eta \leq \frac{1}{2L}$. Then, for all $h \in [H-1]$ and $r$, we have
\begin{align*}
H f(u_{r+1}) \leq \frac{1}{M} \sum_{m, h<H} \left [ \ec{f(y_{m, r, h})} - \frac{\eta}{4} \sum_{h^{\prime}=h}^{H-1} \ecn{\bar{g}_{m, r, h}} \right ] + \frac{L\eta^2 \sigma^2 H^2}{2 M} + \frac{\eta^3 \sigma^2 H^3}{2}.
\end{align*}
\end{lemma}
\begin{proof}
By smoothness,
\begin{align*}
    f(u_{r, h+1}) &\leq f(u_{r, h}) + \ev{\nabla f(u_{r, h}), u_{r, h+1} - u_{r, h}} + \frac{L}{2} \sqn{u_{r, h+1} - u_{r, h}} \\
    &= f(u_{r, h}) - \eta \ev{\nabla f(u_{r, h}), g_{r, h}} + \frac{L \eta^2}{2} \sqn{g_{r, h}}.
\end{align*}
Taking conditional expectation we have
\begin{align*}
    &\ec[h]{f(u_{r, h+1})} \leq f(u_{r, h}) - \frac{\eta}{M} \sum_{m=1}^M \ev{\nabla f(u_{r, h}), \nabla f(u_{m, r, h})} + \frac{L \eta^2 \sigma^2}{2 M} + \frac{L \eta^2}{2} \sqn{\frac{1}{M} \sum_{m=1}^M \nabla f(u_{m, r, h})} \\
    \begin{split}
    &\leq f(u_{r, h}) - \frac{\eta}{2M} \sum_{m=1}^M \left [ \sqn{\nabla f(u_{r, h})} + \sqn{\nabla f(u_{m, r, h})} - \sqn{\nabla f(u_{r, h}) - \nabla f(u_{m, r, h})}  \right] + \frac{L \eta^2 \sigma^2}{2M} \\
    &\qquad + \frac{L \eta^2}{2 M} \sum_{m=1}^M \sqn{\nabla f(u_{m, r, h})}
    \end{split} \\
    \begin{split}
        &= f(u_{r, h}) - \frac{\eta}{2} \sqn{\nabla f(u_{r, h})} - \frac{\eta (1-L \eta)}{2 M} \sum_{m=1}^M \sqn{\nabla f(u_{m, r, h})} + \frac{\eta}{2} V_{r, h} + \frac{L \eta^2 \sigma^2}{2M},
    \end{split}
\end{align*}
where $V_{r, h} = \frac{1}{M} \sum_{m=1}^M \sqn{\nabla f(u_{r, h}) - \nabla f(u_{m, r, h})} \leq \frac{L^2}{M} \sum_{m=1}^M \sqn{u_{r, h}-u_{m,r,h}}$. Taking unconditional expectation, dropping the $\sqn{\nabla f(u_{r, h})}$ term and using Lemma~\ref{lem:Vr-bound} we have
\begin{align*}
   \ec{f(u_{r, h+1})} \leq \ec{f(u_{r, h})} - \frac{\eta (1 - L \eta)}{2 M} \sum_{m=1}^M \ecn{\nabla f(u_{m, r, h})} + \frac{L \eta^2 \sigma^2}{2 M} + \eta^3 L^2 \sigma^2 h.
\end{align*}
Observe that in the current scheme, $\bar{g}_{m, r, h} = \nabla f(u_{m, r, h})$. Suppose that $1-L \eta \geq \frac{1}{2}$, using this and telescoping yields
\begin{align*}
   \ec{f(u_{r+1})} \leq \ec{f(u_{r, h})} - \frac{\eta}{4 M} \sum_{m=1}^M \sum_{h^{\prime}=h}^{H-1} \ecn{\bar{g}_{m, r, h^{\prime}}} + \frac{L \eta^2 \sigma^2 (H-h)}{2 M} + \eta^3 L^2 \sigma^2 \sum_{h^{\prime}=h}^{H-1} h^{\prime}.
\end{align*}
Using Jensen's inequality on $u_{r, h} = \frac{1}{M} \sum_{m=1}^M u_{m, r, h} = \frac{1}{M} \sum_{m=1}^M y_{m, r, h}$ we obtain
\begin{align*}
    \ec{f(u_{r+1})} \leq \frac{1}{M} \sum_{m=1}^M \left [ f(y_{m, r, h}) - \frac{\eta}{4} \sum_{h^{\prime}=h}^{H-1} \ecn{\bar{g}_{m, r, h}} \right] + \frac{L \eta^2 \sigma^2 (H-h)}{2 M} + L^2 \eta^3 \sigma^2 \sum_{h^{\prime}=h}^{H-1} h^{\prime}.
\end{align*}
Summing up both sides as $h$ varies from $0$ to $H-1$ we get
\begin{align*}
   H f(u_{r+1}) \leq \frac{1}{M} \sum_{m, h<H} \left [ \ec{f(y_{m, r, h})} - \frac{\eta}{4} \sum_{h^{\prime}=h}^{H-1} \ecn{\bar{g}_{m, r, h}} \right ] + \frac{L\eta^2 \sigma^2 H^2}{2 M} + \frac{\eta^3 L^2 \sigma^2 H^3}{2}.
\end{align*}
\end{proof}

Define $G_r = \sum_{h=0}^{H-1} g_{r,h}$ and $\bar{G}_r = \sum_{h=0}^{H-1} \bar{g}_{r,h}$. The following lemma characterizes the evolution of the momentum sequence $z_1, z_2, \ldots$.
\begin{lemma}[Momentum sequence bound]\label{lem:local-nesterov-z-update}
For any $r \geq 0$, the momentum sequence satisfies:
\begin{align*}
  \ec[r]{\sqn{z_{r+1} - x_{\ast}}} &= \sqn{z_r - x_{\ast}} + \gamma_r^2 \eta^2\ec[r]{\sqn{\bar{G}_r}} + \frac{\gamma_r^2 \eta^2\sigma^2}{M}  - \gamma_r\eta \ev{ z_r - x_{\ast} , \ec[r]{\bar{G}_r}}.
  \end{align*}
\end{lemma}
\begin{proof}
  Expanding the square,
  \begin{align*}
  \sqn{z_{r+1} - x_{\ast}} &= \sqn{z_r - x_{\ast}} + \gamma_r^2 \eta^2 \sqn{G_r} - 2 \gamma_r\eta \ev{ z_r - x_{\ast} , G_r}.
  \end{align*}
  Taking expectations and using Lemma~\ref{lem:var-generic},
  \begin{align*}
  \ec[r]{\sqn{z_{r+1} - x_{\ast}}} &= \sqn{z_r - x_{\ast}} + \gamma_r^2 \eta^2\ec[r]{\sqn{\bar{G}_r}} + \frac{\gamma_r^2 \eta^2\sigma^2H}{M}  - \gamma_r\eta \ev{ z_r - x_{\ast} , \ec[r]{\bar{G}_r}}.
  \end{align*}
\end{proof}

\begin{proof}[Proof of \Cref{thm:accelerated}]

Define the potential function
$$\Phi_r = r(r+1)H(f(u_r) - f(x_*)) + \frac{2}{\gamma \eta}\|z_r - x_*\|^2.$$

Using Lemma~\ref{lem:loc-sgd-descent} and Lemma~\ref{lem:local-nesterov-z-update}, we have
\begin{align*}
& \mathbb{E}_r[\Phi_{r+1}] - \Phi_r \\
& = (r+1) (r+2) H\left( \mathbb{E}_r[f(u_{r+1})] - f(x_{*}) \right) - r (r+1) H \left( f(u_r) - f(x_{*}) \right) \\
&\qquad + \frac{2}{\gamma\eta} \left[ \mathbb{E}_r[\|z_{r+1} - x_{*}\|^2] - \|z_r - x_{*}\|^2 \right] \\
&\leq (r+1)(r+2)\left[\frac{1}{M}\sum_{m, h<H} \left[(\mathbb{E}_r[f(y_{m, r, h})] - f(x_*)) - \frac{\eta}{4}\sum_{h' = h}^{H-1} \mathbb{E}_r[\|\bar{g}_{m,r,h'}\|^2]\right] + \frac{L \eta^2 \sigma^2H^2}{2M} + \frac{\eta^3 L^2 \sigma^2 H^3}{2}\right] \\
& \qquad - r (r+1) H \left( f(u_r) - f(x_{*}) \right) \\
&\qquad +  \frac{\gamma \eta(r+1)^2}{2}\mathbb{E}_r[\|\bar{G}_r\|^2] +  \frac{\gamma \eta \sigma^2(r+1)^2H}{2M}  - 2(r+1)\langle z_r - x_{*} , \mathbb{E}_r[\bar{G}_r]\rangle \\
&= \underbrace{\frac{1}{M}\sum_{m, h<H} \left[2(r+1)(\mathbb{E}_r[f(y_{m, r, h})] - f(x_*)) + r(r+1)(\mathbb{E}_r[f(y_{m, r, h})] - f(u_r)) - 2(r+1)\langle z_r - x_{*} , \mathbb{E}_r[\bar{g}_{m,r,h}]\rangle\right]}_{=: A} \\
& \qquad \underbrace{-\frac{(r+1)(r+2)\eta}{4M}\sum_{m, h<H}\sum_{h' = h}^{H-1}\mathbb{E}_r[\|\bar{g}_{m,r,h'}\|^2] + \frac{\gamma \eta(r+1)^2}{2}\mathbb{E}_r[\|\bar{G}_r\|^2]}_{=:B} \\
& \qquad + \frac{(r+1)(r+2)L \eta^2 \sigma^2H^2}{2M} + \frac{(r+1)(r+2)\eta^3 L^2 \sigma^2 H^3}{2} + \frac{\gamma \eta \sigma^2(r+1)^2H}{2M}.
\end{align*}

Now, we bound the terms above separately. First, we bound $A$. Fix any $m, h < H$. We have, using convexity of $f$,
\begin{align*}
&2(r+1)(f(y_{m, r, h}) - f(x_*)) + r(r+1)(f(y_{m, r, h}) - f(u_r)) - 2(r+1)\langle z_r - x_{*} , \bar{g}_{m,r,h}\rangle \\
&\leq 2(r+1)\langle y_{m, r, h} - x_*,  \bar{g}_{m,r,h}\rangle + r(r+1)\langle y_{m, r, h} - u_r, \bar{g}_{m,r,h}\rangle - 2(r+1)\langle z_r - x_{*} , \bar{g}_{m,r,h}\rangle \\
&= \langle (r+1)(r+2)y_{m,r,h} - r(r+1)u_r - 2(r+1)z_r, \bar{g}_{m,r,h}\rangle \\
&= (r+1)(r+2)\langle y_{m,r,h} - x_r, \bar{g}_{m,r,h}\rangle\\
&= -\eta(r+1)(r+2) \sum_{h'<h} \langle  g_{m,r,h'}, \bar{g}_{m,r,h}\rangle.
\end{align*}

Hence,
\begin{align*}
&2(r+1)(\mathbb{E}_r[f(y_{m, r, h})] - f(x_*)) + r(r+1)(\mathbb{E}_r[f(y_{m, r, h})] - f(u_r)) - 2(r+1)\langle z_r - x_{*} , \mathbb{E}_r[\bar{g}_{m,r,h}]\rangle \\
&\leq -\eta(r+1)(r+2) \sum_{h'<h} \mathbb{E}_r[\langle  \bar{g}_{m,r,h'}, \bar{g}_{m,r,h}\rangle]
\end{align*}

Since $A$ equals the sum of the above over all $m, h < H$ and dividing by $M$, we get:
\begin{align*}
A &= \frac{-\eta(r+1)(r+2)}{M}\sum_{m, h<H} \sum_{h'<h} \mathbb{E}_r[\langle  \bar{g}_{m,r,h'}, \bar{g}_{m,r,h}\rangle] \\
&=\frac{-\eta(r+1)(r+2)}{2M}\sum_m\mathbb{E}_r\left[\|\sum_{h<H} \bar{g}_{m,r,h}\|^2 - \sum_{h<H}\|\bar{g}_{m,r,h}\|^2\right],
\end{align*}
where in the last line we used the algebraic identity that for any sequence of vectors $v_0, \ldots, v_{H-1}$,
\begin{align*}
\sum_{h<H} \sum_{s < h} \langle v_s, v_h \rangle = \frac{1}{2} \left [ \| \sum_{h<H} v_h\|^2 - \sum_{h < H} \|v_h\|^2 \right ].
\end{align*}

Next, we have
\begin{align*}
B &= -\frac{(r+1)(r+2)\eta}{4M}\sum_{m, h<H}\sum_{h' = h}^{H-1}\mathbb{E}_r[\|\bar{g}_{m,r,h'}\|^2] + \frac{\gamma \eta(r+1)^2}{2}\mathbb{E}_r[\|\bar{G}_r\|^2] \\
&\leq -\frac{(r+1)(r+2)\eta}{4M}\sum_{m}\sum_{h < H}\mathbb{E}_r[\|\bar{g}_{m,r,h}\|^2] + \frac{\gamma \eta(r+1)^2}{2M}\sum_m\mathbb{E}_r\left[\|\sum_{h<H} \bar{g}_{m,r,h}\|^2\right].
\end{align*}

Hence, we have
\begin{align*}
A + B &\leq \frac{-\eta(r+1)(r+2)}{2M}\sum_m\mathbb{E}_r\left[\|\sum_{h<H} \bar{g}_{m,r,h}\|^2 - \sum_{h<H}\|\bar{g}_{m,r,h}\|^2\right] \\
      &\qquad -\frac{(r+1)(r+2)\eta}{4M}\sum_{m}\sum_{h < H}\mathbb{E}_r[\|\bar{g}_{m,r,h}\|^2] + \frac{\gamma \eta(r+1)^2}{2M}\sum_m\mathbb{E}_r\left[\|\sum_{h<H} \bar{g}_{m,r,h}\|^2\right] \\
  \begin{split}
    &= \frac{\eta(r+1)}{2M}\sum_m\mathbb{E}_r\left[\|\sum_{h<H} \bar{g}_{m,r,h}\|^2\right]\left[\gamma(r+1) - (r+2)\right] \\
    &\qquad + \frac{\eta(r+1)(r+2)}{4M}\sum_{m}\sum_{h < H}\mathbb{E}_r[\|\bar{g}_{m,r,h}\|^2] \\
  \end{split}
&\leq 0
\end{align*}
since $\gamma \leq 1$ implies $\gamma(r+1) - (r+2) = (r+1)(\gamma - 1) - 1 \leq -1 < 0$, and the second term has a positive coefficient with a negative sign.

So overall, we have
\begin{align*}
\mathbb{E}_r[\Phi_{r+1}] - \Phi_r &\leq \frac{(r+1)(r+2)L \eta^2 \sigma^2H^2}{2M} + \frac{(r+1)(r+2)\eta^3 L^2 \sigma^2 H^3}{2} + \frac{\gamma \eta \sigma^2(r+1)^2H}{2M} \\
&\leq \frac{R^2 L\eta^2 \sigma^2 H^2}{2M} + \frac{R^2\eta^3 L^2 \sigma^2 H^3}{2} + \frac{\gamma \eta \sigma^2 R^2H}{2M}.
\end{align*}

Summing up from $r=0$ to $R-1$, and taking expectations, we get
$$\mathbb{E}[\Phi_R] - \Phi_0 \leq \frac{R^3 L\eta^2 \sigma^2 H^2}{2M} + \frac{R^3\eta^3 L^2 \sigma^2 H^3}{2} + \frac{\gamma \eta \sigma^2 R^3H}{2M}.$$

Thus,
\begin{align*}
  &R^2H(\mathbb{E}[f(u_R)] - f(x_*)) \\
  &\qquad \leq \mathbb{E}[\Phi_R] \leq \frac{2\|x_0 - x_*\|^2}{\gamma \eta} + \frac{R^3 L\eta^2 \sigma^2 H^2}{2M} + \frac{R^3\eta^3 L^2 \sigma^2 H^3}{2} + \frac{\gamma \eta \sigma^2 R^3H}{2M},
\end{align*}
$$$$

which implies that
$$\mathbb{E}[f(u_R)] - f(x_*) \leq \frac{2\|x_0 - x_*\|^2}{\gamma \eta R^2H} + \frac{RL\eta^2 \sigma^2 H}{2M} + \frac{RL^2 \eta^3 \sigma^2 H^2}{2} + \frac{\gamma \eta \sigma^2 R}{2M}.$$
\end{proof}

The proof of \Cref{corr:acceleration} is straightforward by substitution and is omitted for brevity.

\subsection{Data-dependent guarantees}\label{sec:data-dependent-guarantees}

\begin{lemma}\label{lem:vt-bound}
Let $f$ be a convex and $L$-smooth function. Suppose that we run SGD on $f$ on
$M$ parallel nodes as follows
\begin{align*}
y_{m, r, 0} &= x_r, \\
y_{m, r, h+1} &= y_{m, r, h} - \eta g_{m, r, h},
\end{align*}
where $m = 1, 2, \ldots, M$, $h = 0, 1, \ldots, H-1$, and $g_{1, r, h}, g_{2, r, h}, \ldots, g_{M, r, h}$ are i.i.d.
stochastic gradient estimates such that $\ec[r,h]{g_{m, r, h}} = \nabla f(y_{m, r, h})$, where $\ec[r,h]{\cdot}$ denotes expectation conditional on all information up to and including round $r$ and local step $h$, and
$\norm{g_{m, r, h} - \nabla f(y_{m, r, h})} \leq \sigma$. Define further
$y_{r, h} = \frac{1}{M} \sum_{m=1}^M y_{m, r, h}$. Let
$V_{r, h} = \frac{1}{M} \sum_{m=1}^M \sqn{y_{m, r, h} - y_{r, h}}$. Then for all
$\eta \leq \frac{1}{L}$ we have with probability at least $1-\delta$ that for all
$h=0, 1, \ldots, H$
\begin{align*}
V_{r, h} \leq 4104 \eta^2 \sigma^2 (h+1) \theta^2_{h-1, \delta},
\end{align*}
where $\theta_{h, \delta} = \log \frac{60 \log 6h}{\delta}$.
\end{lemma}
\begin{proof}
Define
\begin{align}
\label{eq:a-l2}
  \Lambda_{r, h+1} &= \frac{1}{M^2} \sum_{m=1}^M \sum_{s=1}^M \sqn{y_{m, r, h+1} - y_{s, r, h+1}}.
\end{align}
We will bound $\Lambda_{r, h}$ first, and then use it to bound $V_{r, h}$ later. We have
\begin{align*}
&y_{m, r, h+1} - y_{s, r, h+1} = y_{m, r, h} - \eta g_{m, r, h} - \left[ y_{s, r, h} - \eta g_{s, r, h} \right] \\
&\quad= y_{m, r, h} - \eta \nabla f(y_{m, r, h}) - \eta \left[ g_{m, r, h} - \nabla f(y_{m, r, h}) \right] - \left[ y_{s, r, h} - \eta \nabla f(y_{s, r, h}) - \eta \left[ g_{s, r, h} - \nabla f(y_{s, r, h}) \right] \right] \\
&\quad= \left[ y_{m, r, h} - \eta \nabla f(y_{m, r, h}) - \left[ y_{s, r, h} - \eta \nabla f(y_{s, r, h}) \right] \right] - \eta \left[ \left( g_{m, r, h} - g_{s, r, h} \right)  - \left[ \nabla f(y_{m, r, h}) - \nabla f(y_{s, r, h}) \right]  \right].
\end{align*}
Therefore
\begin{align}
\begin{split}
&\sqn{y_{m, r, h+1} - y_{s, r, h+1}} = \quad \| T_{\eta} (y_{m, r, h}) - T_{\eta} (y_{s, r, h})  \|^2 \\
&\qquad+ \eta^2 \| (g_{m, r, h} - g_{s, r, h}) - (\nabla f(y_{m, r, h}) - \nabla f(y_{s, r, h})) \|^2 \\
&\qquad - 2 \eta \ev{ T_{\eta} (y_{m, r, h}) - T_{\eta} (y_{s, r, h}) , (g_{m, r, h} - g_{s, r, h}) - (\nabla f(y_{m, r, h}) - \nabla f(y_{s, r, h})) }
\end{split}\label{eq:a-013}
\end{align}
We define $\rho_{m, r, h}$ as the stochastic gradient noise on node $m$ at round $r$, step $h$:
$\rho_{m, r, h} = g_{m, r, h} - \nabla f(y_{m, r, h})$. Then we can write \cref{eq:a-013} as
\begin{align}
\sqn{y_{m, r, h+1} - y_{s, r, h+1}} &= \quad \| T_{\eta} (y_{m, r, h}) - T_{\eta} (y_{s, r, h})  \|^2 + \eta^2 \| \rho_{m, r, h} - \rho_{s, r, h} \|^2 \nonumber\\
\label{eq:a-3}
&\qquad - 2 \eta \ev{ T_{\eta} (y_{m, r, h}) - T_{\eta} (y_{s, r, h}) , \rho_{m, r, h} - \rho_{s, r, h} }.
\end{align}
We now use the inequality $\sqn{a + b} \leq 2 \sqn{a} + 2 \sqn{b}$ to get
\begin{align*}
\begin{split}
\sqn{y_{m, r, h+1} - y_{s, r, h+1}} &\leq \| T_{\eta} (y_{m, r, h}) - T_{\eta} (y_{s, r, h})  \|^2 + 2 \eta^2 \| \rho_{m, r, h} \|^2 + 2 \eta^2 \sqn{\rho_{s, r, h}} \\
&\qquad - 2 \eta \ev{ T_{\eta} (y_{m, r, h}) - T_{\eta} (y_{s, r, h}) , \rho_{m, r, h} - \rho_{s, r, h} }.
\end{split}
\end{align*}
By Lemma~\ref{lem:contractivity}, we have
\begin{align*}
\sqn{y_{m, r, h+1} - y_{s, r, h+1}} &\leq \sqn{y_{m, r, h} - y_{s, r, h}} + 2 \eta^2 \sqn{\rho_{m, r, h}} + 2 \eta^2 \sqn{\rho_{s, r, h}} \\
&\quad - 2 \eta\ev{ T_{\eta} (y_{m, r, h}) - T_{\eta} (y_{s, r, h}) , \rho_{m, r, h} - \rho_{s, r, h} }.
\end{align*}
Now, we consider the inner product term, observe
\begin{align*}
&\ev{ T_{\eta} (y_{m, r, h}) - T_{\eta} (y_{s, r, h}) , \rho_{m, r, h} - \rho_{s, r, h} } \\
&\qquad = \ev{ T_{\eta} (y_{m, r, h}) - T_{\eta} (y_{r, h}) + T_{\eta} (y_{r, h}) - T_{\eta} (y_{s, r, h}), \rho_{m, r, h} - \rho_{s, r, h} } \\
&\qquad = \ev{ T_{\eta} (y_{m, r, h}) - T_{\eta} (y_{r, h}) , \rho_{m, r, h} - \rho_{s, r, h} } + \ev{ T_{\eta} (y_{r, h}) - T_{\eta} (y_{s, r, h}) , \rho_{m, r, h} - \rho_{s, r, h} } \\
&\qquad =\ev{ T_{\eta} (y_{m, r, h}) - T_{\eta} (y_{r, h}) , \rho_{m, r, h} - \rho_{s, r, h} } + \ev{ -(T_{\eta} (y_{s, r, h}) - T_{\eta} (y_{r, h})), -(\rho_{s, r, h} - \rho_{m, r, h}) } \\
&\qquad =\ev{ T_{\eta} (y_{m, r, h}) - T_{\eta} (y_{r, h})  , \rho_{m, r, h} - \rho_{s, r, h} } + \ev{ T_{\eta} (y_{s, r, h}) - T_{\eta} (y_{r, h}) , \rho_{s, r, h} - \rho_{m, r, h} }.
\end{align*}
Averaging with respect to $s$ and $m$
\begin{align}
  \frac{1}{M^2} \sum_{m=1}^M \sum_{s=1}^M &\ev{ T_{\eta} (y_{m, r, h}) - T_{\eta} (y_{r, h}) + T_{\eta} (y_{r, h}) - T_{\eta} (y_{s, r, h}), \rho_{m, r, h} - \rho_{s, r, h} } \nonumber\\
&\qquad = \frac{1}{M^2} \sum_{m=1}^M \sum_{s=1}^M  \ev{ T_{\eta} (y_{m, r, h}) - T_{\eta} (y_{r, h})  , \rho_{m, r, h} - \rho_{s, r, h} } \nonumber\\
&\qquad\qquad + \frac{1}{M^2} \sum_{m=1}^M \sum_{s=1}^M \ev{ T_{\eta} (y_{s, r, h}) - T_{\eta} (y_{r, h}) , \rho_{s, r, h} - \rho_{m, r, h} } \nonumber\\
&\qquad = \frac{2}{M^2} \sum_{m=1}^M \sum_{s=1}^M \ev{ T_{\eta} (y_{m, r, h}) - T_{\eta} (y_{r, h}) , \rho_{m, r, h} - \rho_{s, r, h} }.\label{eq:a-l1}
\end{align}
Averaging \cref{eq:a-3} with respect to $m$ and $s$ and using \cref{eq:a-l1} we get
\begin{align*}
  \begin{split}
  \frac{1}{M^2} \sum_{m=1}^M \sum_{s=1}^M \sqn{y_{m, r, h+1} - y_{s, r, h+1}} &\leq \frac{1}{M^2} \sum_{m=1}^{M} \sum_{s=1}^M \sqn{y_{m, r, h} - y_{s, r, h}} + \frac{4 \eta^2}{M} \sum_{m=1}^M \sqn{\rho_{m, r, h}} \\
  &\qquad - \frac{2\eta}{M^2} \sum_{m=1}^M \sum_{s=1}^M \ev{ T_{\eta} (y_{m, r, h}) - T_{\eta} (y_{r, h}) , \rho_{m, r, h} - \rho_{s, r, h} }.
  \end{split}
\end{align*}
Using $\Lambda_{r, h}$ as defined in~\Cref{eq:a-l2} we obtain the recursion
\begin{align*}
  \Lambda_{r, h+1} &\leq \Lambda_{r, h} + \frac{4 \eta^2}{M} \sum_{m=1}^M \sqn{\rho_{m, r, h}} - \frac{2 \eta}{M^2} \sum_{m=1}^M \sum_{s=1}^M \ev{ T_{\eta} (y_{m, r, h}) - T_{\eta} (y_{r, h}) , \rho_{m, r, h} - \rho_{s, r, h} }.
\end{align*}
Now observe that $\sqn{\rho_{m, r, h}} \leq \sigma^2$ by assumption, therefore
\begin{align*}
\Lambda_{r, h+1} &\leq \Lambda_{r, h} + 4 \eta^2 \sigma^2 - \frac{2 \eta}{M^2} \sum_{m=1}^M \sum_{s=1}^M \ev{ T_{\eta} (y_{m, r, h}) - T_{\eta} (y_{r, h}) , \rho_{m, r, h} - \rho_{s, r, h} }.
\end{align*}
Recursing the above inequality we get
\begin{align}
  \Lambda_{r, h} &\leq \Lambda_{r, 0} + 4 \eta^2 \sigma^2 h - \frac{2 \eta}{M^2} \sum_{k=0}^{h-1} \sum_{m=1}^M \sum_{s=1}^M\ev{ T_{\eta} (y_{m, r, k}) - T_{\eta} (y_{r, k}) , \rho_{m, r, k} - \rho_{s, r, k} } \nonumber\\
\label{eq:a-8}
  &= 4 \eta^2 \sigma^2 h - \frac{2 \eta}{M^2} \sum_{k=0}^{h-1} \sum_{m=1}^M \sum_{s=1}^M\ev{ T_{\eta} (y_{m, r, k}) - T_{\eta} (y_{r, k}) , \rho_{m, r, k} - \rho_{s, r, k} },
\end{align}
where we used the fact that since $y_{m, r, 0} = y_{s, r, 0} = x_r$ for all $m, s$ then
$\Lambda_{r, 0} = 0$. Define
\begin{align}
  \mu_{r, h} &= \frac{1}{M} \sum_{m=1}^M \norm{y_{m, r, h} - y_{r, h}}, && \overline{\mu}_{r, h} = \max_{k \leq h} \mu_{r, k}, \label{eq:a-5}\\
  X_{r, h} &= \frac{1}{\overline{\mu}_{r, h}} \frac{1}{M^2} \sum_{m=1}^M \sum_{s=1}^M  \ev{ T_{\eta} (y_{m, r, h}) - T_{\eta} (y_{r, h}) , \rho_{m, r, h} - \rho_{s, r, h} }. \label{eq:a-6}
\end{align}
Let $\ec[r,h]{\cdot}$ denote the expectation conditional on all information up to and
including round $r$ and local step $h$. Then,
\begin{align*}
\ec[r,h]{X_{r, h}} &= 0.
\end{align*}
Furthermore, we have by the triangle inequality, then our assumption on the
noise followed by Lemma~\ref{lem:contractivity} that almost surely
\begin{align}
\abs{\ev{ T_{\eta} (y_{m, r, h}) - T_{\eta} (y_{r, h}) , \rho_{m, r, h} - \rho_{s, r, h} }} &\leq  \norm{T_{\eta} (y_{m, r, h}) - T_{\eta} (y_{r, h})} \norm{\rho_{m, r, h} - \rho_{s, r, h}} \nonumber\\
&\leq \norm{T_{\eta} (y_{m, r, h}) - T_{\eta} (y_{r, h})} \left( \norm{\rho_{m, r, h}} + \norm{\rho_{s, r, h}} \right) \nonumber\\
&\leq 2 \sigma \norm{T_{\eta} (y_{m, r, h}) - T_{\eta} (y_{r, h})} \nonumber\\
\label{eq:a-4}
&\leq 2 \sigma \norm{y_{m, r, h} - y_{r, h}}.
\end{align}
By the definition of $X_{r, h}$ (\Cref{eq:a-6}), the triangle inequality, \Cref{eq:a-4}, and the
definition of $\overline{\mu}_{r, h}$ (\Cref{eq:a-5}) we have almost surely
\begin{align*}
  \abs{X_{r, h}} &= \frac{1}{\overline{\mu}_{r, h}} \abs{\frac{1}{M^2} \sum_{m=1}^M \sum_{s=1}^M  \ev{ T_{\eta} (y_{m, r, h}) - T_{\eta} (y_{r, h}) , \rho_{m, r, h} - \rho_{s, r, h} } } \\
            &\leq \frac{1}{\overline{\mu}_{r, h}} \frac{1}{M^2} \sum_{m=1}^M \sum_{s=1}^M \abs{\ev{ T_{\eta} (y_{m, r, h}) - T_{\eta} (y_{r, h}) , \rho_{m, r, h} - \rho_{s, r, h} }} \\
            &\leq \frac{2 \sigma}{\overline{\mu}_{r, h}} \frac{1}{M^2} \sum_{m=1}^M \sum_{s=1}^M \norm{y_{m, r, h} - y_{r, h}} \\
            &= 2 \sigma \frac{\frac{1}{M} \sum_{m=1}^M \norm{y_{m, r, h} - y_{r, h}}}{\overline{\mu}_{r, h}}  \\
            &\leq 2 \sigma.
\end{align*}
Then by \Cref{lem:dog-concentration} with $y_h = \overline{\mu}_{r, h}$ we have with
probability at least $1-\delta$
\begin{align}
  \abs{\sum_{k=0}^{h-1} \overline{\mu}_{r, k} X_{r, k}} &\leq 8 \overline{\mu}_{r, h-1} \sqrt{\theta_{h-1, \delta} \sum_{k=0}^{h-1} X_{r, k}^2 + 4 \sigma^2 \theta_{h, \delta}^2} \nonumber\\
  &\leq 8 \overline{\mu}_{r, h-1} \sqrt{\theta_{h-1, \delta} 4 h \sigma^2 + 4 \sigma^2 \theta_{h, \delta}^2 } \nonumber\\
  &\leq 16 \overline{\mu}_{r, h-1} \theta_{h-1, \delta} \sigma \sqrt{h+1}. \label{eq:a-7}
\end{align}
Observe that
\begin{align*}
  \sum_{k=0}^{h-1} \overline{\mu}_{r, k} X_{r, k} &= \frac{1}{M^2} \sum_{k=0}^{h-1} \sum_{m=1}^M \sum_{s=1}^M \ev{ T_{\eta} (y_{m, r, k}) - T_{\eta} (y_{r, k}) , \rho_{m, r, k} - \rho_{s, r, k} }.
\end{align*}
Using this and \cref{eq:a-7} to upper bound the right hand side of \cref{eq:a-8} we obtain
\begin{align}
\Lambda_{r, h} &\leq 4 \eta^2 \sigma^2 h + 32 \eta \overline{\mu}_{r, h-1} \theta_{h-1, \delta} \sigma \sqrt{h+1} \nonumber\\
&\leq 4 \eta^2 \sigma^2 h + 2 \alpha (32 \eta \theta_{h-1, \delta} \sigma \sqrt{h+1})^2 + \frac{\overline{\mu}_{r, h-1}^2}{2\alpha} \nonumber\\
&= \eta^2 \sigma^2 (h+1) \theta_{h-1, \delta}^2 (4 + 2048 \alpha) + \frac{\overline{\mu}_{r, h-1}^2}{2\alpha},
\label{eq:a-9}
\end{align}
where we used that $2ab \leq \alpha a^2 + \frac{1}{\alpha} b^2$ in the second step. Let $\overline{\Lambda}_{r, h} = \max_{k \leq h} \Lambda_{r, k}$. Observe that the right hand side of
\cref{eq:a-9} is increasing in $h$, therefore
\begin{align}
\label{eq:a-10}
\overline{\Lambda}_{r, h} \leq \eta^2 \sigma^2 (h+1) \theta_{h-1, \delta}^2 (4 + 2048 \alpha) + \frac{\overline{\mu}_{r, h-1}^2}{2\alpha}.
\end{align}
Observe that by the triangle inequality followed by \Cref{lem:jensen-application}
\begin{align*}
  \mu_{r, h} &= \frac{1}{M} \sum_{m=1}^M \norm{y_{m, r, h} - y_{r, h}} \\
  &\leq \frac{1}{M^2} \sum_{m=1}^M \sum_{s=1}^M \norm{y_{m, r, h} - y_{s, r, h}} \\
      &\leq \sqrt{\frac{1}{M^2} \sum_{m=1}^M \sum_{s=1}^M \norm{y_{m, r, h} - y_{s, r, h}}^2} \\
      &= \sqrt{\Lambda_{r, h}}.
\end{align*}
It follows that $\overline{\mu}_{r, h} \leq \sqrt{\overline{\Lambda}_{r, h}}$. Using this in
\cref{eq:a-10} we get
\begin{align*}
\overline{\Lambda}_{r, h} &\leq \eta^2 \sigma^2 (h+1) \theta_{h-1, \delta}^2 (4 + 2048 \alpha) + \frac{\overline{\Lambda}_{r, h-1}}{2\alpha} \\
&\leq \eta^2 \sigma^2 (h+1) \theta_{h-1, \delta}^2 (4 + 2048 \alpha) + \frac{\overline{\Lambda}_{r, h}}{2\alpha}.
\end{align*}
Rearranging we get
\begin{align*}
\left( 1 - \frac{1}{2\alpha} \right) \overline{\Lambda}_{r, h} \leq \eta^2 \sigma^2 (h+1) \theta_{h-1, \delta}^2 (4 + 2048 \alpha)
\end{align*}
Put $\alpha = 1$, then
\begin{align}
\label{eq:a-14}
\overline{\Lambda}_{r, h} \leq 4104 \eta^2 \sigma^2 (h+1) \theta_{h-1, \delta}^2.
\end{align}
Now that we have our bound on $\overline{\Lambda}_{r, h}$, we can use it to bound $V_{r, h}$ as follows
\begin{align}
\label{eq:a-13}
  V_{r, h} &= \frac{1}{M} \sum_{m=1}^M \sqn{y_{m, r, h} - y_{r, h}}.
\end{align}
Observe that by Jensen's inequality
\begin{align}
  \sqn{y_{m, r, h} - y_{r, h}} &= \sqn{y_{m, r, h} - \frac{1}{M} \sum_{s=1}^M y_{s, r, h}} \nonumber\\
  &= \sqn{\frac{1}{M} (y_{m, r, h} - y_{s, r, h})} \nonumber\\
\label{eq:a-11}
&\leq \frac{1}{M} \sum_{s=1}^M  \sqn{y_{m, r, h} - y_{s, r, h}}.
\end{align}
Combining \cref{eq:a-13,eq:a-11} we have
\begin{align*}
  V_{r, h} &\leq \frac{1}{M^2} \sum_{m=1}^M \sum_{s=1}^M \sqn{y_{m, r, h} - y_{s, r, h}} = \Lambda_{r, h}.
\end{align*}
Combining this with \cref{eq:a-14} yields the lemma's statement.
\end{proof}

\begin{lemma}\label{lem:round-regret} (Per-round regret).
In Algorithm~\ref{alg:fed-opt}, the iterates in a single communication round satisfy
\begin{align*}
\begin{split}
  \sqn{x_{r+1} - x_{\ast}} &\leq \sqn{x_r - x_{\ast}} + \gamma^2 \eta^2 \sum_{h=0}^{H-1} \sqn{g_{r, h}} + 2\gamma \eta \abs{1-\gamma} \zeta_2 \sum_{h=0}^{H-1} \norm{g_{r, h}} \\
  &\qquad + \frac{\gamma \zeta_3 H}{\alpha} + \frac{\alpha\gamma \eta^2}{2} \frac{1}{M} \sum_{m=1}^M \sum_{h=0}^{H-1} \sqn{g_{m, r, h}} - \frac{2\gamma \eta}{M} \sum_{h=0}^{H-1} \sum_{m=1}^M \ev{ y_{m, r, h} - x_{\ast} , g_{m, r, h} },
\end{split}
\end{align*}
where $\alpha > 0$ is arbitrary and
\begin{align*}
  \zeta_2 = \max_h \norm{y_{r, h} - y_{r, 0}}, && \zeta_{3} = \max_h \frac{1}{M} \sum_{m=1}^M \sqn{y_{m, r, h} - y_{r, h}}.
\end{align*}
\end{lemma}
\begin{proof}
Define the virtual sequences
\begin{align*}
  g_{r, h} &= \frac{1}{M} \sum_{m=1}^M g_{m, r, h}, && x_{r, 0} = x_r, && x_{r, h+1} = x_{r, h} - \gamma \eta g_{r, h}.
\end{align*}
We have
\begin{align}
\label{equ:15}
\sqn{x_{r, h+1} - x_{\ast}} &= \sqn{x_{r, h} - x_{\ast}} + \gamma^2 \eta^2 \sqn{g_{r, h}} - 2 \gamma \eta \ev{ x_{r, h} - x_{\ast} , g_{r, h} }
\end{align}
The inner product term can be decomposed as
\begin{align}
\label{equ:10}
- \ev{ x_{r, h} - x_{\ast} , g_{r, h} } &= - \ev{ x_{r, h} - y_{r, h} , g_{r, h} } - \ev{ y_{r, h} - x_{\ast} , g_{r, h} }.
\end{align}
Observe that $x_{r, h} = x_r - \gamma \eta \sum_{s=0}^{h-1} g_{r, s}$ and
$y_{r, h} = x_r - \eta \sum_{s=0}^{h-1} g_{r, s}$. Therefore,
\begin{align*}
  \norm{x_{r, h} - y_{r, h}} &= \norm{(\gamma-1) \eta \sum_{s=0}^{h-1} g_{r, s}} \\
&= \abs{\gamma-1} \norm{y_{r, h} - y_{r, 0}} \\
&\leq \abs{\gamma-1} \zeta_2,
\end{align*}
where $\zeta_2 = \max_h \norm{y_{r, h} - y_{r, 0}}$. Using this in \cref{equ:10}
\begin{align}
\label{equ:9}
- \ev{ x_{r, h} - y_{r, h} , g_{r, h} } \leq \norm{x_{r, h} - y_{r, h}} \norm{g_{r, h}} \leq \abs{1-\gamma} \zeta_2 \norm{g_{r, h}}.
\end{align}
Plugging \cref{equ:9} into \cref{equ:10} we get
\begin{align}
- &\ev{ x_{r, h} - x_{\ast} , g_{r, h} } \leq \abs{1-\gamma} \zeta_2 \norm{g_{r, h}} - \ev{ y_{r, h} - x_{\ast} , g_{r, h} } \nonumber \\
&\quad = \abs{1-\gamma} \zeta_2 \norm{g_{r, h}} - \frac{1}{M} \sum_{m=1}^M \ev{ y_{r, h} - x_{\ast} , g_{m, r, h} } \nonumber \\
\label{equ:11}
&\quad = \abs{1-\gamma} \zeta_2 \norm{g_{r, h}} - \frac{1}{M} \sum_{m=1}^M \ev{ y_{r, h} - y_{m, r, h} , g_{m, r, h} } - \frac{1}{M} \sum_{m=1}^M \ev{ y_{m, r, h} - x_{\ast} , g_{m, r, h} }.
\end{align}
For the second term in~\cref{equ:11} we have
\begin{align}
  - \frac{1}{M} \sum_{m=1}^M \ev{ y_{r, h} - y_{m, r, h} , g_{m, r, h} } &\leq \frac{1}{M} \sum_{m=1}^M \norm{y_{r, h} - y_{m, r, h}} \norm{g_{m, r, h}} \nonumber\\
  &\leq \frac{1}{M} \sum_{m=1}^M \left[ \frac{\sqn{y_{r, h} - y_{m, r, h}}}{2\alpha\eta} + \frac{\alpha \eta}{2} \sqn{g_{m, r, h}}  \right] \nonumber\\
\label{equ:13}
&\leq \frac{\zeta_3}{2\alpha \eta} + \frac{\alpha \eta}{2} \frac{1}{M} \sum_{m=1}^M \sqn{g_{m, r, h}}.
\end{align}
Plugging \cref{equ:13} into \cref{equ:11} we get
\begin{align}
\label{equ:14}
\begin{split}
      - \ev{ x_{r, h} - x_{\ast} , g_{r, h} } &\leq \abs{1-\gamma} \zeta_2 \norm{g_{r, h}} + \frac{\zeta_3}{2\alpha \eta} + \frac{\alpha \eta}{2} \frac{1}{M} \sum_{m=1}^M \sqn{g_{m, r, h}} \\
      &\quad - \frac{1}{M} \sum_{m=1}^M \ev{ y_{m, r, h} - x_{\ast} , g_{m, r, h} }.
\end{split}
\end{align}
Plug \cref{equ:14} back into \cref{equ:15} to get
\begin{align*}
\begin{split}
  \sqn{x_{r, h+1} - x_{\ast}} &\leq \sqn{x_{r, h} - x_{\ast}} + \gamma^2 \eta^2 \sqn{g_{r, h}} + 2\gamma \eta \abs{1-\gamma} \zeta_2 \norm{g_{r, h}} \\
&\qquad + \frac{\gamma \zeta_3}{\alpha} + \frac{\alpha \gamma \eta^2}{M} \sum_{m=1}^M \sqn{g_{m, r, h}} - \frac{2\gamma \eta}{M} \sum_{m=1}^M \ev{ y_{m, r, h} - x_{\ast} , g_{m, r, h} }.
\end{split}
\end{align*}
Recursing we get
\begin{align*}
\begin{split}
  \sqn{x_{r+1} - x_{\ast}} &\leq \sqn{x_r - x_{\ast}} + \gamma^2 \eta^2 \sum_{h=0}^{H-1} \sqn{g_{r, h}} + 2\gamma \eta \abs{1-\gamma} \zeta_2 \sum_{h=0}^{H-1} \norm{g_{r, h}} \\
  &\qquad + \frac{\gamma \zeta_3 H}{\alpha} + \frac{\alpha\gamma \eta^2}{2} \frac{1}{M} \sum_{m=1}^M \sum_{h=0}^{H-1} \sqn{g_{m, r, h}} - \frac{2\gamma \eta}{M} \sum_{h=0}^{H-1} \sum_{m=1}^M \ev{ y_{m, r, h} - x_{\ast} , g_{m, r, h} }.
\end{split}
\end{align*}
\end{proof}

\begin{proof}[Proof of \Cref{thm:loc-sgd-guarantee}]
Starting with the per-round recursion lemma, we have
\begin{align*}
\begin{split}
  \sqn{x_{r+1} - x_{\ast}} &\leq \sqn{x_r - x_{\ast}} + \gamma^2 \eta^2 \sum_{h=0}^{H-1} \sqn{g_{r, h}} + 2\gamma \eta \abs{1-\gamma} \zeta_2 \sum_{h=0}^{H-1} \norm{g_{r, h}} \\
  &\qquad + \frac{\gamma \zeta_3 H}{\alpha} + \frac{\alpha\gamma \eta^2}{2} \frac{1}{M} \sum_{m=1}^M \sum_{h=0}^{H-1} \sqn{g_{m, r, h}} - \frac{2\gamma\eta}{M} \sum_{h=0}^{H-1} \sum_{m=1}^M \ev{ y_{m, r, h} - x_{\ast} , g_{m, r, h} }.
\end{split}
\end{align*}
Observe that
\begin{align}
  \norm{y_{r, h} - y_{r, 0}} &= \eta \norm{\sum_{k=0}^{h-1} g_{r, k}} \nonumber\\
                                         &\leq \eta \sum_{k=0}^{h-1} \norm{g_{r, k}} \nonumber\\
\label{equ:18}
                                         &\leq \eta \sum_{k=0}^{H-1} \norm{g_{r, k}}.
\end{align}
Since this holds for any $h$, we have that $\zeta_2 \leq \eta \sum_{k=0}^{H-1} \| g_{r, k} \|$, where $\zeta_2$ is defined in \Cref{lem:round-regret}. Moreover, by \Cref{lem:vt-bound} we have that with probability $1-\delta$ and an application of the union bound that for all $r, h$
\begin{align}
    \frac{1}{M} \sum_{m=1}^{M} \| y_{m, r, h} - y_{r, h} \|^2 \leq 4104 \iota \eta^2 \sigma^2 H,
\end{align}
where $\iota =  2 \cdot \log \frac{60 \log 6RH}{\delta}$ and we used that $H+1 \leq 2H$. Since this bound holds for all $h$, we have
\begin{align*}
    \zeta_{3} = \max_{h} \frac{1}{M} \sum_{m=1}^{M} \|y_{m, r, h} - y_{r, h}\|^2 \leq 4104 \iota \eta^2 \sigma^2 H.
\end{align*}
Therefore by \Cref{equ:18} and \Cref{lem:vt-bound}
\begin{align*}
&\sqn{x_{r+1} - x_{\ast}} \leq \sqn{x_r - x_{\ast}} + \gamma^2 \eta^2 \sum_{h=0}^{H-1} \sqn{g_{r, h}} + 2\gamma \abs{1-\gamma} \eta^2 \left( \sum_{h=0}^{H-1} \norm{g_{r, h}} \right)^2 \\
  &\qquad + \frac{4104 \gamma \eta^2 \sigma^2 H^2}{\alpha} \iota + \frac{\alpha\gamma \eta^2}{2} \frac{1}{M} \sum_{m=1}^M \sum_{h=0}^{H-1} \sqn{g_{m, r, h}} - \frac{2\gamma \eta}{M} \sum_{h=0}^{H-1} \sum_{m=1}^M \ev{ y_{m, r, h} - x_{\ast} , g_{m, r, h} }.
\end{align*}
Let
$\xi_{m, r, h} = g_{m, r, h} - \nabla f(y_{m, r, h})$. Then,
\begin{align}
\begin{split}
\sqn{x_{r+1} - x_{\ast}} &\leq \sqn{x_r - x_{\ast}} + \gamma^2 \eta^2 \sum_{h=0}^{H-1} \sqn{g_{r, h}} + 2\gamma \abs{1-\gamma} \eta^2 \left( \sum_{h=0}^{H-1} \norm{g_{r, h}} \right)^2 + \frac{4104 \gamma \eta^2 \sigma^2 H^2}{\alpha} \iota \\
&\qquad + \frac{\alpha\gamma \eta^2}{2} \frac{1}{M} \sum_{m=1}^M \sum_{h=0}^{H-1} \sqn{g_{m, r, h}} - \frac{2\gamma\eta}{M} \sum_{h=0}^{H-1} \sum_{m=1}^M \ev{ y_{m, r, h} - x_{\ast} , \nabla f(y_{m, r, h}) } \\
&\qquad - \frac{2 \gamma \eta}{M} \sum_{h=0}^{H-1} \sum_{m=1}^M \ev{ y_{m, r, h} - x_{\ast} , \xi_{m, r, h} },
\end{split}\label{equ:123}
\end{align}
where $\xi_{m, r, h} = g_{m, r, h} - \nabla f(y_{m, r, h})$. Define
\begin{align*}
  \nu_{r, h} &= \frac{1}{M} \sum_{m=1}^M \norm{y_{m, r, h} - x_{\ast}}, && \overline{\nu}_{r, h} = \max_{p \leq r, s \leq h} \nu_{p, s}.
\end{align*}
Let 
\begin{align*}
    X_{r, h} = \frac{1}{\bar{\nu}_{r, h}} \frac{1}{M} \sum_{m=1}^{M} \ev{y_{m, r, h} - x_{\ast}, \xi_{m, r, h}}
\end{align*}
Let $\mathcal{F}_{r, h-1}$ denote the sigma algebra generated by all randomness up to and including step $r, h-1$. Note that
\begin{align*}
    \ec[\mathcal{F}_{r, h-1}]{X_{r, h}} &= \frac{1}{\bar{\nu}_{r, h}} \frac{1}{M} \sum_{m=1}^{M} \ec[\mathcal{F}_{r, h}]{\ev{y_{m, r, h} - x_{\ast}, \xi_{m, r, h}}} \\
    &= \frac{1}{\bar{\nu}_{r, h}} \frac{1}{M} \sum_{m=1}^{M} \ev{y_{m, r, h} - x_{\ast}, \ec[\mathcal{F}_{r, h}]{\xi_{m, r, h}}} \\
    &= 0,
\end{align*}
where we used that $\nu_{r, h}$ and $y_{m, r, h}$ are both $\mathcal{F}_{r, h-1}$-measurable and that the noise has mean zero. The edge cases $X_{r, 0}$ are handled similarly. Moreover, using the assumption that $\norm{\xi_{m, r, h}} \leq \sigma$ almost surely and the definition of $\bar{\nu}_{r, h}$,
\begin{align*}
    \norm{X_{r, h}} &= \norm{\frac{1}{\bar{\nu}_{r, h}} \frac{1}{M} \sum_{m=1}^{M} \ev{y_{m, r, h} - x_{\ast}, \xi_{m, r, h}}} \\
    &\leq \frac{1}{M} \sum_{m=1}^{M} \frac{\norm{y_{m, r, h} - x_{\ast}} \norm{\xi_{m, r, h}}}{\bar{\nu}_{r, h}} \\
    &\leq \frac{1}{M} \sum_{m=1}^{M} ( 1 \cdot \sigma) \\
    &= \sigma.
\end{align*}
Applying \Cref{lem:dog-concentration} on $X_{r, h}$ with $y_{r, h} = \bar{\nu}_{r, h}$, $C_{r, h} = \sigma$, $\hat{X}_{r, h} = 0$ we have 
\begin{align} \label{equ:cumulative-noise}
  \abs{\frac{1}{M} \sum_{r=0}^{R-1} \sum_{h=0}^{H-1} \sum_{m=1}^M \ev{ y_{m,r , h} - x_{\ast} , \xi_{m, r, h} } } \leq 16 \overline{\nu}_{R, H} \iota \sigma \sqrt{RH},
\end{align}
where $\iota$ is defined as before. 
Using \cref{equ:cumulative-noise} in \cref{equ:123}
\begin{align}
\begin{split}
  &\frac{2 \gamma \eta}{M} \sum_{m, r, h} \ev{ y_{m, r, h} - x_{\ast} , \nabla f(y_{m, r, h}) } \leq \sqn{x_0 - x_{\ast}} - \sqn{x_R - x_{\ast}} + \gamma^2 \eta^2 \sum_{r, h} \sqn{g_{r, h}}  \\
&\qquad + 2\gamma \abs{1-\gamma} \eta^2 \sum_{r=0}^{R-1} \left( \sum_{h=0}^{H-1} \norm{g_{r, h}} \right)^2 + R \cdot \frac{4104 \gamma \eta^2 \sigma^2 H^2}{\alpha} \iota \\
&\qquad + \frac{\alpha\gamma \eta^2}{2} \frac{1}{M} \sum_{m, r, h} \sqn{g_{m, r, h}}  + 2 \gamma \eta \left[ 16 \overline{\nu}_{R, H} \iota \sigma \sqrt{RH} \right].
\end{split}
\label{equ:12}
\end{align}
Let
\begin{align}
\label{equ:22}
\begin{split}
\Omega &= \gamma^2 \eta^2 \sum_{r, h} \sqn{g_{r, h}} + 2\gamma \abs{1-\gamma} \eta^2 \sum_{r=0}^{R-1} \left( \sum_{h=0}^{H-1} \norm{g_{r, h}} \right)^2 + R \cdot \frac{4104 \gamma \eta^2 \sigma^2 H^2}{\alpha} \iota \\
&+ \frac{\alpha\gamma \eta^2}{2} \frac{1}{M} \sum_{m, r, h} \sqn{g_{m, r, h}}
\end{split}
\end{align}
Then by convexity and \cref{equ:12} we get
\begin{align}\nonumber
 \sqn{x_R - x_{\ast}} &\leq \sqn{x_0 - x_{\ast}} + \Omega + 2 \gamma \eta \left[ 16 \overline{\nu}_{R, H} \iota \sigma \sqrt{RH} \right] - \frac{2 \gamma \eta}{M} \sum_{m, r, h} \ev{ y_{m, r, h} - x_{\ast} , \nabla f(y_{m, r, h}) } \\
&\leq \sqn{x_0 - x_{\ast}} + \Omega + 2 \gamma \eta \left[ 16 \overline{\nu}_{R, H} \iota \sigma \sqrt{RH} \right],
\label{equ:21}
\end{align}
where in the second line we used that $x_{\ast}$ is the minimizer of $f$ and therefore $\ev{y_{m, r, h} - x_{\ast}, \nabla f(y_{m, r, h})} \geq 0$ by convexity. It is not difficult to see that this guarantee in fact applies not just on
$\sqn{x_R - x_{\ast}}$ but on any $x_r$.  Let $d_r = \norm{x_r - x_{\ast}}$ and $\overline{d}_r = \max_{r^{\prime} \leq r} d_{r^{\prime}}$. Observe
\begin{align}
  \nu_{r, h} = \frac{1}{M} \sum_{m=1}^M \norm{y_{m, r, h} - x_{\ast}} &\leq \frac{1}{M} \sum_{m=1}^M \left[ \norm{y_{m, r, h} - y_{m, r, 0}} + \norm{x_r - x_{\ast}} \right] \nonumber\\
 &\leq \left[ \frac{\eta}{M} \sum_{m=1}^M \sum_{k=0}^{h-1} \norm{g_{m, r, k}} \right] + \norm{x_r - x_{\ast}} \nonumber \\
\label{equ:20}
&\leq \left[ \frac{\eta}{M} \sum_{m=1}^M \sum_{k=0}^{H-1} \norm{g_{m, r, k}} \right] + \norm{x_r - x_{\ast}}.
\end{align}

Using \cref{equ:20} in \cref{equ:21} we get 
\begin{align*}
\overline{d}_R^2 &\leq d_0^2 + \Omega + 32 \gamma \eta \iota \sigma \sqrt{RH} \overline{\nu}_{R, H} \\
               &\leq d_0^2 + \Omega + 32 \gamma \eta \iota \sigma \sqrt{RH} \left[ \frac{\eta}{M} \sum_{m, h} \norm{g_{m, r, h}} \right] + 32 \gamma \eta \iota \sigma \sqrt{RH} \overline{d}_R \\
               &\leq d_0^2 + \Omega + 2 \left( 32 \gamma \eta \iota \sigma \sqrt{RH} \right)^2 + \eta^2 \left( \frac{1}{M} \sum_{m, h} \norm{g_{m, r, h}} \right)^2 + \frac{\overline{d}_R^2}{2}.
\end{align*}
Therefore
\begin{align}\label{dr-bound}
\overline{d}_R^2 &\leq 2 d_0^2 + 2 \Omega + 4096 \gamma^2 \eta^2 \iota^2 \sigma^2 RH + 2 \eta^2 \left( \frac{1}{M} \max_{r} \sum_{m, h} \norm{g_{m, r, h}} \right)^2.
\end{align}

By the triangle inequality applied twice and the definition of $\bar{d}_R$,
\begin{align*}
    \norm{y_{m, r, s} - x_{\ast}} &\leq \norm{y_{m, r, 0} - y_{m, r, s}} + \norm{y_{m, r, 0} - x_{\ast}} \\
    &= \eta \norm{\sum_{h=0}^{s-1} g_{m, r, h}} + \norm{y_{m, r, 0} - x_{\ast}} \\
    &\leq \eta \sum_{h=0}^{s-1} \norm{g_{m, r, h}} + \norm{y_{m, r, 0} - x_{\ast}} \\
    &\leq \eta \sum_{h=0}^{s-1} \norm{g_{m, r, h}} + \bar{d}_{R} \\
    &\leq \eta \sum_{h=0}^{H-1} \norm{g_{m, r, h}} + \bar{d}_{R}.
\end{align*}
Therefore
\begin{align*}
    \frac{1}{M} \sum_{m=1}^{M} \norm{y_{m, r, s} - x_{\ast}} \leq \eta \left (\frac{1}{M} \sum_{m=1}^{M} \sum_{h=0}^{H-1} \norm{g_{m, r, h}} \right) + \bar{d}_R
\end{align*}
We now use the inequality $(a+b)^2 \leq 2 a^2 + 2 b^2$ to get
\begin{align*}
    \nu_{r, s}^2 &= \left ( \frac{1}{M} \sum_{m=1}^{M} \norm{y_{m, r, s} - x_{\ast}}  \right)^2 \\
    &\leq 2 \left ( \eta \left (\frac{1}{M} \sum_{m=1}^{M} \sum_{h=0}^{H-1} \norm{g_{m, r, h}} \right) \right)^2 + 2 \bar{d}_R^2 \\
    &= 2 \eta^2 \left( \frac{1}{M} \sum_{m, h} \norm{g_{m, r, h}} \right)^2 + 2 \overline{d}_R^2.
\end{align*}
Finally, using our bound on $\bar{d}_{R}^2$ given by equation~\eqref{dr-bound}
\begin{align*}
      \nu_{r, s}^2  &\leq 4 d_0^2 + 4 \Omega + 8192 \gamma^2 \eta^2 \iota^2 \sigma^2 RH + 6 \eta^2 \left( \frac{1}{M} \sum_{m, h} \norm{g_{m, r, h}} \right)^2,
\end{align*}
Therefore
\begin{align*}
    \bar{\nu}_{R, H}^2 &= \max_{r, s} \nu_{r, s}^2 \\
    &\leq 4 d_0^2 + 4 \Omega + 8192 \gamma^2 \eta^2 \iota^2 \sigma^2 RH + 6 \eta^2 \left( \frac{1}{M} \max_{r} \sum_{m, h} \norm{g_{m, r, h}} \right)^2.
\end{align*}
By \cref{equ:12,equ:22} and the last equation,
\begin{align}
  &\frac{2 \gamma \eta}{M} \sum_{m, r, h} \ev{ y_{m, r, h}  - x_{\ast} , \nabla f(y_{m, r, h}) } \leq \sqn{x_0 - x_{\ast}} - \sqn{x_R - x_{\ast}} + \Omega + 2 \gamma \eta \left[ 16 \overline{\nu}_{R, H} \iota \sigma \sqrt{RH} \right] \nonumber \\
&\qquad \leq d_0^2 - d_R^2 + \Omega + \frac{(32 \gamma \eta \iota \sigma \sqrt{RH})^2}{2} + 4 \left[ d_0^2 + \Omega + 2048 \gamma^2 \eta^2 \iota^2 \sigma^2 RH \right] + 6 \eta^2 R \left( \frac{1}{M} \max_r \sum_{m, h} \norm{g_{m, r, h}}  \right)^2 \nonumber \\
&\qquad = d_0^2 - d_R^2 + \Omega + \frac{(32 \gamma \eta \iota \sigma \sqrt{RH})^2}{2} + 4 \left[ d_0^2 + \Omega + 2048 \gamma^2 \eta^2 \iota^2 \sigma^2 RH \right] + 6 \eta^2 R \left( \frac{1}{M} \max_r \sum_{m, h} \norm{g_{m, r, h}}  \right)^2 \nonumber \\
&\qquad\leq d_0^2 - d_R^2 + 6 \Omega + 8704 \gamma^2 \eta^2 \iota^2 \sigma^2 RH + 4 d_0^2 + 6 \eta^2 R \left( \frac{1}{M} \max_r \sum_{m, h} \norm{g_{m, r, h}}  \right)^2.
\label{equ:903232}
\end{align}
Dropping the $-d_R^2$ term, we get
\begin{align*}
&\frac{2 \gamma \eta}{M} \sum_{m, r, h} \ev{ y_{m, r, h}  - x_{\ast} , \nabla f(y_{m, r, h}) } \leq 5 d_0^2 + 6 \Omega + 8704 \gamma^2 \eta^2 \iota^2 \sigma^2 RH + 6 \eta^2 R \left( \frac{1}{M} \max_r \sum_{m, h} \norm{g_{m, r, h}}  \right)^2 \\
\begin{split}
&\qquad \leq 5 d_0^2 + 6 \gamma^2 \eta^2 \sum_{r, h} \sqn{g_{r, h}} + 12 \gamma \abs{1-\gamma} \eta^2 \sum_{r=0}^{R-1} \left( \sum_{h=0}^{H-1} \norm{g_{r, h}} \right)^2 + RH \frac{24624 \gamma \eta^2 \sigma^2 H \iota}{\alpha}  \\
&\qquad + \frac{3 \alpha \gamma \eta^2}{M} \sum_{m,r,h} \sqn{g_{m, r, h }} + 8704 \gamma^2 \eta^2 \iota^2 \sigma^2 RH + 6 \eta^2 R \left( \frac{1}{M} \max_r \sum_{m, h} \norm{g_{m, r, h}} \right)^2.
\end{split}
\end{align*}
Dividing both sides by $2 \gamma \eta RH$ gives
\begin{align}
\begin{split}
&\frac{1}{M RH} \sum_{m, r, h} \ev{ y_{m, r, h}  - x_{\ast} , \nabla f(y_{m, r, h}) } \leq \frac{5 d_0^2}{2 \gamma \eta RH}  + \frac{3 \gamma \eta}{RH}  \sum_{r, h} \sqn{g_{r, h}} \\
&\qquad + \frac{6 \abs{1-\gamma} \eta}{RH} \sum_{r=0}^{R-1} \left( \sum_{h=0}^{H-1} \norm{g_{r, h}} \right)^2 + \frac{24624 \eta \sigma^2 H \iota}{\alpha}  \\
&\qquad  + \frac{3 \alpha \eta}{M RH} \sum_{m,r,h} \sqn{g_{m, r, h }} + 8704 \gamma \eta \iota^2 \sigma^2 + \frac{6 \eta}{\gamma H}  \left( \frac{1}{M} \max_r \sum_{m, h} \norm{g_{m, r, h}} \right)^2.
\end{split}\label{equ:23}
\end{align}
Observe that by optimizing over $\alpha$ we have
\begin{align*}
  \frac{24624 \eta \sigma^2 H \iota}{\alpha} + \frac{3 \alpha \eta}{M RH} \sum_{m,r,h} \sqn{g_{m, r, h }} &\leq 2 \sqrt{ \left( 24624 \eta \sigma^2 H \iota \right) \left( \frac{3 \eta}{M RH} \sum_{m, r, h} \sqn{g_{m, r, h}} \right) } \\
&\leq 544 \eta \sigma \iota \sqrt{ \frac{1}{MR} \sum_{m, r, h} \sqn{g_{m, r, h}} }.
\end{align*}
Using this in \cref{equ:23} followed by convexity completes the proof.
\end{proof}

\end{document}